\documentclass{article}

\PassOptionsToPackage{numbers, compress}{natbib}
\usepackage[main, final]{neurips_2026}
\usepackage[utf8]{inputenc} % allow utf-8 input
\usepackage[T1]{fontenc}    % use 8-bit T1 fonts
\usepackage{hyperref}       % hyperlinks
\usepackage{url}            % simple URL typesetting
\usepackage{booktabs}       % professional-quality tables
\usepackage{amsfonts}       % blackboard math symbols
\usepackage{nicefrac}       % compact symbols for 1/2, etc.
\usepackage{microtype}      % microtypography
\usepackage{xcolor}         % colors
\usepackage[dvipsnames]{xcolor}

\usepackage{multirow}
\usepackage{multicol}
\usepackage{colortbl}
\usepackage{pifont}
\usepackage{float}
\usepackage{algorithm, algpseudocode}
\makeatletter
\renewcommand{\theHALG@line}{\thealgorithm.\arabic{ALG@line}}
\makeatother
\usepackage{bm}
\usepackage{setspace}
\usepackage{makecell}

\usepackage[utf8]{inputenc}
\usepackage[T1]{fontenc}
\usepackage{hyperref}
\usepackage{url}
\usepackage{booktabs}
\usepackage{amsfonts}
\usepackage{nicefrac}
\usepackage{microtype}
\usepackage{amsthm}
\usepackage{amssymb}
\usepackage{mathtools}
\usepackage{tabularx}
\usepackage{enumitem}
\usepackage{graphicx}
\graphicspath{{./figs}}
\DeclareGraphicsExtensions{.eps, .pdf, .png, .jpg}

\usepackage{amsmath}

\usepackage[capitalize]{cleveref}
\crefname{section}{Sec.}{Secs.}
\crefname{table}{Tab.}{Tabs.}
\Crefname{figure}{Fig.}{Figs.}

\usepackage{subcaption}
\theoremstyle{plain}
\newtheorem{theorem}{Theorem}[section]

\theoremstyle{definition}
\newtheorem{definition}[theorem]{Definition}

\theoremstyle{remark}

\title{Two Halves are More than One: \\ Phase-wise Velocity Distillation for Fast and High-Quality Image Generation}

\author{
  Zhen Guo$^{1,2}$, Rongyuan Wu$^{1,2}$, Qiaosi Yi$^{1,2}$,
  Chenxi Xie$^{1,2}$, Xinyu Wei$^{1,2}$, Lei Zhang$^{1,2}$\thanks{Corresponding author. This work is supported by the PolyU-OPPO Joint Innovative Research Center.} \\
  $^1$Department of Computing, The Hong Kong Polytechnic University \\
  $^2$OPPO Research Institute \\
}

\begin{document}

\maketitle

\begin{abstract}
Recent diffusion-based image generation backbones have grown substantially in scale, making the network inference cost increase rapidly. While diffusion distillation techniques can reduce the number of inference steps, high-quality image generation within a single full-backbone-forward compute budget remains challenging.
Existing one-step methods typically allocate this budget to a single evaluation
of a monolithic student. However, approximating the heterogeneous coarse-to-fine
transport with a single monolithic mapping is difficult and often leads to
over-smoothed outputs.
To address this issue, we propose \textbf{P}hase-wise \textbf{V}elocity
\textbf{D}istillation (\textbf{PVD}), which partitions the generation timeline
into a coarse and a fine phase, and models the transition within each phase via
the average velocity. A dedicated half-sized expert is assigned to each
phase, decoupling structural composition from detail refinement while keeping
the cumulative computation equivalent to one full-backbone forward pass.
We show that the use of two half-sized phase-specific experts outperforms a single full-size monolithic student.
On class-conditional image generation, PVD achieves an FID of 1.48 on ImageNet $256 \times 256$.
On more complex text-to-image (T2I) tasks, PVD-distilled models (Stable Diffusion 3.5-Medium, FLUX.1-dev, Qwen-Image) produce results competitive
with their multi-step teachers, significantly outperforming prior distillation methods.
Moreover, across the evaluated T2I backbones,
PVD reduces active parameters by 49.10--50.89\% and peak VRAM by
45.76--48.36\% compared to the corresponding teachers.
Source code and distilled models are available at \url{https://github.com/PolyU-VCLab/PVD}.
\end{abstract}

\begin{figure}[h]
    \centering
    \includegraphics[width=1\textwidth]{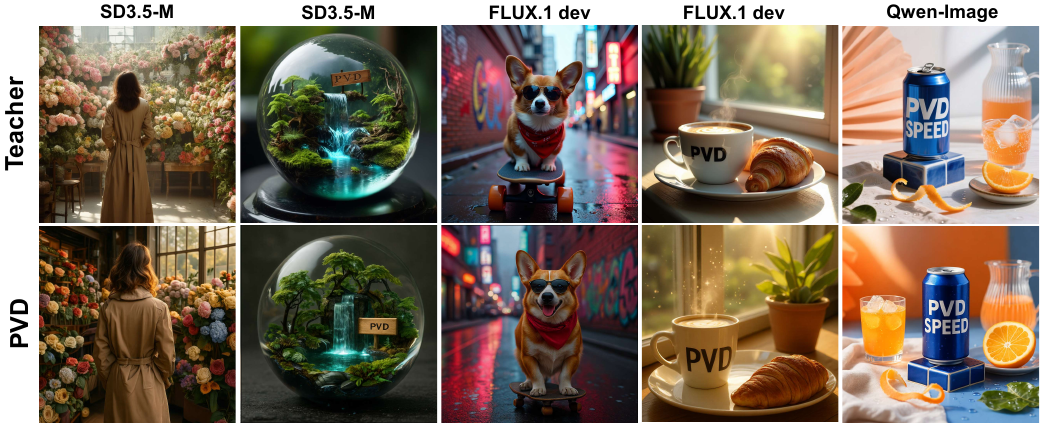}
    \caption{Text-to-image results. \textbf{Top}: Teacher model (Stable Diffusion3.5-Medium, FLUX.1 dev and Qwen-Image under $50\times2$, $50$, and $50\times2$ steps, respectively). \textbf{Bottom}: {With two half-sized phase-wise experts, PVD achieves fast and high-quality image generation using cumulative computation equivalent to one full-backbone forward pass of the teacher model.}}
    \label{fig:head}
\end{figure}

\section{Introduction}
\label{sec:intro}
Diffusion models have substantially advanced image synthesis and can generate high-quality images with diverse content~\citep{ho2020denoising, dhariwal2021diffusion, liu2022flow, esser2024scaling, flux2024, wu2025qwen, cai2025z}. However, their generation process follows an iterative coarse-to-fine trajectory~\citep{ho2020denoising, kulikov2025flowedit, cheng2025structure, wan2025wan}, which typically requires tens to hundreds of costly full-backbone-forward passes to transform noise into desired images. This computational bottleneck poses a significant barrier to interactive applications in the real world~\citep{kodaira2025streamdiffusion, liu2025survey, lee2025semanticdraw}.

To accelerate generative inference, extensive efforts have been dedicated to diffusion distillation by reducing multi-step sampling into a few or even a single forward pass. Existing methods can be broadly grouped into two paradigms (see Fig.~\ref{fig:intro}, left and middle). \emph{Trajectory-based} methods train the student to replicate the teacher's generative trajectory by progressively halving the sampling steps~\citep{salimans2022progressive, meng2023distillation, liu2023instaflow}, enforcing self-consistency along the trajectory~\citep{song2023consistency, song2023improved, lu2024simplifying}, or directly regressing the average velocity over a temporal interval~\citep{frans2024one, geng2025mean, geng2025improved, zhang2025alphaflow, peng2025facm}. \emph{Distribution-based} methods instead align the student's output distribution with that of the teacher or the data by adversarial supervision~\citep{sauer2024adversarial, sauer2024fast} or score/divergence-based distribution matching~\citep{yin2024one, yin2024improved, zhou2024score, ge2025senseflow}.

Despite the remarkable progress of the above-mentioned two paradigms, they share a common architectural assumption: a \emph{single, monolithic student} is able to accomplish the entire generative mapping from pure noise to images with rich details, regardless of whether the inference unrolls into a few steps or a single forward pass. This assumption conflicts with the intrinsically dynamic nature of the diffusion process, in which early stages establish global low-frequency structure while later stages add high-frequency details~\citep{ho2020denoising, wang2023diffusion, kulikov2025flowedit, wan2025wan, liu2025navigation}. When the inference budget is relaxed to multiple steps, the same set of parameters can revisit the trajectory several times and partially compensate for this mismatch; under a strict single full-backbone-forward budget, however, the monolithic student must resolve coarse layout and fine details simultaneously in one forward pass, exceeding the capacity of a single compact model. 
Since a noisy state may admit multiple plausible clean targets, the model tends to regress toward their conditional average to minimize the training loss~\citep{frans2024one}, leading to over-smoothed outputs with diminished fine details and even generation collapse in one-step sampling, especially for complex Text-to-Image (T2I) tasks~\citep{pu2025few}.

To overcome these limitations, we propose \textbf{P}hase-wise \textbf{V}elocity \textbf{D}istillation (\textbf{PVD}), a phase-decomposed distillation framework for high-quality image generation under a single full-backbone-forward compute budget (see Fig.~\ref{fig:intro}, right). PVD partitions the generation trajectory into a structural phase and a refinement phase, and assigns a dedicated half-sized expert to each phase to align the model capacity with the coarse-to-fine nature of diffusion while keeping the cumulative inference cost within a single forward pass of the teacher. Each expert is supervised to approximate the localized average velocity over its phase, replacing the global noise-to-data jump with two shorter-horizon transport objectives. For more complex T2I tasks, we further introduce phase-specific discriminators, one focusing on global structural consistency and the other on fine-grained details, to compensate for the over-smoothing tendency of $L_2$ velocity regression. The two experts are able to share a common half-sized backbone and are instantiated as lightweight Low-Rank Adaptation (LoRA) modules, substantially reducing memory overhead.
In summary, our contributions are as follows:
\begin{enumerate}[label=(\alph*), leftmargin=*]
    \item We introduce \textbf{Phase-wise Velocity Distillation (PVD)}, a trajectory decomposition distillation framework that partitions the diffusion process into temporally phase-wise average-velocity transports under a single full-backbone-forward budget constraint.
    \item We train \textbf{Phase-wise Experts} via localized mean-velocity objectives to align model capacity with the coarse-to-fine dynamics of diffusion, improving generation quality without increasing the compute budget.
    \item PVD achieves \textbf{superior performance} on both class-conditional image (C2I) and T2I generation to existing distillation methods. On C2I, it reaches FID 1.48 on ImageNet $256\times 256$; on T2I, distilled Stable Diffusion 3.5-Medium, FLUX.1-dev, and Qwen-Image models demonstrate competitive performance on GenEval, TIIF-Bench, and Qwen-Image-Bench, etc., while significantly reducing active parameters and peak VRAM relative to the corresponding teachers.
\end{enumerate}

\begin{figure}[tb]
    \centering
    \includegraphics[width=1\textwidth]{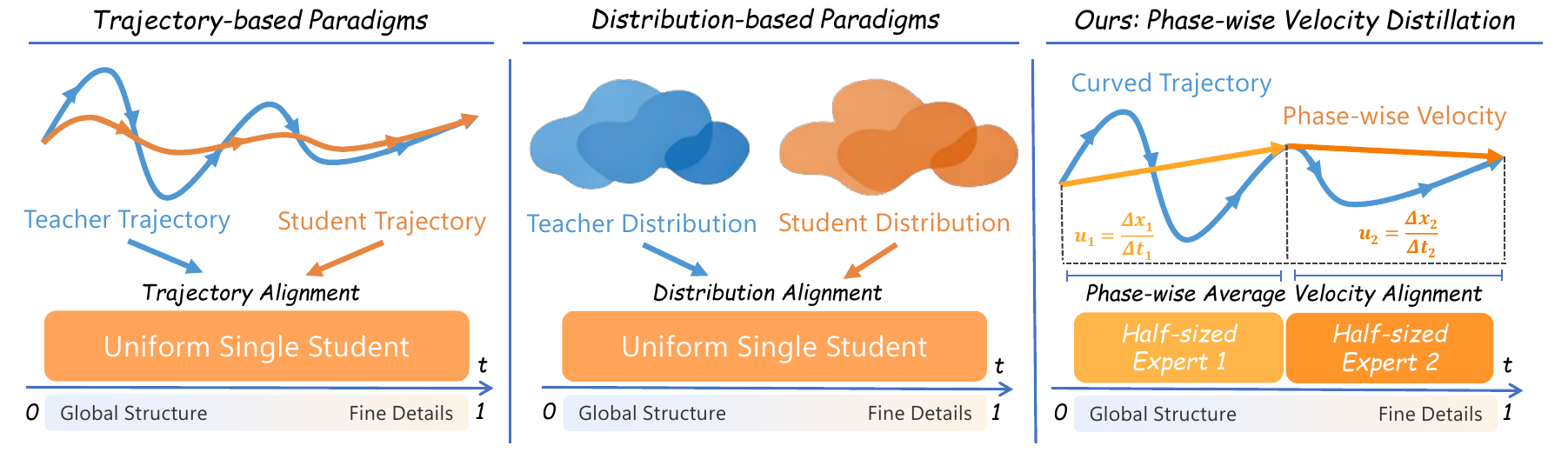}
    \vspace{-5mm}
    \caption{Trajectory-based distillation (left) and distribution-based distillation (middle) employ a uniform student model to approximate the full coarse-to-fine mapping, while our proposed phase-wise velocity distillation (right) decomposes the objective into phase-wise velocity targets and sequentially evaluates two half-sized experts under a single full-backbone-forward compute budget.}
    \label{fig:intro}
    \vspace{-2mm}
\end{figure}

\section{Related Work}
\label{sec:related_work}
\textbf{Denoising Diffusion and Flow Matching.}
Early diffusion models, represented by DDPM~\citep{ho2020denoising} and DDIM~\citep{song2020denoising}, employ iterative denoising for high-fidelity generation, which enjoys stable optimization but is tied to a fixed reverse diffusion path of tens to hundreds of small denoising steps, making it difficult to apply to large-scale models.
Flow matching~\citep{lipman2022flow} and rectified flow~\citep{liu2022flow} revisit this problem from the perspective of continuous transportation. Instead of committing to a fixed diffusion path, flow matching learns velocity fields by simulation-free regression over conditional probability paths, while rectified flow favors straighter trajectories to regress with coarse ODE solvers. With transformer backbones~\citep{peeblesScalableDiffusionModels2023, esser2024scaling}, these paradigms can be well scaled to large T2I models, such as FLUX~\citep{flux2024, flux2025}, Qwen-Image~\citep{wu2025qwen}, and Z-Image~\citep{cai2025z}. Nevertheless, their inference still relies on multi-step integration, and even rectified training paths can yield curved marginal trajectories in practice~\citep{frans2024one, geng2025mean}. Such an expensive inference cost motivates few-step distillation techniques.

\textbf{Few-step Distillation.}
Existing few-step distillation methods can be broadly categorized into trajectory-based and distribution-based paradigms. The former aims to preserve teacher dynamics while reducing solver steps. For example, progressive distillation~\citep{salimans2022progressive, meng2023distillation} recursively compresses the teacher sampler, while InstaFlow~\citep{liu2023instaflow} shortens the transport path through reflow. Consistency methods, including CM~\citep{song2023consistency}, iCM~\citep{song2023improved}, sCM~\citep{lu2024simplifying}, and Hyper-SD~\citep{ren2024hyper}, replace explicit solver imitation with consistency constraints. More recent methods, including Shortcut Models~\citep{frans2024one}, MeanFlow~\citep{geng2025mean}, improved MeanFlow (iMF)~\citep{geng2025improved}, $\alpha$-Flow~\citep{zhang2025alphaflow}, FACM~\citep{peng2025facm}, ArcFlow~\citep{yang2026arcflow}, and TDM~\citep{luo2025tdm}, explicitly model long-horizon transport.
The distribution-based paradigm relaxes path fidelity to match the student output distribution to the teacher or data distribution. ADD~\citep{sauer2024adversarial} and LADD~\citep{sauer2024fast} use adversarial supervision to improve perceptual sharpness. DMD~\citep{yin2024one}, DMD2~\citep{yin2024improved}, and their variants~\citep{zhou2024score, ge2025senseflow} use score- or divergence-based objectives to align the student distribution. Recent extensions further refine this family: Phased DMD~\citep{fan2025phaseddmd} introduces subinterval-wise progressive distribution matching, while Decoupled DMD~\citep{liu2025decoupleddmd} uses CFG augmentation as the main distillation engine and distribution matching as a stabilizing regularizer.
By relaxing path-level supervision in favor of distribution alignment, distribution-based approaches tend to improve perceptual sharpness, at the cost of weaker trajectory fidelity under aggressively reduced budgets.

\textbf{Phase-wise Generation.}
The diffusion process is phase-structured: early denoising steps produce global layout and semantics, while later steps refine textures and high-frequency details~\citep{ho2020denoising, kulikov2025flowedit}. It has been shown~\citep{wang2023diffusion, liu2025navigation} that the early navigation stage is guided by multiple data modes and the later refinement stage is dominated by the closest samples.
Based on these observations, eDiff-I~\citep{balaji2022ediff} uses an ensemble of experts for different denoising regimes; Wan2.2~\citep{wan2025wan} adopts an SNR-based two-expert MoE that switches from a high-noise expert to a low-noise expert; Glance~\citep{dong2025glance} employs a lightweight variant through SNR-guided slow and fast LoRA experts; and TimeStep Master~\citep{zhuang2025timestep}
combines interval-specific LoRA experts through time-dependent routing.
On the other hand, Hierarchical Distillation~\citep{cheng2025structure} employs trajectory distillation for structure generation and distribution refinement for details synthesis, while Phased DMD~\citep{fan2025phaseddmd} introduces subinterval-wise progressive refinement within DMD-style distillation.
While reducing the gap between coarse semantic planning and fine-detail rendering, 
these methods operate in the few-step regime, and do not work under a strict single full-backbone-forward compute budget.

{Further discussions on the training targets, parameter specialization, and inference structures of related methods are provided in \textbf{Appendix~{\ref{app:related_comparison}}}.}

\section{Method}
\label{sec:method}

\begin{figure}[tb]
    \centering
    \includegraphics[width=1\textwidth]{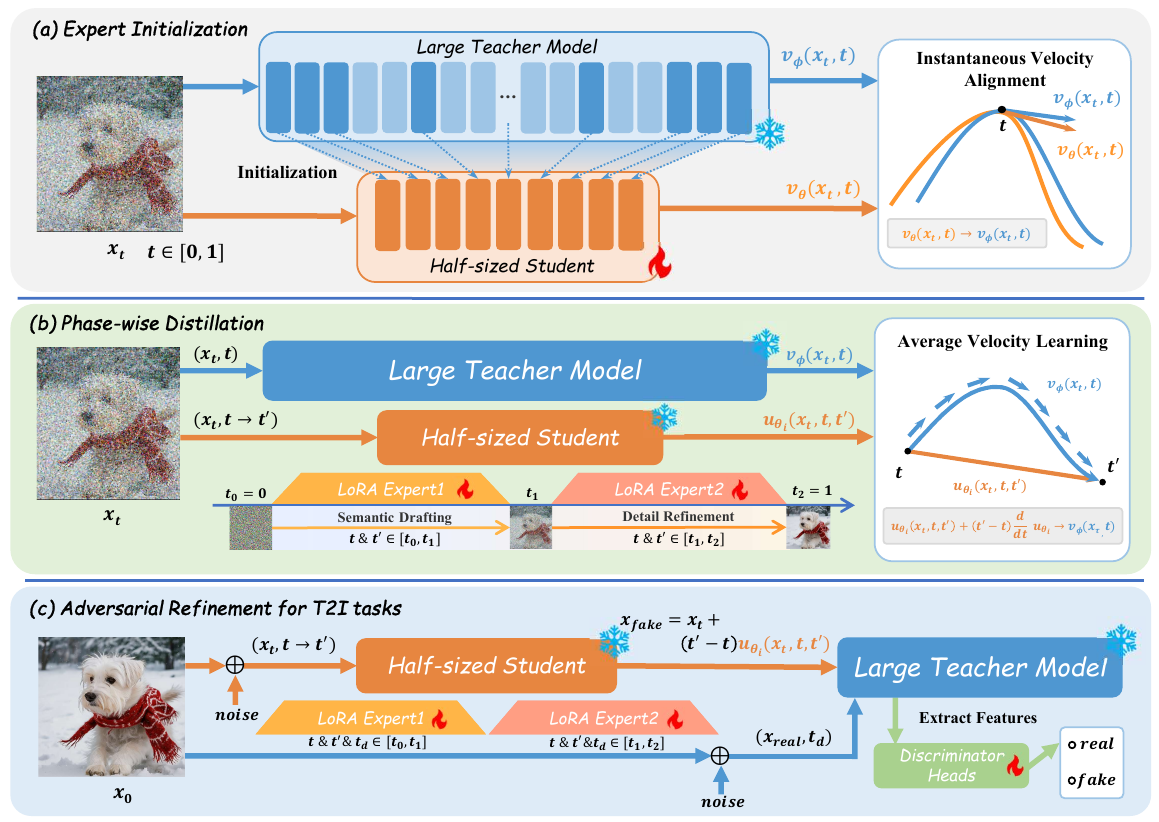}
    \caption{Overview of Phase-wise Velocity Distillation (PVD). (a) A half-sized backbone is extracted from the teacher and fine-tuned via flow matching to initialize the local students. (b) Phase-specific LoRA experts are distilled by mean velocity learning across partitioned time intervals. (c) For T2I tasks, phase-wise discriminators are used to provide adversarial refinement.}
    \label{fig:method}
    \vspace{-3mm}
\end{figure}

\subsection{Framework Overview}

Due to the highly nonlinear geometry of global diffusion trajectories, high-fidelity generation under constrained inference budgets often encounters unexpected errors and oversmoothed results. To address this issue, we propose \textbf{Phase-wise Velocity Distillation (PVD)}, which decomposes the generative trajectory into localized phases to distill their \emph{average velocity fields}. For each local phase, a lightweight, phase-specific expert is used to accomplish the transport. 
As shown in \cref{fig:method}, PVD employs a three-stage training paradigm. First, a half-sized backbone extracted from the teacher is fine-tuned via flow matching to initialize the student. Second, we partition the time horizon into distinct phases, and use lightweight, phase-specific experts to learn localized average velocities. Third, for the more difficult T2I tasks, discriminators are trained to provide adversarial refinement.

\textbf{Notation.} Let $\bm{x}_t = (1-t)\bm{x}_0 + t\bm{x}_1$, $t \in [0,1]$, be the linear flow path from pure noise $\bm{x}_0 \sim p_0$ to clean data $\bm{x}_1 \sim p_1$. The teacher $\phi$ and its half-sized backbone $\theta$ predict instantaneous velocities $v_\phi(\bm{x}_t, t)$ and $v_\theta(\bm{x}_t, t)$, respectively. We partition the time horizon into intervals $\mathcal{I}_i = [t_{i-1}, t_i]$. Within $\mathcal{I}_i$, an expert $u_{\theta_i}(\bm{x}_t, t, t')$ predicts the average velocity for $t < t'$, and is parameterized as the shared half-sized backbone $\theta$ augmented with phase-specific parameters $\psi_i$ (instantiated as either full fine-tuning or lightweight LoRA modules). Discriminator is denoted as $D_{\omega_i}(\cdot)$.

\subsection{Expert Initialization} \label{sec:expert_init}
Before localized distillation, we extract a reduced-capacity backbone and finetune it to serve as a shared, well-conditioned foundation for all phase-wise experts.

\textbf{Backbone Extraction and Finetuning.} To stabilize optimization and reduce computational overhead, we extract a half-sized backbone from the teacher model by selectively inheriting its weights. Following empirical observations that intermediate DiT blocks contribute less to final results than early and late blocks~\citep{xie2025sana}, we retain the first three and the last three blocks, and uniformly sample the remaining intermediate layers until the depth is exactly halved. The extracted half-sized backbone $\theta$ is then fine-tuned over the continuous time horizon $t \sim \mathcal{U}[0,1]$ via standard flow matching to mimic the teacher's instantaneous dynamics:
\begin{equation} \label{eq:fm}
    \mathcal{L}_{\text{FM}} = \mathbb{E}_{\bm{x}_0, \bm{x}_1, t} \left[ \left\| v_\theta(\bm{x}_t, t) - v_\phi(\bm{x}_t, t) \right\|_2^2 \right],
\end{equation}
where $v_\phi$ and $v_\theta$ denote the instantaneous velocity fields of the teacher and the student backbone, respectively. We have the following theorem.

\begin{theorem}[Integration Error Bound] \label{theorem:error_bound}
For any temporal phase $[t,t'] \subseteq [0,1]$, let $u_\phi$ denote the teacher average velocity and let $\widetilde{u}_\theta$ denote the average velocity induced by integrating the aligned backbone velocity field $v_\theta$ along the same trajectory. Then their discrepancy satisfies:
\begin{equation} \label{eq:lemma_bound}
    \left\| \widetilde{u}_\theta(\bm{x}_t, t, t') - u_\phi(\bm{x}_t, t, t') \right\|
    \le
    \frac{1}{t' - t} \int_{t}^{t'} \left\| v_\theta(\bm{x}_\tau, \tau) - v_\phi(\bm{x}_\tau, \tau) \right\| d\tau.
\end{equation}
\end{theorem}
\textit{Proof}: The full proof is provided in \textbf{Appendix~{\ref{app:analysis}}}.

\textit{Remark.}
{For a fixed interval $[t,t']$, \cref{theorem:error_bound} bounds the discrepancy between the teacher average velocity $u_\phi$ and the backbone-induced average velocity $\widetilde{u}_\theta$ by the interval-averaged instantaneous prediction error. The objective in \cref{eq:fm} minimizes this error over the full time horizon, providing a shared initialization for PVD to optimize each expert on its assigned interval.}

\textbf{Phase-wise Expert Initialization.} The aligned half-sized backbone serves as the starting point for phase-wise experts. To construct distinct experts $u_{\theta_i}$ for each temporal interval, we freeze the shared backbone weights $\theta$ and then introduce an additional zero-initialized time embedding layer alongside phase-specific LoRA parameters $\psi_i$. Consequently, at the beginning of phase-wise distillation, the initial output of each expert is identical to the instantaneous velocity of the backbone.

\subsection{Phase-wise Distillation} \label{sec:velocity_distillation}

Rather than globally modeling the teacher's probability-flow ODE (PF-ODE) $d\bm{x}_t / dt = v_\phi(\bm{x}_t, t)$, PVD decomposes the generative trajectory into temporally distinct phases, theoretically motivated by the following localized transport property:

\begin{theorem}[Temporal Marginal-Shift Bound] \label{theorem:smoothness}
Assume that the velocity field $v(\bm{x},t)$ is uniformly bounded in space and time, i.e., there exists a constant $C>0$ such that:
\begin{equation}
\|v(\bm{x},t)\| \leq C, \quad \forall (\bm{x},t)\in \mathbb{R}^d \times [0,1].
\end{equation}
Let $p_t$ denote the marginal distribution of $\bm{x}_t$ induced by the PF-ODE.
For any interval $[t,t'] \subseteq [0,1]$, the induced marginal distributions satisfy:
\begin{equation}
W_2(p_t, p_{t'}) \le C |t'-t|,
\end{equation}
where $W_2$ denotes the 2-Wasserstein distance.
\end{theorem} 
\textit{Proof}: The full proof is provided in \textbf{Appendix~{\ref{app:analysis}}}.

\textit{Remark.}
\cref{theorem:smoothness} shows that the $W_2$ distance between adjacent marginals is upper bounded linearly with the interval length. Partitioning $[0,1]$ into intervals tightens the marginal-shift bound from $C$ to $C|t_i-t_{i-1}|$ per phase, motivating the use of lightweight experts for each phase.

\textbf{Phase-wise Generation.} 
Motivated by \cref{theorem:smoothness}, we construct a localized two-phase generation process (see \cref{fig:method}b). During the initial \emph{Semantic Drafting} phase ($\mathcal{I}_1 = [0, t_1]$), the first expert $u_{\theta_1}$ transports pure noise $\bm{x}_0$ to an intermediate state $\bm{x}_{t_1}$, aiming to emphasize global layout and low-frequency structure while producing the initial condition for the subsequent phase. The \emph{Detail Refinement} phase ($\mathcal{I}_2 = [t_1, 1]$) then applies the second expert $u_{\theta_2}$ to the intermediate state to refine the remaining local structure and high-frequency detail.
Mathematically, to enable {one-step jumps} within $\mathcal{I}_i$, each expert $u_{\theta_i}$ predicts the average velocity for $t < t'$:
\begin{equation} \label{eq:vu}
u(\bm{x}_t, t, t') = \frac{1}{t' - t} \int_{t}^{t'} v_\phi(\bm{x}_s, s)\, ds.
\end{equation}
Following MeanFlow~\citep{geng2025mean}, differentiating the interval average in \cref{eq:vu} with respect to the interval start $t$ along the teacher PF-ODE while fixing $t'$ yields:
\begin{equation} \label{eq:mean_velocity_identity}
u_\phi(\bm{x}_t,t,t') = v_\phi(\bm{x}_t,t) + (t'-t)\,\frac{d}{dt}u_\phi(\bm{x}_t,t,t'),
\end{equation}
or equivalently $v_\phi=u_\phi-(t'-t)\frac{d}{dt}u_\phi$. Here, $u_\phi$ denotes the corresponding teacher average, and $\frac{d}{dt}$ is the total derivative with respect to $t$ along the trajectory. We distill each phase expert $u_{\theta_i}$ using the self-consistent stop-gradient target induced by \cref{eq:mean_velocity_identity}:
\begin{equation} \label{eq:mean_velocity_loss}
\mathcal{L}_{\text{vel}}^{(i)} = \mathbb{E} \left[ \left\| u_{\theta_i}(\bm{x}_t, t, t') - \operatorname{sg}\!\left( v_\phi(\bm{x}_t,t) + (t'-t)\,\frac{d}{dt}u_{\theta_i}(\bm{x}_t,t,t') \right) \right\|_2^2 \right],
\end{equation}
where $\operatorname{sg}(\cdot)$ denotes the stop-gradient operator.
We also sample $t'=t$ during training to recover instantaneous matching and stabilize optimization.

\subsection{Adversarial Refinement} \label{sec:discriminator}
While $L_2$ velocity matching (\cref{eq:mean_velocity_loss}) provides teacher-anchored supervision, it may regress to the mean and yield perceptually oversmoothed outputs. We therefore introduce phase-wise discriminators for complementary adversarial supervision in T2I tasks.
Unlike traditional GANs that learn global mappings, our discriminators, denoted by $D_{\omega_i}$, act as localized perceptual enhancers operating strictly within their respective phases $\mathcal{I}_i$. For $t < t'$, the expert generates $\bm{x}_{\text{fake}} = \bm{x}_t + (t' - t) u_{\theta_i}(\bm{x}_t, t, t')$, and minimizes the adversarial penalty $\mathcal{L}_{\text{adv}}^{(i)} = -\mathbb{E} [ D_{\omega_i}(\bm{x}_{\text{fake}}) ]$ derived from a standard hinge loss. 
The final training objective integrates this penalty as a secondary marginal regularizer, while the teacher-anchored velocity loss remains the primary distillation objective.

The training and inference procedures are summarized in \textbf{Algorithms~\ref{alg:pvd_training} and~\ref{alg:pvd_inference}} in \textbf{Appendix~{\ref{app:algo}}}.

\section{Experiments}
\label{sec:experiments}

\subsection{Experiment Setup}
\label{sec:exp_setup}

\textbf{Experimental Settings.}
We evaluate PVD on both class-conditional image (C2I) generation and text-to-image (T2I) generation.
For C2I, we follow the ImageNet $256\times 256$ setting used in recent one-step studies such as FACM~\citep{peng2025facm}, and report FID-50K and Inception Score (IS).
For T2I, we distill students from SD3.5-Medium~\citep{esser2024scaling}, FLUX.1-dev~\citep{flux2024}, and Qwen-Image~\citep{wu2025qwen}, and evaluate them on various evaluation benchmarks, including GenEval~\citep{ghosh2023geneval}, DPG-Bench~\citep{hu2024ella}, WISE~\citep{niu2025wise}, TIIF-Bench miniset~\citep{wei2025tiif}, {Qwen-Image-Bench (QI-Bench)~\citep{liQwenImageBenchGenerationCreation2026}}, together with {Aesthetic Score using Aesthetic Predictor V2.5~\citep{aesthetic_predictor_v2_5}}, PickScore~\citep{kirstain2023pick} and ImageReward~\citep{xu2023imagereward}.
For the computational cost, we report Normalized Sampling FLOPs,
defined as $N_{\mathrm{flops}} = F_{\mathrm{sample}} / F_{\mathrm{ref}}$,
where $F_{\mathrm{sample}}$ sums the accumulated FLOPs of the distilled student model to generate one image, including guidance
branches, and $F_{\mathrm{ref}}$ denotes the FLOPs of one forward pass of the
corresponding teacher backbone under the same input setting.

\textbf{Training Configuration.}
For C2I, we adopt the same teacher configuration as FACM~\citep{peng2025facm}, and perform full-parameter distillation on our phase-wise experts, which have a total of 676M parameters. Optimization is executed using AdamW with a batch size of $1024$, a learning rate of $10^{-4}$.
For T2I tasks, the SD3.5-Medium configuration undergoes full-parameter phase-wise finetuning, with 1.11B parameters per phase. 
For FLUX.1-dev~\citep{flux2024}, we utilize rank-$64$ LoRA experts, pairing a shared 5.82B backbone with two 0.16B LoRA experts (6.14B in total).
For Qwen-Image~\citep{wu2025qwen}, we use rank-32 LoRA experts, pairing a shared 10.25B backbone with two 0.15B LoRA experts (10.55B in total).
The experts for different temporal phases are trained independently. Detailed training hyperparameters are summarized in \cref{tab:train_config} of \textbf{Appendix {\ref{app:training}}}.

\textbf{Compared Methods.}
For C2I, we compare with consistency-style models (iCT~\citep{song2023improved}), mean-velocity / rectified-flow one-step learners (Shortcut~\citep{frans2024one}, MeanFlow~\citep{geng2025mean}, $\alpha$-Flow~\citep{zhang2025alphaflow}, FACM~\citep{peng2025facm}, iMF~\citep{geng2025improved}), and report multi-step methods  (SiT-XL/2~\citep{ma2024sit}, DiT-XL/2~\citep{peeblesScalableDiffusionModels2023}, REPA~\citep{yu2024representation}, LightningDiT~\citep{yao2025reconstruction}) for reference.
iCT, Shortcut, MeanFlow, $\alpha$-Flow, and iMF are trained directly from scratch, while FACM and PVD are distilled from a pretrained teacher.
For T2I, based on the given backbone model (SD3.5-Medium, FLUX.1-dev, and Qwen-Image), we compare against adversarial distillation (LADD~\citep{sauer2024fast}), distribution-matching distillation (DMD2~\citep{yin2024improved}, SenseFlow~\citep{ge2025senseflow}), self-adversarial / paired-flow approaches (TwinFlow~\citep{cheng2025twinflow}, SWD~\citep{starodubcev2025scale}), and recent few-step / one-step flow distillations (ArcFlow~\citep{yang2026arcflow}, Pi-Flow~\citep{chen2025pi}, TDD~\citep{wang2025target}), together with strong FLUX-oriented acceleration baselines (Hyper-FLUX~\citep{ren2024hyper}, FLUX-Turbo-$\alpha$~\citep{flux_turbo_alpha}).

\subsection{Class-Conditional Image Generation}
\label{sec:exp_4.2}
\begin{table*}[t]
    \centering

    \begin{minipage}[t]{0.48\textwidth}
    \caption{\textbf{Class-conditional image generation results on ImageNet256.} Multi-step generation results of representative methods are reported for reference. Best and second-best accelerated results are highlighted in \textbf{bold} and \underline{underlined}, respectively.}
    \label{tab:c2i}
    \centering
    \scriptsize
    \setlength{\tabcolsep}{3.2pt}
    \renewcommand{\arraystretch}{1.05}
    \resizebox{\linewidth}{!}{%
    \begin{tabular}{@{}lcccc@{}}
    \toprule
    \textbf{Methods} & \textbf{Params} & \textbf{$N_{\mathrm{flops}}$} & \textbf{FID-50K $\downarrow$} & \textbf{IS $\uparrow$} \\
    \midrule
    \rowcolor{gray!15} \multicolumn{5}{l}{\textit{Multi-step Generation}} \\
    SiT-XL/2~\citep{ma2024sit} & 675M & 500.00 & 2.06 & 278.24 \\
    DiT-XL/2~\citep{peeblesScalableDiffusionModels2023} & 675M & 500.00 & 2.27 & 270.30 \\
    REPA~\citep{yu2024representation} & 675M & 500.00 & 1.42 & 284.00 \\
    LightningDiT~\citep{yao2025reconstruction} & 675M & 500.00 & 1.35 & 295.30 \\
    \midrule
    \rowcolor{gray!15} \multicolumn{5}{l}{\textit{Accelerated Generation}} \\
    iCT~\citep{song2023improved} & 675M & 1.00 & 34.24 & - \\
    Shortcut~\citep{frans2024one} & 675M & 1.00 & 10.60 & - \\
    MeanFlow~\citep{geng2025mean} & 676M & 1.00 & 3.43 & - \\
    $\alpha$-Flow~\citep{zhang2025alphaflow} & 675M & 1.00 & 2.58 & - \\
    FACM~\citep{peng2025facm} & 675M & 1.00 & 1.76 & - \\
    iMF~\citep{geng2025improved} & 610M & 1.00 & \underline{1.72} & \underline{282.00} \\
    \textbf{PVD (Ours)} & 676M & 1.00 & \textbf{1.48} & \textbf{295.89} \\
    \bottomrule
    \end{tabular}
    }
    \end{minipage}\hfill%
    \begin{minipage}[t]{0.48\textwidth}
    \caption{\textbf{Ablation studies on class-conditional generation.} We vary the temporal split, overlap ratio, and expert composition. The reference configuration is listed first, followed by the grouped variants. The best result within each component is highlighted in \textbf{bold}.}
    \label{tab:c2i_ablation}
    \centering
    \scriptsize
    \setlength{\tabcolsep}{2pt}
    \renewcommand{\arraystretch}{0.90}
    \resizebox{\linewidth}{!}{%
    \begin{tabular}{@{}ccccc@{}}
    \toprule
    \textbf{Split} & \textbf{Overlap} & \makecell{\textbf{Expert}\\\textbf{Composition}} & \textbf{FID-50K} $\downarrow$ & \textbf{IS} $\uparrow$ \\
    \midrule
    $[0.5, 0.5]$ & 0.25 & Early $\rightarrow$ Late & 12.89 & 126.70 \\
    \midrule
    \rowcolor{gray!15}\multicolumn{5}{l}{\textit{Varying Time Intervals Split}} \\
    $[0.3, 0.7]$ & 0.25 & Early $\rightarrow$ Late & 13.34 & 124.62 \\
    $[0.4, 0.6]$ & 0.25 & Early $\rightarrow$ Late & \textbf{11.51} & \textbf{132.12} \\
    $[0.6, 0.4]$ & 0.25 & Early $\rightarrow$ Late & 13.36 & 123.37 \\
    \midrule
    \rowcolor{gray!15}\multicolumn{5}{l}{\textit{Varying Overlap Ratio}} \\
    $[0.5, 0.5]$ & 0.00 & Early $\rightarrow$ Late & \textbf{12.10} & 127.09 \\
    $[0.5, 0.5]$ & 0.10 & Early $\rightarrow$ Late & 12.36 & \textbf{127.50} \\
    $[0.5, 0.5]$ & 1.00 & Early $\rightarrow$ Late & 18.69 & 101.89 \\
    \midrule
    \rowcolor{gray!15}\multicolumn{5}{l}{\textit{Varying Expert Composition}} \\
    - & - & Shared $\times 2$ & 26.62 & 77.12 \\
    \bottomrule
    \end{tabular}
    }
    \end{minipage}
\end{table*}

\textbf{Main Results.}
\cref{tab:c2i} reports the quantitative results on ImageNet. 
PVD achieves an FID of \textbf{1.48} and an IS of \textbf{295.89}, significantly outperforming previous methods such as iMF (\underline{1.72} FID, \underline{282.00} IS) and FACM (1.76 FID).
Notably, it is highly competitive even compared with the multi-step references.  
Visualizations are provided in \textbf{Appendix~{\ref{app:visuals}}}, where we can see that PVD consistently preserves coherent global structures and intricate fine details.
These results support the advantages of decomposing the global diffusion trajectory into localized phase targets.

{\textbf{Ablation Studies.}
We conduct ablations with a 133M-parameter model along three axes: (i) the temporal partition $(t_1, 1{-}t_1)$, (ii) the overlap between adjacent training intervals, and (iii) expert composition, comparing ordered phase-specific experts with two evaluations of a shared student.
As shown in \cref{tab:c2i_ablation}, the asymmetric $[0.4, 0.6]$ split achieves the best FID and IS, suggesting that the two phases play unequal roles: the early expert establishes global structure from high-noise inputs, whereas the late expert requires a broader interval to refine local details.
Reducing the overlap improves performance, while full overlap substantially degrades both metrics, indicating that excessive sharing weakens phase specialization.
Finally, replacing the ordered experts with two evaluations of a shared student degrades FID from 12.89 to 26.62 and IS from 126.70 to 77.12. Thus, repeated computation alone cannot replace phase-specific capacity allocation.
Based on these results, we use the $[0.4, 0.6]$ split, zero overlap, and early-to-late expert composition by default.}

\begin{table*}[t]
    \centering
    \caption{\textbf{Quality and efficiency of distilled text-to-image generation.} 
    The right two columns report active parameters (B) and peak VRAM (GiB), together with PVD's reductions relative to its teacher model.
    Higher is better for quality metrics; the best and second-best results among distilled methods within each backbone group are highlighted in \textbf{bold} and \underline{underlined}, respectively.}
    \label{tab:t2i_all}
    \setlength{\tabcolsep}{4.5pt}
    \resizebox{\textwidth}{!}{
    \begin{tabular}{lccccccccc !{\color{gray!55}\vrule width 0.35pt} ccc}
    \toprule
    \multirow{2.5}{*}{\textbf{Model}} & \multirow{2.5}{*}{\textbf{GenEval}} & \multirow{2.5}{*}{\textbf{DPG}} & \multirow{2.5}{*}{\textbf{WISE}} & {\textbf{QI}} & \multicolumn{2}{c}{\textbf{TIIF-Bench}} & \textbf{Aesthetic} & \textbf{Pick} & \textbf{Image} & \multirow{2.5}{*}{\textbf{$N_{\mathrm{flops}}$}} & \multirow{2.5}{*}{\makecell{\textbf{Active}\\\textbf{Params.}}} & \multirow{2.5}{*}{\makecell{\textbf{Peak}\\\textbf{VRAM}}} \\
    \cmidrule(lr){6-7}
    & & & & {\textbf{Bench}} & \textbf{short} & \textbf{long} & \textbf{Score} & \textbf{Score} & \textbf{Reward} & & & \\
    \midrule
    \rowcolor{gray!15} \textbf{SD3.5 Medium}~\citep{esser2024scaling} & 0.6829 & 84.49 & 0.45 & {43.19} & 0.7166 & 0.7190 & 5.7212 & 21.45 & 0.6127 & 100.00 & 2.24 & 4.99 \\
    \midrule
    LADD~\citep{sauer2024fast} & 0.0941 & 26.46  & 0.05 & {6.55} & 0.2232 & 0.2019 & 4.3533 & 18.40 & -2.0724 & 1.00 & 2.27 & 4.86 \\
    DMD2~\citep{yin2024improved} & 0.4563 & 46.15 & 0.31 & {29.27} & 0.5410 & 0.4931 & 4.7885 & 20.09 & -0.1003 & 1.00 & 2.24 & 4.81 \\
    TwinFlow~\citep{cheng2025twinflow} & 0.6171 & 72.81 & \underline{0.37} & {30.55} & 0.5856 & 0.6094 & \underline{5.0111} & 20.26 & -0.0154 & 1.00 & 2.25 & 4.83 \\
    SenseFlow~\citep{ge2025senseflow} & 0.5222 & \textbf{81.98} & 0.31 & {29.43} & 0.6273 & 0.5957 & 4.9707 & \underline{20.69} & \underline{0.4834} & 1.00 & 2.24 & 4.81 \\
    SWD~\citep{starodubcev2025scale} & \underline{0.6286} & 79.23 & 0.36 & {\underline{31.44}} & \underline{0.6493} & \underline{0.6452} & 4.9440 & 20.53 & 0.3578 & 1.11 & 2.32 & 5.28 \\
    \rowcolor{gray!8} \textbf{PVD (Ours)} & \textbf{0.6948} & \underline{81.81} & \textbf{0.40} & {\textbf{35.51}} & \textbf{0.6561} & \textbf{0.6638} & \textbf{5.4618} & \textbf{20.84} & \textbf{0.4947} & \multicolumn{1}{!{\color{gray!55}\vrule width 0.35pt}>{\columncolor{gray!8}}c}{0.99} & \multicolumn{1}{>{\columncolor{gray!8}[\tabcolsep][\dimexpr\tabcolsep+3pt\relax]}c}{\makecell{\textbf{1.10}\\{\scriptsize \textbf{$\downarrow$50.89\%}}}} & \multicolumn{1}{!{\color{gray!8}\vrule width 0.6pt}>{\columncolor{gray!8}[\dimexpr\tabcolsep+3pt\relax][\tabcolsep]}c}{\makecell{\textbf{2.67}\\{\scriptsize \textbf{$\downarrow$46.49\%}}}} \\
    \midrule
    \rowcolor{gray!15} \textbf{FLUX.1 dev}~\citep{flux2024} & 0.6675 & 83.95 & 0.50 & {43.52} & 0.6760 & 0.7144 & 5.9286 & 21.81 & 0.8139 & 50.00 & 11.90 & 22.64 \\
    \midrule
    Hyper-FLUX~\citep{ren2024hyper} & 0.3437 & 53.50 & 0.17 & {13.85} & 0.4614 & 0.4382 & 4.5131 & 19.56 & -1.3969 & 1.00 & 12.25 & 23.29 \\
    FLUX-Turbo-$\alpha$~\citep{flux_turbo_alpha} & 0.1456 & 33.90 & 0.09 & {8.47} & 0.2964 & 0.2843 & 4.2431 & 18.94 & -2.0362 & 1.00 & 12.25 & 23.29 \\
    TDD~\citep{wang2025target} & 0.2037 & 50.63 & 0.14 & {13.71} & 0.4332 & 0.3862 & 4.2815 & 19.25 & -1.7594 & 0.94 & 12.06 & 22.91 \\
    Pi-Flow~\citep{chen2025pi} & 0.4817 & 71.40 & 0.23 & {20.38} & 0.5840 & 0.5822 & 4.3435 & 20.04 & -0.3111 & 1.07 & 12.53 & 23.87 \\
    ArcFlow~\citep{yang2026arcflow} & 0.0026 & 37.70 & 0.10 & {3.87} & 0.3392 & 0.2732 & 3.4819 & 18.15 & -2.2052 & 1.07 & 12.53 & 23.89 \\
    SenseFlow~\citep{ge2025senseflow} & 0.4551 & 67.01 & 0.25 & {21.69} & 0.5357 & 0.5528 & 4.8732 & 20.08 & -0.4266 & 1.00 & 11.89 & 22.62 \\
    SWD~\citep{starodubcev2025scale} & \underline{0.6285} & \underline{82.43} & \underline{0.37} & {\underline{39.71}} & \underline{0.6800} & \underline{0.6874} & \underline{5.3846} & \underline{21.22} & \underline{0.6562} & 1.05 & 12.29 & 24.45 \\
    \rowcolor{gray!8} \textbf{PVD (Ours)} & \textbf{0.6532} & \textbf{82.83} & \textbf{0.44} & {\textbf{43.14}} & \textbf{0.6819} & \textbf{0.6987} & \textbf{5.7145} & \textbf{21.32} & \textbf{0.6747} & \multicolumn{1}{!{\color{gray!55}\vrule width 0.35pt}>{\columncolor{gray!8}}c}{1.01} & \multicolumn{1}{>{\columncolor{gray!8}[\tabcolsep][\dimexpr\tabcolsep+3pt\relax]}c}{\makecell{\textbf{5.98}\\{\scriptsize \textbf{$\downarrow$49.75\%}}}} & \multicolumn{1}{!{\color{gray!8}\vrule width 0.6pt}>{\columncolor{gray!8}[\dimexpr\tabcolsep+3pt\relax][\tabcolsep]}c}{\makecell{\textbf{12.28}\\{\scriptsize \textbf{$\downarrow$45.76\%}}}} \\
    \midrule
    \rowcolor{gray!15} \textbf{Qwen-Image}~\citep{wu2025qwen} & 0.8720 & 89.10 & 0.64 & {48.78} & 0.8811 & 0.8645 & 5.8584 & 21.90 & 0.9852 & 100.00 & 20.43 & 38.42 \\
    \midrule
    ArcFlow~\citep{yang2026arcflow} & 0.0913 & 43.38 & 0.20 & {4.88} & 0.6510 & 0.5638 & 2.9363 & 18.51 & -1.8307 & 1.07 & 21.37 & 40.36 \\
    Pi-Flow~\citep{chen2025pi} & 0.5636 & 68.45 & 0.30 & {18.08} & 0.8073 & 0.8270 & 4.0918 & 19.72 & -0.4517 & 1.07 & 21.37 & 40.36 \\
    TwinFlow~\citep{cheng2025twinflow} & \underline{0.8338} & \textbf{86.75} & \underline{0.57} & {\underline{46.16}} & \underline{0.8343} & \underline{0.8452} & \underline{5.5880} & \underline{21.42} & \underline{0.9042} & 1.00 & 20.44 & 38.48 \\
    \rowcolor{gray!8} \textbf{PVD (Ours)} & \textbf{0.8846} & \underline{86.54} & \textbf{0.59} & {\textbf{46.85}} & \textbf{0.8604} & \textbf{0.8687} & \textbf{5.7526} & \textbf{21.66} & \textbf{0.9063} & \multicolumn{1}{!{\color{gray!55}\vrule width 0.35pt}>{\columncolor{gray!8}}c}{1.02} & \multicolumn{1}{>{\columncolor{gray!8}[\tabcolsep][\dimexpr\tabcolsep+3pt\relax]}c}{\makecell{\textbf{10.40}\\{\scriptsize \textbf{$\downarrow$49.10\%}}}} & \multicolumn{1}{!{\color{gray!8}\vrule width 0.6pt}>{\columncolor{gray!8}[\dimexpr\tabcolsep+3pt\relax][\tabcolsep]}c}{\makecell{\textbf{19.84}\\{\scriptsize \textbf{$\downarrow$48.36\%}}}} \\
    \bottomrule
    \end{tabular}
    }
    \vspace{-3mm}
\end{table*}

\subsection{Text-to-Image Generation}
\label{sec:exp_4.3}
\textbf{Main Results.}
The overall quantitative comparisons of T2I models across all benchmarks are summarized in \cref{tab:t2i_all}, with category-wise results reported in \textbf{Appendix~{\ref{app:category}}}.

Across the three backbone groups, PVD consistently achieves most of the best results on various evaluation dimensions. For prompt alignment, such as GenEval and TIIF-Bench, and general generation capability, such as Qwen-Image-Bench, monolithic models have to jointly handle semantic grounding, layout composition, and visual synthesis within a single forward process. In contrast, PVD accomplishes this task with half-sized phase-wise experts: the early expert focuses on semantic grounding and global layout construction, while the late expert refines perceptual details and high-frequency visual patterns.
This divide-and-conquer design also improves perceptual quality. On aesthetic and perceptual metrics such as Aesthetic Score, PickScore, and ImageReward, conventional single-forward models produce over-smoothed results, whereas PVD allows the second expert to enhance local details after the global structure has been established. As a result, PVD achieves a better balance between structural accuracy and perceptual sharpness.

The right two columns in \cref{tab:t2i_all} further show that these quality gains are achieved with substantially reduced deployment cost (the complete measurements are provided in \textbf{Appendix~{\ref{app:hardware}}}). Compared with the corresponding teachers, PVD reduces active parameters by 50.89\%, 49.75\%, and 49.10\%, and peak VRAM by 46.49\%, 45.76\%, and 48.36\% on SD3.5-Medium, FLUX.1-dev, and Qwen-Image, respectively. Thus, PVD preserves strong generation quality while approximately halving the active model size and memory footprint across all three backbones.

\begin{figure}[!htbp]
    \centering
    \includegraphics[width=1\textwidth]{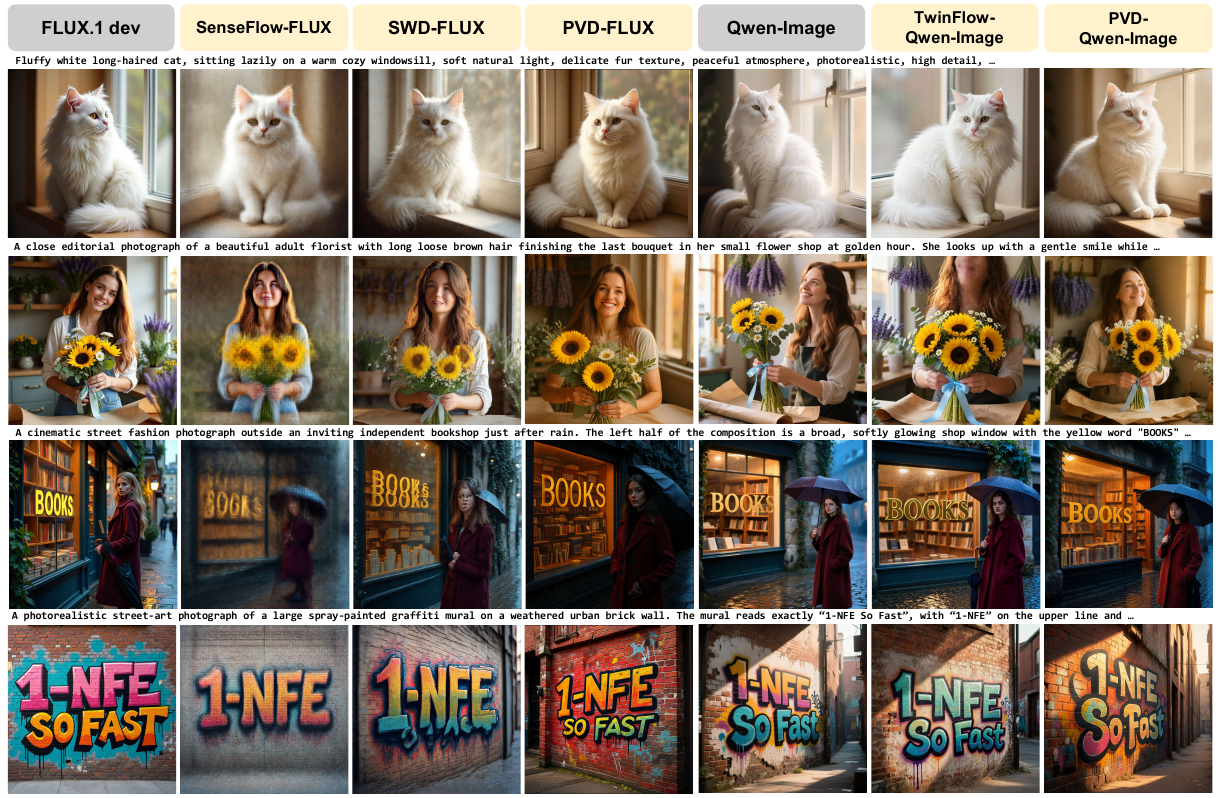}
    \vspace{-4mm}
    \caption{\textbf{Visual comparisons of T2I generation.} Compared with existing distillation methods, PVD preserves much clearer structure and finer details under a single-forward budget, demonstrating close T2I generation quality to its multi-step teachers.}
    \label{fig:t2i_compare}
    \vspace{-3mm}
\end{figure}

The visual comparisons of major T2I models are shown in \cref{fig:t2i_compare}. Existing distillation methods suffer from severe perceptual degradation. They typically generate oversmoothed outputs and frequently introduce noticeable structural artifacts along object boundaries.
In contrast, our PVD successfully overcomes these limitations, delivering high-fidelity images that closely mirror the multi-step teachers. This corroborates the effectiveness of our decoupled detail-refinement phase in synthesizing complex scenes.
Additional visual comparisons and intermediate-to-final visualizations of PVD distilled from Qwen-Image are provided in \textbf{Appendix~\ref{app:t2i_visuals}}.

\textbf{Ablation Studies.}
We ablate PVD along three axes: capacity allocation, adversarial refinement, and expert composition. Detailed results are provided in \textbf{Appendix~{\ref{app:exp_results}}}. Under compute-controlled comparisons, the compact phase-wise experts outperform a monolithic full-size student even when the latter costs twice the computation. In addition, adversarial refinement provides a complementary quality gain, and expert reuse and order swaps further reveal that the learned mappings are directional and non-interchangeable. Together, these results indicate that PVD's gains come from allocating capacity to ordered phase-specific experts rather than simply increasing the number of evaluations.

\begin{figure}[!htbp]
    \centering
    \includegraphics[width=1\textwidth]{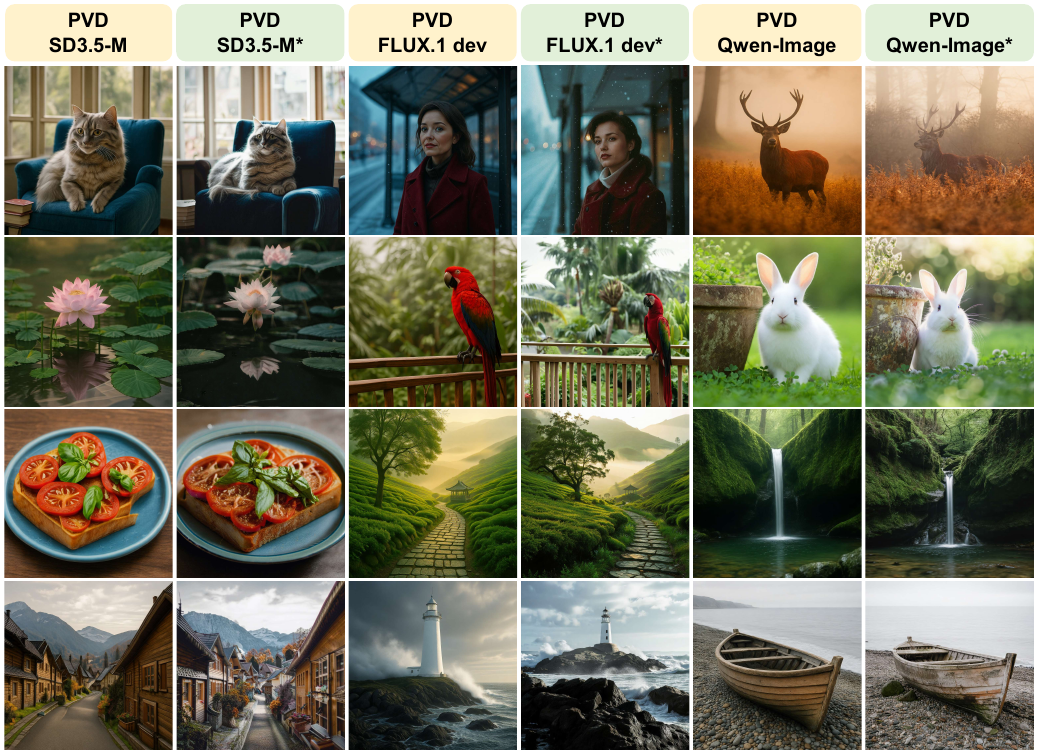}
    \caption{{\textbf{Visual comparisons after further training on Unsplash data.} The models further trained on Unsplash images are marked with $^{*}$, which can generate more authentic photorealistic images with improved lighting, shadows, composition, and fine-grained details.}}
    \label{fig:t2i_futher_train}
    \vspace{-2mm}
\end{figure}

\begin{table}[!t]
    \centering
    \captionsetup{skip=4pt}
    \caption{{\textbf{Results of further training on Unsplash data.} The models further trained on Unsplash images are marked with $^{*}$. Better results per backbone are highlighted in \textbf{bold}.}}
    \label{tab:t2i_further_training}
    \resizebox{0.8\textwidth}{!}{
        \setlength{\tabcolsep}{10pt}
        \renewcommand{\arraystretch}{1.10}
        \begin{tabular}{lcccc}
            \toprule
            \textbf{Backbone} & \textbf{Weights} & \textbf{Aesthetic score} $\uparrow$ & \textbf{DreamSim diversity} $\uparrow$ & \textbf{GenEval} $\uparrow$ \\
            \midrule
            \multirow{2}{*}{SD3.5-Medium} & PVD & 6.4753 & 0.1473 & \textbf{0.6948} \\
            & \cellcolor{gray!10}PVD$^{*}$ & \cellcolor{gray!10}\textbf{6.5993} & \cellcolor{gray!10}\textbf{0.1832} & \cellcolor{gray!10}0.6697 \\
            \midrule
            \multirow{2}{*}{FLUX.1-dev} & PVD & 6.4471 & 0.1457 & \textbf{0.6532} \\
            & \cellcolor{gray!10}PVD$^{*}$ & \cellcolor{gray!10}\textbf{6.7941} & \cellcolor{gray!10}\textbf{0.2238} & \cellcolor{gray!10}0.6334 \\
            \midrule
            \multirow{2}{*}{Qwen-Image} & PVD & 6.6463 & 0.1079 & \textbf{0.8846} \\
            & \cellcolor{gray!10}PVD$^{*}$ & \cellcolor{gray!10}\textbf{6.7289} & \cellcolor{gray!10}\textbf{0.1432} & \cellcolor{gray!10}0.8594 \\
            \bottomrule
        \end{tabular}
    }
    \vspace{-4mm}
\end{table}

\textbf{Further Training on Better-Quality Data.}
In the above experiments, we use the same widely adopted training datasets as prior methods (\textbf{Appendix~\ref{app:training}}).
However, these datasets consist primarily of synthetic images, whose quality and diversity may constrain the performance attainable through distillation.
To explore PVD's potential, we further train the distilled models on Unsplash~\citep{unsplashdataset}, which provides high-resolution real-world photographs characterized by diverse scenes, natural lighting, authentic photographic quality, and carefully composed scenes.

We perform further training on SD3.5-Medium, FLUX.1-dev, and Qwen-Image, with quantitative results summarized in \cref{tab:t2i_further_training}. GenEval scores decrease slightly. This is because Unsplash data provide rich supervision for lighting, texture, and composition, but do not provide equally strong supervision for the precise object counts, attribute bindings, and spatial relations evaluated by GenEval. On metrics of aesthetic quality and diversity, however, all three backbones improve consistently. The Aesthetic scores~\citep{aesthetic_predictor_v2_5} and DreamSim diversity~\citep{NEURIPS2023_9f09f316} are computed using 200 prompts with 10 random seeds per prompt.
These results show that further training improves aesthetics while substantially broadening the variety of generated outputs.

The visual comparisons in \cref{fig:t2i_futher_train} further demonstrate marked improvements in photographic realism. Compared with their initial distilled counterparts, the further-trained models produce more natural lighting and shadows, richer material textures, and stronger spatial depth.
These results highlight PVD's ability to effectively adapt the visual characteristics of high-quality training data into the generation process, extending model capacity for perceptual refinement beyond mere teacher imitation.
Additional visualizations after further training are provided in \cref{fig:t2i_appendix_further_train} in \textbf{Appendix~\ref{app:t2i_visuals}}.

\section{Conclusion}
\label{sec:conclusion}
{In this work, we introduce Phase-wise Velocity Distillation (PVD) for high-quality image generation within a single full-backbone-forward compute budget. Rather than assigning the heterogeneous coarse-to-fine transport to one monolithic student, PVD partitions the generation timeline into coarse and fine phases, models each transition through its average velocity, and assigns a dedicated half-sized expert to each phase. This phase-wise allocation separates structural composition from detail refinement while keeping the cumulative computation equivalent to one full-backbone evaluation.
Experiments on both class-conditional and text-to-image generation show that two ordered phase-specific experts outperform a single full-size monolithic student under comparable computation. PVD achieves an FID of 1.48 on ImageNet $256\times256$. When applied to Stable Diffusion 3.5-Medium, FLUX.1-dev, and Qwen-Image, it remains competitive with the corresponding multi-step teachers, significantly outperforming prior distillation methods. Across these T2I backbones, PVD reduces active parameters by 49.10--50.89\% and peak VRAM by 45.76--48.36\%. These results indicate that PVD is an effective way to improve the quality--efficiency trade-off of few-step image generation.}

\textbf{Limitations}.
For the sake of simplicity, we empirically split the teacher backbone into two half-sized experts, yet the optimal number of experts and more effective partitioning strategies can be further investigated. 
Although effective, PVD can still fail for cases such as dense object counting and uncommon spatial relations, which are also challenging cases for the teacher backbones. Representative examples of failure cases are provided in \textbf{Appendix~\ref{app:failure_cases}}.

{
    \small
    \bibliographystyle{IEEEtran}
    \bibliography{references}
}
\clearpage
\appendix

\setlength{\textfloatsep}{10pt plus 2pt minus 2pt}
\setlength{\floatsep}{10pt plus 2pt minus 2pt}
\setlength{\intextsep}{10pt plus 2pt minus 2pt}

\section*{Appendix}
\phantomsection
\label{app:appendix}

\begingroup
\small
\setlength{\parindent}{0pt}
\setlength{\parskip}{0pt}

\newlength{\appnumwidth}
\setlength{\appnumwidth}{2.2em}

\newcommand{\appcontentsline}[4]{%
  \noindent
  \makebox[\appnumwidth][l]{\hyperref[#2]{\textbf{#1}}}%
  \hyperref[#2]{\textbf{#3}}%
  \nobreak\leaders\hbox to 0.75em{\hss.\hss}\hfill
  \nobreak\hyperref[#2]{\pageref*{#2}}%
  \par
  \vspace{0.08em}
  \hspace*{\appnumwidth}%
  {\footnotesize\textcolor{gray}{#4}}%
  \par
  \vspace{0.45em}
}

{\large\bfseries Contents}\par
\vspace{0.35em}
\hrule
\vspace{0.65em}

\appcontentsline{\ref{app:related_comparison}}{app:related_comparison}
{Further Discussions on Related Distillation Methods}
{Corresponding to \cref{sec:related_work} in the main paper.}

\appcontentsline{\ref{app:analysis}}{app:analysis}
{Mathematical Analysis on PVD}
{Corresponding to \cref{sec:expert_init} in the main paper.}

\appcontentsline{\ref{app:algo}}{app:algo}
{Training and Inference Procedure of PVD}
{Corresponding to \cref{sec:discriminator} in the main paper.}

\appcontentsline{\ref{app:training}}{app:training}
{Detailed Experimental Settings}
{Corresponding to \cref{sec:exp_setup} in the main paper.}

\appcontentsline{\ref{app:visuals}}{app:visuals}
{More C2I Visualizations}
{Corresponding to \cref{sec:exp_4.2} in the main paper.}

\appcontentsline{\ref{app:category}}{app:category}
{Category-wise T2I Results}
{Corresponding to \cref{sec:exp_4.3} in the main paper.}

\appcontentsline{\ref{app:hardware}}{app:hardware}
{Detailed Cost Measurements}
{Corresponding to \cref{sec:exp_4.3} in the main paper.}

\appcontentsline{\ref{app:t2i_visuals}}{app:t2i_visuals}
{More T2I Visualizations}
{Corresponding to \cref{sec:exp_4.3} in the main paper.}

\appcontentsline{\ref{app:exp_results}}{app:exp_results}
{Ablation Studies on T2I Task}
{Corresponding to \cref{sec:exp_4.3} in the main paper.}

\appcontentsline{\ref{app:failure_cases}}{app:failure_cases}
{Failure Cases}
{Corresponding to \cref{sec:conclusion} in the main paper.}

\appcontentsline{\ref{app:broader_impacts}}{app:broader_impacts}
{Broader Impacts}
{Corresponding to \cref{sec:conclusion} in the main paper.}

\vspace{0.25em}
\hrule
\endgroup

\vspace{1em}

\clearpage
\begingroup
\section{Further Discussions on Related Distillation Methods}
\label{app:related_comparison}
\begin{table}[H]
    \centering
    \caption{
    \textbf{Comparison with representative related methods} in terms of temporal localization of the training objective, parameter
    specialization, average-velocity modeling, and explicit phase-factorized
    inference. 
    ``Init.'' indicates that average-velocity learning is used only for student
    initialization rather than retained as the final phase-specific transport
    objective.
    }
    \label{tab:related_phase_comparison}
    \footnotesize
    \setlength{\tabcolsep}{3.0pt}
    \renewcommand{\arraystretch}{1.12}

    \resizebox{\columnwidth}{!}{%
    \begin{tabular}{@{}lcccccc@{}}
        \toprule
        &
        \multicolumn{4}{c}{\textbf{Transport factorization}}
        &
        \multicolumn{2}{c}{\textbf{Inference efficiency}}
        \\
        \cmidrule(lr){2-5}
        \cmidrule(l){6-7}

        \textbf{Method}
        &
        \makecell{\textbf{Phase-restricted}\\\textbf{training}}
        &
        \makecell{\textbf{Phase-specific}\\\textbf{parameters}}
        &
        \makecell{\textbf{Average-velocity}\\\textbf{objective}}
        &
        \makecell{\textbf{Phase-factorized}\\\textbf{inference}}
        &
        \makecell{\textbf{Reduced active}\\\textbf{params. \& VRAM}}
        &
        \makecell{\textbf{Few-step}\\\textbf{distillation}}
        \\
        \midrule

        MeanFlow
        & $\times$
        & $\times$
        & \checkmark
        & $\times$
        & $\times$
        & \checkmark
        \\

        Hierarchical Distill.
        & $\times$
        & $\times$
        & Init.
        & $\times$
        & $\times$
        & \checkmark
        \\

        TimeStep Master
        & \checkmark
        & \checkmark
        & $\times$
        & $\times$
        & $\times$
        & $\times$
        \\

        Glance
        & \checkmark
        & \checkmark
        & $\times$
        & $\times$
        & $\times$
        & $\times$
        \\

        Phased DMD
        & \checkmark
        & \checkmark
        & $\times$
        & \checkmark
        & $\times$
        & \checkmark 
        \\

        \rowcolor{gray!15}
        \textbf{PVD (Ours)}
        & \checkmark
        & \checkmark
        & \checkmark
        & \checkmark
        & \checkmark
        & \checkmark
        \\

        \bottomrule
    \end{tabular}%
    }
\end{table}

We further compare PVD with closely related methods that
instantiate different forms of temporal specialization, phase-wise distillation,
or average-velocity modeling.
As summarized in \cref{tab:related_phase_comparison}, we compare these methods from the following aspects: 
(i) whether the training objective is explicitly restricted to predefined
temporal phases,
(ii) whether different phases are associated with distinct trainable parameters,
(iii) whether average velocity is used as the primary transport target,
(iv) whether inference explicitly factorizes the full transport into a sequential
composition of phase-specific operators, and
(v) how the resulting design allocates active model capacity and cumulative
sampling computation.
This decomposition separates temporal specialization itself from
the compute-constrained transport factorization targeted by PVD.

MeanFlow~\citep{geng2025mean} learns interval-conditioned average velocities
with a single shared model and supports few-step generation. However, its
standard few-step inference allocates the shared model capacity to a single
global noise-to-data transport, rather than explicitly decomposing inference
into a sequence of phase-specific operators.
Hierarchical Distillation~\citep{cheng2025structure} uses MeanFlow-based
trajectory distillation to initialize a single-step student and subsequently
refines the same generator with distribution-level objectives; thus, the
hierarchy is primarily realized through the training procedure rather than
through temporally distinct inference operators.

TimeStep Master~\citep{zhuang2025timestep} and
Glance~\citep{dong2025glance} introduce timestep- or phase-specific lightweight
parameters, while retaining the underlying backbone during each iterative
denoising evaluation.
These methods therefore specialize parameters across temporal regimes, but do
not explicitly replace the global transport with a sequence of compact
phase-local average-velocity operators under a one-backbone-equivalent
cumulative compute budget.
Phased DMD~\citep{fan2025phaseddmd} directly decomposes generation into
subinterval-wise generators and sequential intermediate transports, but its
few-step inference configuration requires multiple generator
evaluations.

PVD differs from these methods in that it integrates phase-restricted
average-velocity learning, distinct compact phase experts, and explicit
sequential transport composition.
Instead of assigning the full model capacity to a single long-horizon mapping,
PVD temporally factorizes the transport and reallocates the same aggregate
sampling-compute budget across shorter phase-specific mappings.
Consequently, both the transport horizon and the active capacity per evaluation
are factorized, while the cumulative sampling cost remains approximately equal to one full-backbone forward pass.

\par
\endgroup
\clearpage

\section{Mathematical Analysis on PVD} \label{app:analysis}

We consider a continuous-time setting $t \in [0,1]$ with latent states $\bm{x}_t \in \mathbb{R}^d$ and marginal distributions $p_t$.
The teacher dynamics are governed by a probability-flow ODE:
$\frac{d\bm{x}_t}{dt} = v_\phi(\bm{x}_t,t),$
which induces a flow map $\Psi_{t \to t'}$ transporting $p_t$ to $p_{t'}$.
We denote the teacher velocity field by $v_\phi(\bm{x},t)$, and the student instantaneous velocity field by $v_\theta(\bm{x},t)$.
{In the analysis below, we distinguish the trajectory-induced average velocity $\widetilde{u}_\theta$ obtained by integrating $v_\theta$ from the directly parameterized phase expert $u_{\theta_i}$.}
Throughout the analysis, we assume that $v_\phi$ and $v_\theta$ are measurable and integrable along trajectories, and the teacher velocity field satisfies the regularity conditions: $\|v_\phi(\bm{x},t)\| \le C$ and $\|\nabla_{\bm{x}} v_\phi(\bm{x},t)\| \le L$, where $C$ is the global bound on the velocity magnitude and $L$ is the global Lipschitz constant in the spatial variable.

\begin{definition}[Instantaneous velocity / probability-flow ODE]
Let $p_0 = \mathcal{N}(\mathbf{0}, \mathbf{I})$ be the prior noise distribution and $p_1$ the data distribution.
Consider $\bm{x}_0 \sim p_0$, and let $\{\bm{x}_t\}_{t \in [0,1]}$ denote the trajectory induced by the teacher probability-flow ODE, there is:
\begin{equation}
  \frac{\mathrm{d}\bm{x}_t}{\mathrm{d}t} = v_\phi(\bm{x}_t,t).
\end{equation}
\end{definition}

\begin{definition}[Average velocity via integration]
{For any temporal phase $[t,t']$ with $\Delta t \coloneqq t'-t>0$, we define the teacher average velocity $u_\phi$ and the trajectory-induced backbone average velocity $\widetilde{u}_\theta$ by integrating their respective instantaneous velocity fields along the trajectory $\{\bm{x}_\tau\}_{\tau\in[t,t']}$:}
\begin{equation}
{u_\phi(\bm{x}_t,t,t') \coloneqq \frac{1}{\Delta t} \int_{t}^{t'} v_\phi(\bm{x}_\tau, \tau) \, \mathrm{d}\tau, \quad
\widetilde{u}_\theta(\bm{x}_t,t,t') \coloneqq \frac{1}{\Delta t} \int_{t}^{t'} v_\theta(\bm{x}_\tau, \tau) \, \mathrm{d}\tau.}
\end{equation}
{The directly parameterized phase expert $u_{\theta_i}$ is optimized by the subsequent phase-wise distillation objective and is not assumed to equal $\widetilde{u}_\theta$.}
\end{definition}

\begin{definition}[Estimation error of trajectory-induced average velocity]
{The trajectory-induced initialization error is defined as the discrepancy between the backbone-induced and teacher average velocities, i.e.,}
\begin{equation}
    {\widetilde{\mathcal{E}}_{\text{avg}} \coloneqq \left\| \widetilde{u}_\theta(\bm{x}_t,t,t') - u_\phi(\bm{x}_t,t,t') \right\|.}
\end{equation}
\end{definition}

\textbf{\cref{theorem:error_bound}} [Integration Error Bound]
{\textit{The discrepancy between the teacher average velocity and the trajectory-induced backbone average velocity is upper bounded by the temporal average of the instantaneous prediction error:}}
\begin{equation}
  {\left\| \widetilde{u}_\theta(\bm{x}_t,t,t') - u_\phi(\bm{x}_t,t,t') \right\|
  \le
  \frac{1}{t' - t} \int_{t}^{t'} \left\| v_\theta(\bm{x}_\tau, \tau) - v_\phi(\bm{x}_\tau, \tau) \right\| \mathrm{d}\tau.}
\end{equation}

\begin{proof}
{By substituting the integral definitions of $\widetilde{u}_\theta$ and $u_\phi$ into the trajectory-induced error, we have:}
\begin{equation}
{\left\| \widetilde{u}_\theta(\bm{x}_t,t,t') - u_\phi(\bm{x}_t,t,t') \right\| = \left\| \frac{1}{t' - t} \int_{t}^{t'} v_\theta(\bm{x}_\tau, \tau) \mathrm{d}\tau - \frac{1}{t' - t} \int_{t}^{t'} v_\phi(\bm{x}_\tau, \tau) \mathrm{d}\tau \right\|.}
\end{equation}
Using the linearity of the integral, we can move the positive scalar $\frac{1}{t'-t}$ out of the integration:
\begin{equation}
{\left\| \widetilde{u}_\theta(\bm{x}_t,t,t') - u_\phi(\bm{x}_t,t,t') \right\| = \frac{1}{t' - t} \left\| \int_{t}^{t'} \big( v_\theta(\bm{x}_\tau, \tau) - v_\phi(\bm{x}_\tau, \tau) \big) \mathrm{d}\tau \right\|.}
\end{equation}
Applying the continuous form of the triangle inequality for integrals (which states that the norm of an integral is bounded by the integral of the norm), we obtain:
\begin{equation}
    \frac{1}{t' - t} \left\| \int_{t}^{t'} \big( v_\theta(\bm{x}_\tau, \tau) - v_\phi(\bm{x}_\tau, \tau) \big) \mathrm{d}\tau \right\|
\le
\frac{1}{t' - t} \int_{t}^{t'} \left\| v_\theta(\bm{x}_\tau, \tau) - v_\phi(\bm{x}_\tau, \tau) \right\|  \mathrm{d}\tau.
\end{equation}
Therefore, the inequality holds:
\begin{equation}
{\left\| \widetilde{u}_\theta(\bm{x}_t,t,t') - u_\phi(\bm{x}_t,t,t') \right\| \le \frac{1}{t' - t} \int_{t}^{t'} \left\| v_\theta(\bm{x}_\tau, \tau) - v_\phi(\bm{x}_\tau, \tau) \right\| \mathrm{d}\tau.}
\end{equation}
End of proof.
\end{proof}

\begin{definition}[Wasserstein Discrepancy]
The discrepancy between two probability distributions $p_t$ and $p_{t'}$ is measured by the 2-Wasserstein distance:
\begin{equation}
    W_2(p_t, p_{t'}) \coloneqq \left( \inf_{\pi \in \Pi(p_t, p_{t'})} \mathbb{E}_{(\bm{x}, \bm{y}) \sim \pi} \left[ \|\bm{x} - \bm{y}\|^2 \right] \right)^{1/2}.
\end{equation}
\end{definition}

\textbf{\cref{theorem:smoothness}} [Temporal Marginal-Shift Bound].
\textit{Assume that the velocity field $v(\bm{x},t)$ is uniformly bounded in space and time, i.e., there exists a constant $C>0$ such that:
\[
\|v(\bm{x},t)\| \leq C, \quad \forall (\bm{x},t)\in \mathbb{R}^d \times [0,1].
\]
Let $p_t$ denote the marginal distribution of $\bm{x}_t$ induced by the PF-ODE.
Then for any interval $[t,t'] \subseteq [0,1]$, the induced marginal distributions satisfy:
\begin{equation}
W_2(p_t, p_{t'}) \le C |t'-t|.
\end{equation}
where $W_2$ denotes the 2-Wasserstein distance.}

\begin{proof}
Since the exact ODE mapping $\Psi_{t \to t'}$ pushes $p_t$ forward to $p_{t'}$, it forms a valid transport plan $\pi \in \Pi(p_t, p_{t'})$. Thus, the 2-Wasserstein distance is upper-bounded by the expected trajectory length:
\begin{equation}
    W_2^2(p_t, p_{t'}) \le \mathbb{E}_{\bm{x}_t \sim p_t} \left[ \left\| \Psi_{t \to t'}(\bm{x}_t) - \bm{x}_t \right\|^2 \right].
\end{equation}
By the fundamental theorem of calculus, we have $\Psi_{t \to t'}(\bm{x}_t) - \bm{x}_t = \int_{t}^{t'} v_\phi(\bm{x}_\tau, \tau) \, \mathrm{d}\tau$. Given the bounded velocity $\|v_\phi\| \le C$, we have:
\begin{equation}
    \left\| \int_{t}^{t'} v_\phi(\bm{x}_\tau, \tau) \, \mathrm{d}\tau \right\| \le \int_{t}^{t'} \left\| v_\phi(\bm{x}_\tau, \tau) \right\| \, \mathrm{d}\tau \le C(t' - t) = C \Delta t.
\end{equation}
Therefore, $W_2(p_t, p_{t'}) \le C \Delta t$, i.e., the discrepancy is linearly bounded by the interval length.

End of proof.
\end{proof}

\clearpage

\section{Training and Inference Procedure of PVD}
\label{app:algo}
The complete optimization and sampling pipelines are summarized in the following \textbf{Algorithms~\ref{alg:pvd_training} and~\ref{alg:pvd_inference}}, respectively.

\noindent
\begin{minipage}[t]{\linewidth}
  \centering
  \begin{algorithm}[H]
    \caption{Training procedure of PVD}
    \label{alg:pvd_training}
    \begin{algorithmic}
      \State {\bfseries Input:} Data distribution $p_1$, phases $\{\mathcal{I}_i\}_{i=1}^2$, experts $\{u_{\theta_i}\}_{i=1}^2$, teacher $v_\phi$, discriminators $\{D_{\omega_i}\}_{i=1}^2$, adversarial weights $\{\lambda_i\}_{i=1}^2$, discriminator-update count $K_D$
      \State {\bfseries Output:} Updated expert and discriminator parameters $\{\theta_i,\omega_i\}_{i=1}^2$
      \State Draw phase index $i \sim \mathcal{U}(\{1,2\})$
      \If{adversarial training is enabled}
          \State Freeze $\theta_i$
          \For{$k=1$ {\bfseries to} $K_D$}
              \State Sample $\bm{x}_1^{(k)}\sim p_1$, $\bm{x}_0^{(k)}\sim\mathcal{N}(\mathbf{0},\mathbf{I})$, and $t_k,t'_k\sim\mathcal{I}_i$ with $t_k<t'_k$
              \State $\bm{x}_{t_k}^{(k)} \leftarrow (1-t_k)\bm{x}_0^{(k)}+t_k\bm{x}_1^{(k)}$
              \State $\bm{x}_{\mathrm{real}}^{(k)} \leftarrow (1-t'_k)\bm{x}_0^{(k)}+t'_k\bm{x}_1^{(k)}$
              \State $\bm{x}_{\mathrm{fake}}^{(k)} \leftarrow \bm{x}_{t_k}^{(k)}+(t'_k-t_k)u_{\theta_i}(\bm{x}_{t_k}^{(k)},t_k,t'_k)$
              \State $\mathcal{L}_{D,k}^{(i)} \leftarrow \mathbb{E}\!\left[\max(0,1-D_{\omega_i}(\bm{x}_{\mathrm{real}}^{(k)})) + \max(0,1+D_{\omega_i}(\mathrm{stop\_grad}(\bm{x}_{\mathrm{fake}}^{(k)})))\right]$
              \State Update $\omega_i$ by descending $\nabla_{\omega_i}\mathcal{L}_{D,k}^{(i)}$
          \EndFor
          \State Freeze $\omega_i$ and unfreeze $\theta_i$
      \EndIf
      \State Sample a fresh generator minibatch $\bm{x}_1\sim p_1$, $\bm{x}_0\sim\mathcal{N}(\mathbf{0},\mathbf{I})$, and $t,t'\sim\mathcal{I}_i$ with $t<t'$
      \State $\bm{x}_t \leftarrow (1-t)\bm{x}_0+t\bm{x}_1$, $\bm{v}\leftarrow v_\phi(\bm{x}_t,t)$
      \State $F_{\text{FM}} \leftarrow u_{\theta_i}(\bm{x}_t,t,t)$
      \State $F_{\text{avg}},\partial F_{\text{avg}}\leftarrow\mathrm{JVP}\!\left(u_{\theta_i},(\bm{x}_t,t,t'),(\bm{v},1,0)\right)$
      \State $\bar{\bm{v}}\leftarrow\bm{v}+(t'-t)\partial F_{\text{avg}}$, $\bm{x}_{\mathrm{fake}}\leftarrow\bm{x}_t+(t'-t)F_{\text{avg}}$
      \If{adversarial training is enabled}
          \State $\mathcal{L}_{\text{adv}}^{(i)}\leftarrow-\mathbb{E}\!\left[D_{\omega_i}(\bm{x}_{\mathrm{fake}})\right]$
      \Else
          \State $\mathcal{L}_{\text{adv}}^{(i)}\leftarrow 0$
      \EndIf
      \State $\mathcal{L}_{G}^{(i)} \leftarrow \mathcal{L}_{\text{FM}}^{(i)}(F_{\text{FM}},\bm{v}) + \mathcal{L}_{\text{vel}}^{(i)}\!\left(F_{\text{avg}},\mathrm{stop\_grad}(\bar{\bm{v}})\right) + \lambda_i\mathcal{L}_{\text{adv}}^{(i)}$
      \State Update $\theta_i$ by descending $\nabla_{\theta_i}\mathcal{L}_{G}^{(i)}$
    \end{algorithmic}
  \end{algorithm}
\end{minipage}
\par\medskip
\begin{minipage}[t]{\linewidth}
  \centering
  \begin{algorithm}[H]
    \caption{Inference procedure of PVD}
    \label{alg:pvd_inference}
    \begin{algorithmic}
      \State {\bfseries Input:} phases $\{\mathcal{I}_i=[t_{i-1}, t_i]\}_{i=1}^2$, experts $\{u_{\theta_i}\}_{i=1}^2$, integration steps $\{K_i\}_{i=1}^2$
      \State {\bfseries Output:} Generated image $\hat{\bm{x}}$
      \State Sample initial noise $\bm{x}_{t_0} \sim \mathcal{N}(\mathbf{0}, \mathbf{I})$
      \For{$i = 1$ {\bfseries to} $N$}
          \State Build $t_{i-1}=\sigma_0^{(i)} < \cdots < \sigma_{K_i}^{(i)}=t_i$
          \For{$k = 0$ {\bfseries to} $K_i-1$}
              \State $\Delta \sigma \leftarrow \sigma_{k+1}^{(i)} - \sigma_k^{(i)}$
              \State $\bm{v}_k \leftarrow u_{\theta_i}(\bm{x}_{\sigma_k^{(i)}}, \sigma_k^{(i)}, \sigma_{k+1}^{(i)})$
              \State $\bm{x}_{\sigma_{k+1}^{(i)}} \leftarrow \bm{x}_{\sigma_k^{(i)}} + \Delta \sigma\, \bm{v}_k$
          \EndFor
      \EndFor
      \State Decode $\bm{x}_{t_N}$ to image space and set $\hat{\bm{x}}$ as the decoded output
      \State \Return $\hat{\bm{x}}$
    \end{algorithmic}
  \end{algorithm}
\end{minipage}
\par\medskip
\clearpage
\section{Detailed Experimental Settings}
\label{app:training}
The training process consists of two stages: (i) backbone initialization and (ii) phase-expert training.
\cref{tab:train_config} separates the backbone initialization settings from the subsequent phase-expert training configurations across the C2I and T2I settings. We normalize the transport trajectory to $[0,1]$ and partition it into an early interval $[0,0.4]$ and a late interval $[0.4,1]$. Each expert is trained independently using samples restricted to its assigned interval, with $(P_{\mathrm{mean}},P_{\mathrm{std}})$ rescaled accordingly.

For LightningDiT-XL/1 and SD3.5-Medium, both phase experts are optimized with full-parameter training. For FLUX.1-dev and Qwen-Image, we instead train lightweight LoRA experts on top of a shared backbone, using ranks 64 and 32, respectively.
For Qwen-Image, the backbone is further trained using DiffusionNFT~\cite{zheng2026diffusionnft} after the backbone initialization.
Besides, we use the same training datasets as TwinFlow~\citep{cheng2025twinflow}, which consist of BLIP3o-60k~\citep{chen2025blip3} and Echo-4o~\citep{ye2025echo} for phase-wise training.
We use the released weights of corresponding methods for comparisons.

\begin{table}[H]
\centering
\caption{\textbf{Training configurations for PVD.} Backbone initialization and phase-expert training are reported separately. Weight decay is $0$ and gradient clipping is $1.0$.}
\label{tab:train_config}
\footnotesize
\setlength{\tabcolsep}{5.0pt}
\renewcommand{\arraystretch}{1.06}

\begin{subtable}{\columnwidth}
\centering
\caption{Backbone configuration.}
\label{tab:train_config_backbone}
\resizebox{\columnwidth}{!}{%
\begin{tabular}{@{}lcccc@{}}
\toprule
\textbf{Setting} & \textbf{LightningDiT-XL/1} & \textbf{SD3.5-Medium} & \textbf{FLUX.1-dev} & \textbf{Qwen-Image} \\
\midrule
\rowcolor{gray!15}\multicolumn{5}{l}{\textit{Architecture and data}} \\
Parameterization & \multicolumn{4}{c}{Full} \\
Parameters & 338M & 1.1B & 5.82B & 10.25B \\
Trajectory interval & \multicolumn{4}{c}{$[0,1]$} \\
Resolution & 256 & $256\!\to\!1024$ & $256\!\to\!1024$ & $256\!\to\!1024$ \\
Dataset & ImageNet-256 & BLIP3o-long & BLIP3o-long & BLIP3o-long \\
\midrule
\rowcolor{gray!15}\multicolumn{5}{l}{\textit{Optimization and compute}} \\
Batch size & 1024 & $1024\!\to\!256$ & $512\!\to\!64$ & $64\!\to\!8$ \\
Time sampling $(P_{\mathrm{mean}},P_{\mathrm{std}})$ & \multicolumn{4}{c}{$(0.0,1.0)$} \\
Training steps & 140k & 35k & 35k & 35k \\
$LR_G$ & $10^{-4}$ & $2\!\times\!10^{-4}$ & $10^{-4}$ & $10^{-4}$ \\
$\beta_G$ & $(0.9,0.999)$ & $(0.99,0.999)$ & $(0.99,0.999)$ & $(0.99,0.999)$ \\
GPUs & $32\times$ A800 & $16\times$ PPU & $16\times$ PPU & $16\times$ PPU \\
\bottomrule
\end{tabular}%
}
\end{subtable}

\vspace{2mm}

\begin{subtable}{\columnwidth}
\centering
\caption{Phase-expert configuration.}
\label{tab:train_config_expert}
\resizebox{\columnwidth}{!}{%
\begin{tabular}{@{}lcccc@{}}
\toprule
\textbf{Setting} & \textbf{LightningDiT-XL/1} & \textbf{SD3.5-Medium} & \textbf{FLUX.1-dev} & \textbf{Qwen-Image} \\
\midrule
\rowcolor{gray!15}\multicolumn{5}{l}{\textit{Expert architecture and data}} \\
Parameterization & Full & Full & LoRA ($r=64$) & LoRA ($r=32$) \\
Trainable parameters & 338M & 1.1B & 0.16B & 0.15B \\
Phase length (early / late) & \multicolumn{4}{c}{$0.4$ / $0.6$} \\
Resolution & 256 & 1024 & 1024 & 1024 \\
Dataset & ImageNet-256 & BLIP3o-60k & BLIP3o-60k & \makecell[c]{BLIP3o-60k \\ Echo-4o} \\
\midrule
\rowcolor{gray!15}\multicolumn{5}{l}{\textit{Optimization and compute}} \\
Batch size & 1024 & 32 & 64 & 32 \\
Time sampling $(P_{\mathrm{mean}},P_{\mathrm{std}})$ & \multicolumn{4}{c}{$(-0.8,1.6)$, rescaled to each phase interval} \\
Training steps & 39k & 30k & 4k & 8k \\
$LR_G$ & $10^{-4}$ & $10^{-5}$ & $2\!\times\!10^{-5}$ & $10^{-5}$ \\
$LR_D$ & -- & $10^{-5}$ & $10^{-5}$ & $10^{-5}$ \\
$\beta$ & $(0.9,0.99)$ & $(0.9,0.99)$ & $(0.9,0.99)$ & $(0.9,0.99)$ \\
EMA $\sigma_{\mathrm{rel}}$ & 0.2 & 0.2 & 0.05 & 0.05 \\
GPUs & $32\times$ A800 & $8\times$ PPU & $16\times$ PPU & $16\times$ PPU \\
\bottomrule
\end{tabular}
}
\end{subtable}
\end{table}

\clearpage
\begingroup
\section{More C2I Visualizations}
\label{app:visuals}
\Cref{fig:c2i_1} presents additional class-conditional samples generated by PVD, complementing the ImageNet results in \cref{sec:exp_4.2}. The samples exhibit coherent global structures, clear object boundaries, and rich local textures.
\begin{figure}[H]
    \centering
    \captionsetup{font=small,skip=4pt}
    \includegraphics[width=0.96\columnwidth]{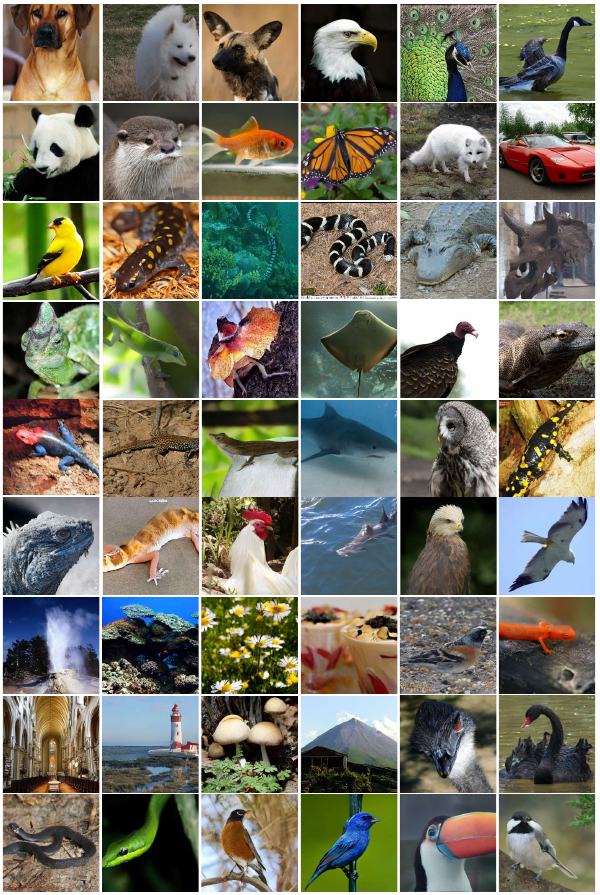}
    \caption{More class-conditional samples generated by PVD.}
    \label{fig:c2i_1}
\end{figure}

\par
\endgroup
\clearpage
\section{Category-wise T2I Results}
\label{app:category}
{In this section, we provide the detailed category-wise results on each benchmark, including GenEval (\cref{tab:geneval_results}), DPG-Bench (\cref{tab:dpg_bench_metrics}), WISE (\cref{tab:wise}), TIIF-Bench mini (\cref{tab:tiif_part1_visual}, \cref{tab:tiif_part2_reasoning}, and \cref{tab:tiif_part3_concepts}), and Qwen-Image-Bench (\cref{tab:qwenimage_bench}).}

\begin{table}[htbp]
\centering
\caption{Quantitative comparison on GenEval. Higher is better ($\uparrow$) for all metrics. The best and second-best accelerated results within each backbone group are in \textbf{bold} and \underline{underlined}, respectively.}
\label{tab:geneval_results}
\resizebox{\textwidth}{!}{
\begin{tabular}{l ccccccc}
\toprule
Model & Single Object & Two Object & Counting & Colors & Position & Attribute Binding & Overall $\uparrow$ \\
\midrule
\rowcolor{gray!15} SD3.5 Medium & 0.9875 & 0.8131 & 0.6469 & 0.8324 & 0.2400 & 0.5775 & 0.6829 \\
LADD & 0.2812 & 0.0025 & 0.0625 & 0.2181 & 0.0000 & 0.0000 & 0.0941 \\
DMD2 & 0.9000 & 0.3636 & 0.4750 & 0.6489 & 0.1200 & 0.2300 & 0.4563 \\
TwinFlow & 0.9594 & 0.6490 & 0.5688 & 0.7527 & \textbf{0.4250} & 0.3475 & 0.6171 \\
SenseFlow & 0.9344 & 0.6212 & 0.4562 & 0.6888 & 0.1625 & 0.2700 & 0.5222 \\
SWD & \underline{0.9656} & \underline{0.7601} & \underline{0.5844} & \underline{0.7793} & 0.2625 & \underline{0.4200} & \underline{0.6286} \\
\rowcolor{gray!8}\textbf{PVD (Ours)} & \textbf{0.9844} & \textbf{0.9066} & \textbf{0.5906} & \textbf{0.7819} & \underline{0.3475} & \textbf{0.5575} & \textbf{0.6948} \\
\midrule
\rowcolor{gray!15} FLUX.1 dev & 0.9750 & 0.8182 & 0.7312 & 0.7979 & 0.2200 & 0.4625 & 0.6675 \\
Hyper-FLUX & 0.8062 & 0.1566 & 0.3312 & 0.6356 & 0.0325 & 0.1000 & 0.3437 \\
FLUX-Turbo-$\alpha$ & 0.3938 & 0.0051 & 0.0844 & 0.3803 & 0.0050 & 0.0050 & 0.1456 \\
ArcFlow & 0.0156 & 0.0000 & 0.0000 & 0.0000 & 0.0000 & 0.0000 & 0.0026 \\
Pi-Flow & 0.9219 & 0.3712 & 0.5188 & 0.6782 & 0.0725 & 0.3275 & 0.4817 \\
SenseFlow & 0.9312 & 0.1944 & \textbf{0.6500} & 0.7473 & 0.0450 & 0.1625 & 0.4551 \\
TDD & 0.5437 & 0.0278 & 0.1312 & 0.4894 & 0.0075 & 0.0225 & 0.2037 \\
SWD & \textbf{0.9906} & \underline{0.7045} & \underline{0.6062} & \underline{0.7872} & \underline{0.1925} & \underline{0.4900} & \underline{0.6285} \\
\rowcolor{gray!8}\textbf{PVD (Ours)} & \underline{0.9875} & \textbf{0.7828} & 0.5813 & \textbf{0.7952} & \textbf{0.2275} & \textbf{0.5450} & \textbf{0.6532} \\
\midrule
\rowcolor{gray!15} Qwen-Image & 0.9906 & 0.9318 & 0.8781 & 0.9016 & 0.7750 & 0.7550 & 0.8720 \\
ArcFlow & 0.3000 & 0.0177 & 0.0281 & 0.1968 & 0.0025 & 0.0025 & 0.0913 \\
Pi-Flow & 0.8625 & 0.5303 & 0.5781 & 0.6809 & 0.4000 & 0.3300 & 0.5636 \\
TwinFlow & \textbf{0.9938} & \underline{0.8965} & \underline{0.7562} & \underline{0.9016} & \underline{0.7150} & \underline{0.7400} & \underline{0.8338} \\
\rowcolor{gray!8}\textbf{PVD (Ours)} & \underline{0.9875} & \textbf{0.9747} & \textbf{0.8000} & \textbf{0.9255} & \textbf{0.8175} & \textbf{0.8025} & \textbf{0.8846} \\
\bottomrule
\end{tabular}
}
\end{table}

\begin{table}[htbp]
\centering
\caption{Quantitative evaluation on DPG-Bench. Higher is better ($\uparrow$) for all metrics. The best and second-best accelerated results within each backbone group are in \textbf{bold} and \underline{underlined}, respectively.}
\label{tab:dpg_bench_metrics}
\resizebox{\textwidth}{!}{
\begin{tabular}{l c ccc ccccc cc cc c}
\toprule
\multirow{2}{*}{Model} & \multirow{2}{*}{Global} & \multicolumn{3}{c}{Entity} & \multicolumn{5}{c}{Attribute} & \multicolumn{2}{c}{Relation} & \multicolumn{2}{c}{Other} & \multirow{2}{*}{Overall} \\
\cmidrule(lr){3-5} \cmidrule(lr){6-10} \cmidrule(lr){11-12} \cmidrule(lr){13-14}
& & Whole & State & Part & Color & Shape & Size & Tex. & Oth. & Spat. & N-Spat. & Cnt. & Txt. & \\
\midrule
\rowcolor{gray!15} SD3.5 Medium & 91.04 & 90.44 & 90.07 & 88.36 & 90.36 & 89.01 & 89.99 & 90.17 & 89.28 & 91.15 & 89.46 & 91.03 & 87.40 & 84.49 \\
LADD & 35.63 & 40.89 & 49.58 & 60.31 & 58.88 & 51.69 & 38.16 & 47.68 & 47.23 & 54.96 & 57.50 & 60.65 & 51.86 & 26.46 \\
DMD2 & 62.69 & 49.92 & 61.15 & 42.04 & 54.09 & 32.37 & 54.49 & 54.74 & 56.69 & 39.39 & 57.86 & 60.65 & 47.09 & 46.15 \\
TwinFlow & 82.16 & 83.64 & 82.10 & 83.57 & 86.04 & \underline{89.33} & 82.89 & 80.87 & 82.32 & 85.47 & 83.95 & 84.87 & 81.22 & 72.81 \\
SenseFlow & \underline{88.56} & \underline{87.29} & \textbf{87.86} & \underline{85.11} & \underline{89.30} & \textbf{91.55} & \textbf{90.30} & 84.80 & \underline{86.92} & 89.23 & \textbf{89.12} & \underline{87.65} & \textbf{88.66} & \textbf{81.98} \\
SWD & 85.73 & 84.12 & 85.39 & 78.70 & 88.31 & 86.76 & \underline{88.10} & \textbf{88.43} & 86.28 & \textbf{91.47} & \underline{86.30} & 83.41 & 76.00 & 79.23 \\
\rowcolor{gray!8}\textbf{PVD (Ours)} & \textbf{90.43} & \textbf{87.64} & \underline{86.48} & \textbf{87.40} & \textbf{91.50} & 87.27 & 86.91 & \underline{88.39} & \textbf{89.13} & \underline{91.25} & 85.60 & \textbf{89.92} & \underline{87.72} & \underline{81.81} \\
\midrule
\rowcolor{gray!15} FLUX.1 dev & 83.28 & 92.28 & 84.48 & 82.23 & 90.92 & 80.79 & 66.53 & 87.48 & 85.14 & 93.53 & 88.68 & 79.00 & 90.00 & 83.95 \\
Hyper-FLUX & 61.81 & 69.32 & 72.24 & 64.66 & 77.92 & 65.41 & 70.32 & 72.00 & 61.24 & 74.42 & 67.66 & 69.04 & 75.34 & 53.50 \\
FLUX-Turbo-$\alpha$ & 64.01 & 52.00 & 65.55 & 49.00 & 66.12 & 59.31 & 63.47 & 61.69 & 50.27 & 45.88 & 69.61 & 55.47 & 68.24 & 33.90 \\
ArcFlow & 68.21 & 59.40 & 67.82 & 68.54 & 61.71 & 61.48 & 54.24 & 65.24 & 55.92 & 63.10 & 61.76 & 67.93 & 70.33 & 37.70 \\
Pi-Flow & 78.97 & 78.34 & 81.74 & 82.58 & 81.55 & \underline{80.95} & \underline{79.59} & 80.76 & 78.98 & 79.87 & 83.67 & \textbf{79.97} & 77.70 & 71.40 \\
SenseFlow & 77.61 & 74.35 & 76.93 & 77.46 & 79.57 & 75.33 & 75.52 & 78.94 & 76.53 & 82.21 & 81.85 & 79.34 & 76.65 & 67.01 \\
TDD & 68.82 & 61.81 & 71.06 & 62.62 & 79.88 & 69.94 & 61.81 & 75.25 & 69.59 & 66.00 & 71.69 & 66.15 & 76.33 & 50.63 \\
SWD & \underline{85.37} & \underline{90.58} & \textbf{88.28} & \textbf{89.01} & \underline{86.32} & \textbf{87.95} & \textbf{88.33} & \underline{85.47} & \textbf{89.01} & \underline{91.39} & \textbf{90.04} & 77.60 & \underline{89.69} & \underline{82.43} \\
\rowcolor{gray!8}\textbf{PVD (Ours)} & \textbf{85.41} & \textbf{91.48} & \underline{82.68} & \underline{83.20} & \textbf{92.62} & 78.60 & 74.38 & \textbf{87.42} & \underline{85.70} & \textbf{93.08} & \underline{89.31} & \underline{79.50} & \textbf{94.00} & \textbf{82.83} \\
\midrule
\rowcolor{gray!15} Qwen-Image & 94.68 & 92.95 & 87.66 & 88.64 & 93.80 & 92.92 & 94.27 & 87.19 & 92.41 & 92.88 & 93.38 & 93.62 & 94.19 & 89.10 \\
ArcFlow & 74.92 & 61.95 & 75.56 & 68.09 & 75.38 & 78.88 & 78.65 & 60.19 & 66.78 & 63.96 & 75.35 & 64.46 & 56.20 & 43.38 \\
Pi-Flow & 81.71 & 75.53 & 84.41 & 74.48 & 78.49 & \underline{79.68} & \underline{81.18} & 79.57 & 80.13 & 81.70 & 70.38 & 80.80 & 79.52 & 68.45 \\
TwinFlow & \textbf{91.15} & \underline{93.58} & \textbf{91.30} & \textbf{90.04} & \underline{89.35} & \textbf{92.46} & \textbf{92.35} & \textbf{90.57} & \underline{86.60} & \underline{93.36} & \textbf{93.21} & \textbf{93.21} & \underline{89.10} & \textbf{86.75} \\
\rowcolor{gray!8}\textbf{PVD (Ours)} & \underline{84.19} & \textbf{93.74} & \underline{85.74} & \underline{88.67} & \textbf{94.28} & 79.04 & 73.97 & \underline{90.38} & \textbf{88.03} & \textbf{94.60} & \underline{89.94} & \underline{86.00} & \textbf{94.00} & \underline{86.54} \\
\bottomrule
\end{tabular}
}
\end{table}

\begin{table}[htbp]
\centering
\caption{Quantitative comparison on WISE. Higher is better ($\uparrow$) for all metrics. The best and second-best accelerated results within each backbone group are in \textbf{bold} and \underline{underlined}, respectively.}
\label{tab:wise}
\resizebox{1\textwidth}{!}{
\begin{tabular}{l ccccccc}
\toprule
Model & Cultural & Time & Space & Biology & Physics & Chemistry & Overall \\
\midrule
\rowcolor{gray!15} SD3.5 Medium & 0.43 & 0.50 & 0.52 & 0.41 & 0.53 & 0.33 & 0.45 \\
LADD & 0.04 & 0.03 & 0.07 & 0.04 & 0.05 & 0.07 & 0.05 \\
DMD2 & 0.28 & 0.28 & 0.40 & 0.25 & 0.44 & \underline{0.30} & 0.31 \\
TwinFlow & 0.27 & \underline{0.40} & \underline{0.53} & \textbf{0.40} & \textbf{0.50} & \textbf{0.38} & \underline{0.37} \\
SenseFlow & 0.27 & 0.32 & 0.43 & 0.33 & 0.39 & 0.22 & 0.31 \\
SWD & \textbf{0.45} & 0.27 & 0.45 & 0.30 & 0.35 & 0.10 & 0.36 \\
\rowcolor{gray!8}\textbf{PVD (Ours)} & \underline{0.35} & \textbf{0.46} & \textbf{0.54} & \underline{0.39} & \underline{0.49} & 0.26 & \textbf{0.40} \\
\midrule
\rowcolor{gray!15} FLUX.1 dev & 0.48 & 0.58 & 0.62 & 0.42 & 0.51 & 0.35 & 0.50 \\
Hyper-FLUX & 0.15 & 0.12 & 0.27 & 0.11 & 0.24 & 0.21 & 0.17 \\
FLUX-Turbo-$\alpha$ & 0.09 & 0.04 & 0.16 & 0.04 & 0.13 & 0.13 & 0.09 \\
ArcFlow & 0.12 & 0.08 & 0.15 & 0.05 & 0.08 & 0.07 & 0.10 \\
Pi-Flow & 0.22 & 0.21 & 0.35 & 0.14 & 0.31 & 0.22 & 0.23 \\
SenseFlow & 0.25 & 0.22 & 0.31 & 0.15 & 0.31 & \underline{0.24} & 0.25 \\
TDD & 0.13 & 0.11 & 0.22 & 0.10 & 0.20 & 0.13 & 0.14 \\
SWD & \textbf{0.46} & \underline{0.27} & \underline{0.43} & \underline{0.34} & \underline{0.40} & 0.14 & \underline{0.37} \\
\rowcolor{gray!8}\textbf{PVD (Ours)} & \underline{0.40} & \textbf{0.48} & \textbf{0.59} & \textbf{0.39} & \textbf{0.52} & \textbf{0.32} & \textbf{0.44} \\
\midrule
\rowcolor{gray!15} Qwen-Image & 0.66 & 0.61 & 0.79 & 0.59 & 0.70 & 0.41 & 0.64 \\
ArcFlow & 0.21 & 0.14 & 0.30 & 0.14 & 0.27 & 0.16 & 0.20 \\
Pi-Flow & 0.30 & 0.23 & 0.44 & 0.20 & 0.41 & 0.21 & 0.30 \\
TwinFlow & \underline{0.56} & \underline{0.58} & \underline{0.70} & \underline{0.51} & \underline{0.63} & \textbf{0.41} & \underline{0.57} \\
\rowcolor{gray!8}\textbf{PVD (Ours)} & \textbf{0.57} & \textbf{0.61} & \textbf{0.74} & \textbf{0.54} & \textbf{0.67} & \underline{0.37} & \textbf{0.59} \\
\bottomrule
\end{tabular}
}
\end{table}

\begin{table*}[htbp]
\centering
\caption{Quantitative comparison on TIIF Benchmark mini (Part I: Visual \& Relational). L/S stands for Long/Short. Scores are multiplied by 100. Higher is better ($\uparrow$) for all metrics. The best and second-best accelerated results within each backbone group are in \textbf{bold} and \underline{underlined}, respectively.}
\label{tab:tiif_part1_visual}
\resizebox{\textwidth}{!}{
\begin{tabular}{l cc c c cccc ccc ccc}
\toprule
\multirow{3}{*}{Model} & \multicolumn{2}{c}{TIIF Overall} & \multicolumn{1}{c}{Spatial (2D)} & \multicolumn{1}{c}{Spatial (3D)} & \multicolumn{4}{c}{Action} & \multicolumn{3}{c}{Color} & \multicolumn{3}{c}{Texture} \\
\cmidrule(lr){2-3} \cmidrule(lr){4-4} \cmidrule(lr){5-5} \cmidrule(lr){6-9} \cmidrule(lr){10-12} \cmidrule(lr){13-15}
& \multirow{2}{*}{S $\uparrow$} & \multirow{2}{*}{L $\uparrow$} & \multirow{2}{*}{(L/S)} & \multirow{2}{*}{(L/S)} & 2D & 3D & Col & Tex & 2D & 3D & Tex & 2D & 3D & Col \\
\cmidrule(lr){6-6} \cmidrule(lr){7-7} \cmidrule(lr){8-8} \cmidrule(lr){9-9} \cmidrule(lr){10-10} \cmidrule(lr){11-11} \cmidrule(lr){12-12} \cmidrule(lr){13-13} \cmidrule(lr){14-14} \cmidrule(lr){15-15}
& & & & & (L/S) & (L/S) & (L/S) & (L/S) & (L/S) & (L/S) & (L/S) & (L/S) & (L/S) & (L/S) \\
\midrule
\rowcolor{gray!15} SD3.5 Medium & 71.66 & 71.90 & 83/83 & 79/70 & 91/86 & 82/82 & 73/65 & 68/75 & 95/90 & 70/50 & 92/88 & 90/85 & 83/77 & 92/100 \\
LADD & 22.32 & 20.19 & 0/41 & 29/41 & 8/13 & 17/12 & 26/34 & 20/27 & 40/40 & 35/20 & 16/8 & 28/19 & 33/50 & 36/56 \\
DMD2 & 54.10 & 49.31 & 75/\underline{87} & \underline{62}/\underline{79} & 72/80 & 76/\textbf{87} & 61/\textbf{73} & \underline{72}/\underline{68} & \textbf{95}/\underline{95} & 70/70 & 52/72 & 76/\textbf{100} & 72/66 & 80/80 \\
TwinFlow & 58.56 & 60.94 & 79/\textbf{91} & \textbf{79}/75 & \underline{75}/\textbf{88} & \textbf{84}/\underline{84} & 57/\textbf{73} & 68/\textbf{75} & \underline{90}/\underline{95} & \textbf{86}/\underline{75} & 72/84 & \textbf{85}/\underline{95} & \underline{77}/\underline{77} & 84/88 \\
SenseFlow & 62.73 & 59.57 & \underline{87}/79 & \textbf{79}/\underline{79} & \underline{75}/\underline{83} & 71/82 & 61/57 & \textbf{75}/\underline{68} & 80/80 & \underline{75}/70 & \underline{80}/\underline{88} & 76/76 & \textbf{83}/\underline{77} & 88/\textbf{100} \\
SWD & \underline{64.93} & \underline{64.52} & \underline{87}/70 & \textbf{79}/\underline{79} & \textbf{80}/\underline{83} & \underline{79}/82 & \textbf{80}/\underline{61} & \textbf{75}/\underline{68} & \underline{90}/\underline{95} & 70/\textbf{80} & \textbf{84}/80 & \underline{80}/85 & \underline{77}/\underline{77} & \textbf{100}/\underline{96} \\
\rowcolor{gray!8}\textbf{PVD (Ours)} & \textbf{65.61} & \textbf{66.38} & \textbf{91}/\textbf{91} & \textbf{79}/\textbf{83} & 66/75 & 76/76 & \underline{76}/\textbf{73} & 68/\textbf{75} & \textbf{95}/\textbf{100} & 70/55 & \textbf{84}/\textbf{96} & \textbf{85}/71 & \underline{77}/\textbf{88} & \underline{96}/\textbf{100} \\
\midrule
\rowcolor{gray!15} FLUX.1 dev & 67.60 & 71.44 & 95/83 & 75/83 & 80/77 & 84/79 & 65/73 & 65/79 & 95/90 & 65/70 & 88/96 & 90/85 & 77/72 & 100/100 \\
Hyper-FLUX & 46.14 & 43.82 & 58/87 & 58/\underline{79} & 58/61 & 71/58 & 50/\underline{73} & 72/68 & 40/60 & 65/50 & 52/72 & 52/52 & 66/72 & 52/52 \\
FLUX-Turbo-$\alpha$ & 29.64 & 28.43 & 33/45 & 41/66 & 19/36 & 41/38 & 26/57 & 6/34 & 60/60 & 30/45 & 36/24 & 19/28 & 61/61 & 40/36 \\
ArcFlow & 33.92 & 27.32 & 33/33 & 33/62 & 52/33 & 32/25 & 23/46 & 31/37 & 45/55 & 30/60 & 44/36 & 23/38 & 27/44 & 8/32 \\
Pi-Flow & 58.40 & 58.22 & 75/75 & \underline{75}/\underline{79} & 66/80 & 79/\textbf{82} & \textbf{84}/65 & 72/72 & \underline{80}/\underline{90} & \underline{70}/60 & \underline{68}/76 & \textbf{100}/\underline{76} & 72/72 & 76/88 \\
SenseFlow & 53.57 & 55.28 & 66/79 & 62/\underline{79} & \underline{77}/72 & 74/66 & 50/69 & \underline{82}/72 & \underline{80}/75 & \textbf{75}/\underline{65} & 60/\underline{80} & 66/57 & \underline{77}/72 & 72/60 \\
TDD & 43.32 & 38.62 & 37/75 & 66/66 & 50/63 & 69/64 & \underline{73}/65 & 48/51 & 45/75 & 65/\underline{65} & 36/56 & 28/19 & 72/72 & 52/68 \\
SWD & \underline{68.00} & \underline{68.74} & \underline{87}/\textbf{100} & \underline{75}/\textbf{87} & \textbf{80}/\underline{83} & \textbf{94}/71 & \underline{73}/65 & \textbf{86}/\textbf{86} & \textbf{95}/\textbf{95} & \textbf{75}/\textbf{75} & \textbf{96}/\textbf{96} & \underline{80}/\underline{76} & \underline{77}/\underline{83} & \textbf{92}/\underline{96} \\
\rowcolor{gray!8}\textbf{PVD (Ours)} & \textbf{68.19} & \textbf{69.87} & \textbf{95}/\underline{95} & \textbf{83}/\textbf{87} & 75/\textbf{86} & \underline{84}/\underline{76} & 69/\textbf{84} & 75/\underline{79} & \textbf{95}/\textbf{95} & \textbf{75}/\textbf{75} & \textbf{96}/\textbf{96} & 66/\textbf{85} & \textbf{88}/\textbf{88} & \underline{88}/\textbf{100} \\
\midrule
\rowcolor{gray!15} Qwen-Image & 88.11 & 86.45 & 92/96 & 83/83 & 94/92 & 85/85 & 73/77 & 86/79 & 95/95 & 85/90 & 96/96 & 95/95 & 78/83 & 100/100 \\
ArcFlow & 65.10 & 56.38 & \underline{38}/67 & 71/\underline{75} & 84/69 & 49/56 & 50/\underline{73} & 55/45 & 75/90 & 65/75 & 55/52 & 44/\underline{67} & 39/50 & 36/72 \\
Pi-Flow & 80.73 & 82.70 & \textbf{92}/92 & \textbf{88}/\textbf{79} & \textbf{94}/\underline{83} & 79/79 & 62/69 & 62/75 & 90/\underline{95} & \underline{75}/75 & \underline{88}/\underline{80} & \underline{81}/\underline{67} & 67/67 & 84/\underline{96} \\
TwinFlow & \underline{83.43} & \underline{84.52} & \textbf{92}/\textbf{100} & \underline{83}/\underline{75} & \underline{86}/\textbf{89} & \textbf{90}/\underline{85} & \underline{73}/\textbf{77} & \textbf{79}/\underline{76} & \underline{95}/\underline{95} & \textbf{85}/\underline{80} & \underline{88}/\textbf{96} & 71/\textbf{86} & \underline{78}/\underline{78} & \textbf{100}/\textbf{100} \\
\rowcolor{gray!8}\textbf{PVD (Ours)} & \textbf{86.04} & \textbf{86.87} & \textbf{92}/\underline{96} & \textbf{88}/\textbf{79} & \textbf{94}/\textbf{89} & \underline{85}/\textbf{87} & \textbf{85}/\textbf{77} & \underline{76}/\textbf{79} & \textbf{100}/\textbf{100} & \textbf{85}/\textbf{85} & \textbf{96}/\textbf{96} & \textbf{90}/\underline{67} & \textbf{83}/\textbf{89} & \underline{96}/\textbf{100} \\
\bottomrule
\end{tabular}
}
\end{table*}

\begin{table*}[htbp]
\centering
\caption{Quantitative comparison on TIIF Benchmark mini (Part II: Logic \& Reasoning). L/S stands for Long/Short. Scores are multiplied by 100. Higher is better ($\uparrow$) for all metrics. The best and second-best accelerated results within each backbone group are in \textbf{bold} and \underline{underlined}, respectively.}
\label{tab:tiif_part2_reasoning}
\resizebox{\textwidth}{!}{
\begin{tabular}{l ccccc ccccc ccccc}
\toprule
\multirow{3}{*}{Model} & \multicolumn{5}{c}{Comparison} & \multicolumn{5}{c}{Differentiation} & \multicolumn{5}{c}{Negation} \\
\cmidrule(lr){2-6} \cmidrule(lr){7-11} \cmidrule(lr){12-16}
& Base & 2D & 3D & Col & Tex & Base & 2D & 3D & Col & Tex & Base & 2D & 3D & Col & Tex \\
\cmidrule(lr){2-2} \cmidrule(lr){3-3} \cmidrule(lr){4-4} \cmidrule(lr){5-5} \cmidrule(lr){6-6} \cmidrule(lr){7-7} \cmidrule(lr){8-8} \cmidrule(lr){9-9} \cmidrule(lr){10-10} \cmidrule(lr){11-11} \cmidrule(lr){12-12} \cmidrule(lr){13-13} \cmidrule(lr){14-14} \cmidrule(lr){15-15} \cmidrule(lr){16-16}
& (L/S) & (L/S) & (L/S) & (L/S) & (L/S) & (L/S) & (L/S) & (L/S) & (L/S) & (L/S) & (L/S) & (L/S) & (L/S) & (L/S) & (L/S) \\
\midrule
\rowcolor{gray!15} SD3.5 Medium & 79/75 & 68/75 & 62/41 & 66/66 & 30/20 & 84/76 & 56/62 & 82/70 & 80/61 & 75/85 & 70/62 & 68/81 & 64/64 & 61/72 & 72/66 \\
LADD & 16/29 & 0/0 & 6/6 & 4/38 & 15/30 & 52/48 & 0/0 & 0/5 & 47/52 & 10/25 & 50/50 & 56/50 & 35/47 & 72/44 & 44/33 \\
DMD2 & 66/\underline{70} & \underline{75}/50 & 25/62 & 61/\textbf{90} & 33/20 & \underline{88}/80 & 43/50 & \textbf{84}/64 & 71/61 & 25/\underline{70} & \textbf{66}/54 & 56/\underline{75} & 64/\textbf{76} & 72/\textbf{77} & \textbf{61}/\underline{55} \\
TwinFlow & 79/\textbf{83} & 50/\textbf{75} & 31/\underline{68} & \textbf{90}/57 & \textbf{70}/\textbf{40} & \textbf{92}/\underline{84} & \textbf{87}/\textbf{93} & 70/70 & 71/\textbf{90} & 55/\textbf{75} & \underline{62}/58 & \textbf{75}/\underline{75} & 58/58 & \underline{77}/61 & \underline{55}/\underline{55} \\
SenseFlow & 75/\underline{70} & \underline{75}/\underline{56} & 50/\textbf{81} & \textbf{90}/57 & 20/20 & 80/\textbf{88} & 50/\underline{56} & 70/\textbf{82} & 80/47 & \underline{70}/50 & \underline{62}/\underline{62} & \underline{68}/68 & \textbf{76}/\underline{70} & 55/\textbf{77} & \underline{55}/\underline{55} \\
SWD & \underline{87}/\textbf{83} & \textbf{81}/\textbf{75} & \textbf{87}/\textbf{81} & 76/47 & \underline{50}/\textbf{40} & 76/\textbf{88} & 50/\underline{56} & 64/\textbf{82} & \textbf{90}/\underline{80} & \textbf{75}/60 & \underline{62}/\textbf{66} & 56/\underline{75} & \underline{70}/\underline{70} & \textbf{83}/61 & \textbf{61}/\textbf{66} \\
\rowcolor{gray!8}\textbf{PVD (Ours)} & \textbf{91}/\textbf{83} & 68/43 & \underline{56}/\underline{68} & \underline{85}/\underline{76} & \underline{50}/\underline{35} & 84/\underline{84} & \underline{68}/\underline{56} & \underline{76}/\underline{76} & \underline{85}/\underline{80} & \underline{70}/\textbf{75} & \textbf{66}/\textbf{66} & \textbf{75}/\textbf{81} & 52/64 & 72/\underline{66} & \underline{55}/\underline{55} \\
\midrule
\rowcolor{gray!15} FLUX.1 dev & 75/79 & 62/68 & 43/31 & 61/66 & 50/15 & 76/76 & 50/68 & 64/82 & 71/66 & 70/55 & 62/66 & 75/75 & 58/64 & 66/66 & 77/61 \\
Hyper-FLUX & 50/\underline{70} & 31/12 & \underline{68}/43 & 57/47 & 30/\underline{25} & 76/80 & \textbf{75}/\textbf{68} & 41/29 & 66/47 & 15/30 & 66/\textbf{75} & 68/\underline{75} & \textbf{76}/\textbf{82} & 50/66 & \textbf{61}/50 \\
FLUX-Turbo-$\alpha$ & 12/12 & 31/37 & 18/0 & 38/28 & 20/20 & 36/64 & 50/50 & 47/41 & 42/38 & 15/10 & 58/45 & 56/56 & 47/58 & \textbf{88}/\underline{77} & 50/16 \\
ArcFlow & 29/29 & 37/37 & 6/31 & 52/\underline{66} & 30/20 & 56/32 & 50/\textbf{68} & 23/52 & \underline{71}/57 & 25/25 & 54/\textbf{75} & 43/56 & 58/64 & 50/38 & 44/33 \\
Pi-Flow & 54/\textbf{87} & 43/\underline{56} & \textbf{81}/43 & 66/\underline{66} & 30/\textbf{30} & 80/\textbf{88} & \textbf{75}/56 & 70/70 & \textbf{80}/\textbf{85} & \underline{50}/30 & 66/62 & \textbf{81}/\textbf{81} & \underline{70}/64 & \underline{83}/72 & \underline{55}/\underline{55} \\
SenseFlow & \underline{66}/50 & 43/50 & \underline{68}/50 & 47/42 & 55/\underline{25} & 76/80 & 56/\underline{62} & 64/35 & 66/52 & 20/45 & 58/\textbf{75} & \textbf{81}/\underline{75} & \textbf{76}/70 & 61/72 & \underline{55}/38 \\
TDD & 20/45 & 18/25 & 43/43 & 42/33 & 40/\underline{25} & 64/\underline{84} & \underline{62}/\textbf{68} & 29/52 & \underline{71}/47 & 25/15 & 66/66 & 68/\underline{75} & 52/\underline{76} & 72/\underline{77} & 38/38 \\
SWD & \textbf{87}/\textbf{87} & \textbf{81}/\textbf{68} & 62/\textbf{75} & \textbf{90}/\textbf{71} & \textbf{80}/20 & \textbf{88}/\textbf{88} & \textbf{75}/\textbf{68} & \textbf{88}/\underline{82} & \textbf{80}/\underline{76} & \textbf{65}/\textbf{90} & \textbf{75}/\underline{70} & 62/68 & 64/64 & 66/72 & 50/\textbf{61} \\
\rowcolor{gray!8}\textbf{PVD (Ours)} & \textbf{87}/\textbf{87} & \underline{62}/\textbf{68} & \underline{68}/\underline{62} & \underline{85}/\textbf{71} & \underline{65}/20 & \underline{84}/\textbf{88} & \underline{62}/56 & \underline{82}/\textbf{94} & \underline{71}/\underline{76} & \textbf{65}/\underline{65} & \underline{70}/62 & \underline{75}/\underline{75} & 64/70 & \underline{83}/\textbf{83} & \textbf{61}/50 \\
\midrule
\rowcolor{gray!15} Qwen-Image & 92/88 & 75/69 & 88/75 & 86/90 & 75/60 & 96/92 & 88/75 & 76/76 & 95/90 & 90/90 & 63/71 & 75/75 & 71/82 & 72/72 & 56/72 \\
ArcFlow & 50/63 & \underline{50}/\underline{63} & 50/69 & 62/71 & 33/30 & \textbf{95}/\textbf{88} & 75/56 & 29/47 & 81/81 & 33/25 & 57/46 & \underline{75}/63 & \underline{83}/71 & \underline{83}/67 & \underline{56}/44 \\
Pi-Flow & \underline{79}/\underline{88} & 38/56 & 63/\underline{75} & \textbf{95}/\textbf{86} & \textbf{90}/\underline{50} & \underline{92}/\underline{84} & \textbf{88}/\textbf{94} & \textbf{71}/\underline{82} & \textbf{95}/\underline{90} & \underline{90}/\underline{80} & \textbf{75}/\textbf{67} & \textbf{81}/69 & \textbf{88}/\textbf{94} & \textbf{89}/\textbf{78} & \underline{56}/\underline{50} \\
TwinFlow & \textbf{92}/83 & \textbf{75}/\textbf{69} & \underline{81}/\textbf{81} & \underline{90}/\underline{76} & \underline{65}/20 & 80/\textbf{88} & 69/\underline{88} & \underline{65}/76 & \underline{90}/81 & \textbf{95}/\underline{80} & \underline{67}/54 & \underline{75}/\underline{75} & 65/76 & 67/\underline{72} & \textbf{61}/\textbf{61} \\
\rowcolor{gray!8}\textbf{PVD (Ours)} & \textbf{92}/\textbf{92} & \textbf{75}/\textbf{69} & \textbf{88}/\underline{75} & 81/71 & \underline{65}/\textbf{75} & 84/80 & \underline{81}/\underline{88} & \textbf{71}/\textbf{94} & 81/\textbf{95} & \underline{90}/\textbf{95} & \textbf{75}/\underline{63} & \underline{75}/\textbf{88} & 76/\underline{88} & 78/\textbf{78} & 50/\textbf{61} \\
\bottomrule
\end{tabular}
}
\end{table*}

\begin{table*}[htbp]
\centering
\caption{Quantitative comparison on TIIF Benchmark mini (Part III: Concepts \& Domains). L/S stands for Long/Short. Scores are multiplied by 100. Higher is better ($\uparrow$) for all metrics.The best and second-best accelerated results within each backbone group are in \textbf{bold} and \underline{underlined}, respectively.}
\label{tab:tiif_part3_concepts}
\resizebox{\textwidth}{!}{
\begin{tabular}{l ccccc cccc ccc}
\toprule
\multirow{3}{*}{Model} & \multicolumn{5}{c}{Numeracy} & \multicolumn{4}{c}{Shape} & \multirow{2}{*}{Real World} & \multirow{2}{*}{Style} & \multirow{2}{*}{Text} \\
\cmidrule(lr){2-6} \cmidrule(lr){7-10}
& Base & 2D & 3D & Col & Tex & 2D & 3D & Col & Tex & & & \\
\cmidrule(lr){2-2} \cmidrule(lr){3-3} \cmidrule(lr){4-4} \cmidrule(lr){5-5} \cmidrule(lr){6-6} \cmidrule(lr){7-7} \cmidrule(lr){8-8} \cmidrule(lr){9-9} \cmidrule(lr){10-10} \cmidrule(lr){11-11} \cmidrule(lr){12-12} \cmidrule(lr){13-13}
& (L/S) & (L/S) & (L/S) & (L/S) & (L/S) & (L/S) & (L/S) & (L/S) & (L/S) & (L/S) & (L/S) & (L/S) \\
\midrule
\rowcolor{gray!15} SD3.5 Medium & 72/70 & 77/48 & 78/87 & 65/77 & 70/59 & 45/40 & 65/70 & 78/86 & 72/68 & 75/73 & 70/66 & 48/67 \\
LADD & 22/12 & 7/22 & 25/40 & 22/30 & 31/9 & 25/15 & 20/20 & 26/36 & 24/16 & 27/23 & 3/3 & 0/0 \\
DMD2 & 50/64 & 55/\underline{70} & 46/71 & 47/60 & 38/54 & 45/30 & 55/60 & 66/\underline{88} & 58/66 & 46/42 & 16/13 & 0/5 \\
TwinFlow & 68/\underline{70} & 62/55 & 65/71 & \underline{60}/57 & 47/56 & \textbf{60}/\textbf{55} & \textbf{80}/\textbf{80} & 72/\textbf{90} & 64/\textbf{78} & 57/55 & 33/26 & 3/6 \\
SenseFlow & \textbf{81}/66 & \textbf{70}/62 & 65/\underline{84} & 50/65 & 52/56 & \underline{50}/40 & 55/\underline{70} & \textbf{80}/74 & 62/\underline{68} & 59/\underline{62} & 43/\underline{56} & 7/18 \\
SWD & \underline{72}/\textbf{77} & \underline{66}/\textbf{77} & \underline{78}/\textbf{90} & \underline{60}/\underline{70} & \underline{63}/\textbf{77} & 45/\underline{45} & \underline{60}/65 & \underline{76}/76 & \underline{70}/66 & \textbf{72}/\textbf{70} & \underline{46}/46 & \underline{10}/\underline{19} \\
\rowcolor{gray!8}\textbf{PVD (Ours)} & \textbf{81}/64 & 59/40 & \textbf{81}/78 & \textbf{75}/\textbf{72} & \textbf{70}/\underline{68} & 45/40 & \underline{60}/55 & 74/80 & \textbf{72}/64 & \underline{60}/58 & \textbf{66}/\textbf{66} & \textbf{19}/\textbf{23} \\
\midrule
\rowcolor{gray!15} FLUX.1 dev & 70/75 & 59/66 & 81/78 & 77/87 & 68/75 & 50/35 & 75/70 & 70/78 & 66/76 & 80/79 & 63/50 & 59/35 \\
Hyper-FLUX & 41/29 & 70/55 & 43/65 & 37/42 & 29/29 & \underline{55}/\underline{40} & 40/\underline{70} & 64/72 & 48/46 & 50/48 & 6/6 & 5/7 \\
FLUX-Turbo-$\alpha$ & 10/18 & 22/18 & 34/25 & 22/27 & 13/13 & 30/25 & 20/25 & 44/66 & 44/22 & 40/37 & 0/0 & 4/3 \\
ArcFlow & 16/39 & 22/44 & 50/37 & 17/35 & 38/45 & 25/25 & 20/60 & 28/48 & 16/32 & 30/32 & 0/0 & 8/\underline{17} \\
Pi-Flow & 60/62 & \underline{77}/48 & \underline{71}/71 & 67/\underline{75} & 68/52 & 50/\textbf{55} & \underline{70}/\textbf{75} & \textbf{70}/\underline{82} & 66/58 & 64/64 & 23/20 & 17/\textbf{19} \\
SenseFlow & \textbf{75}/62 & 70/\underline{62} & \underline{71}/78 & 65/55 & 52/54 & \textbf{60}/35 & 55/\textbf{75} & 58/78 & 56/62 & 62/52 & 30/26 & 15/\underline{17} \\
TDD & 31/39 & \underline{77}/25 & 53/43 & 30/42 & 18/27 & 35/\underline{40} & 35/60 & 46/78 & 34/48 & 43/44 & 10/3 & 6/7 \\
SWD & \underline{72}/\textbf{93} & \textbf{81}/\textbf{66} & \underline{71}/\textbf{93} & \underline{80}/\textbf{80} & \textbf{75}/\textbf{83} & 50/\underline{40} & \underline{70}/50 & \underline{68}/\textbf{88} & \underline{68}/\underline{72} & \underline{77}/\underline{73} & \underline{43}/\underline{53} & \underline{23}/\underline{17} \\
\rowcolor{gray!8}\textbf{PVD (Ours)} & 70/\underline{77} & 70/55 & \textbf{81}/\underline{90} & \textbf{90}/\textbf{80} & \underline{70}/\underline{77} & \textbf{60}/\textbf{55} & \textbf{75}/\textbf{75} & 66/68 & \textbf{70}/\textbf{78} & \textbf{78}/\textbf{75} & \textbf{60}/\textbf{56} & \textbf{26}/13 \\
\midrule
\rowcolor{gray!15} Qwen-Image & 79/90 & 67/81 & 94/91 & 95/95 & 86/84 & 70/60 & 75/75 & 82/84 & 76/82 & 92/91 & 80/87 & 90/92 \\
ArcFlow & 38/50 & 48/63 & 69/69 & 35/75 & 43/43 & 45/\textbf{70} & 50/65 & 62/\underline{80} & 42/64 & 54/56 & 7/0 & 21/45 \\
Pi-Flow & \underline{77}/73 & \underline{70}/\textbf{81} & 69/\underline{81} & 75/\textbf{90} & 61/73 & 50/45 & \underline{70}/\underline{75} & 72/78 & \underline{72}/70 & 63/66 & 27/13 & 22/\underline{48} \\
TwinFlow & \underline{77}/\underline{79} & \underline{70}/67 & \underline{78}/\textbf{91} & \textbf{93}/\underline{85} & \textbf{80}/\textbf{82} & \underline{60}/50 & \textbf{80}/65 & \textbf{84}/76 & \textbf{76}/\underline{76} & \underline{87}/\textbf{88} & \underline{67}/\underline{87} & \textbf{74}/\textbf{76} \\
\rowcolor{gray!8}\textbf{PVD (Ours)} & \textbf{88}/\textbf{81} & \textbf{85}/\underline{70} & \textbf{88}/\underline{81} & \underline{85}/83 & \underline{73}/\underline{80} & \textbf{70}/\underline{65} & 65/\textbf{90} & \underline{78}/\textbf{84} & \textbf{76}/\textbf{86} & \textbf{88}/\underline{84} & \textbf{73}/\textbf{90} & \underline{43}/47 \\
\bottomrule
\end{tabular}
}
\end{table*}

\begin{table*}[t]
    \centering
    \caption{{\textbf{Quantitative comparison on Qwen-Image-Bench.} Higher is better ($\uparrow$) for all metrics. Best results within each distillation group are in \textbf{bold}, and the second-best results are \underline{underlined}.}}
    \label{tab:qwenimage_bench}
    \arrayrulecolor{black}
    \resizebox{\textwidth}{!}{
    \begin{tabular}{l ccc ccc}
        \toprule
        \multirow{2.5}{*}{\textbf{Model}} &
        \multirow{2.5}{*}{\textbf{Quality}} &
        \multirow{2.5}{*}{\textbf{Aesthetics}} &
        \multirow{2.5}{*}{\textbf{Alignment}} &
        \textbf{Real-world} & \textbf{Creative} &
        \multirow{2.5}{*}{\textbf{Overall}} \\
        & & & & \textbf{Fidelity} & \textbf{Generation} & \\
        \midrule
        \rowcolor{gray!15} \textbf{SD3.5-Medium} & 47.77 & 45.88 & 41.83 & 42.08 & 33.05 & 43.19 \\
        LADD & 1.22 & 7.19 & 4.41 & 30.44 & 0.54 & 6.55 \\
        DMD2 & 22.68 & 35.84 & 30.18 & 36.90 & 19.89 & 29.27 \\
        TwinFlow & 29.21 & \underline{37.13} & 28.83 & 36.03 & 18.54 & 30.55 \\
        SenseFlow & \underline{30.26} & 33.04 & 27.60 & 36.35 & 17.83 & 29.43 \\
        SWD & 29.83 & 36.63 & \underline{30.78} & \underline{37.35} & \underline{20.85} & \underline{31.44} \\
        \rowcolor{gray!8}\textbf{PVD (Ours)} & \textbf{42.73} & \textbf{37.33} & \textbf{32.78} & \textbf{37.48} & \textbf{22.41} & \textbf{35.51} \\
        \midrule
        \rowcolor{gray!15} \textbf{FLUX.1-dev} & 47.18 & 48.30 & 42.06 & 40.01 & 34.23 & 43.52 \\
        Hyper-FLUX & 3.43 & 22.28 & 13.67 & 31.67 & 3.02 & 13.85 \\
        FLUX-Turbo-$\alpha$ & 1.85 & 11.97 & 6.23 & 30.70 & 0.83 & 8.47 \\
        TDD & 4.07 & 21.45 & 13.13 & 31.49 & 3.46 & 13.71 \\
        Pi-Flow & 5.97 & 29.69 & 24.79 & 34.69 & 9.87 & 20.38 \\
        ArcFlow & 0.36 & 0.70 & 1.36 & 30.57 & 0.23 & 3.87 \\
        SenseFlow & 8.27 & 34.81 & 22.13 & 34.02 & 11.08 & 21.69 \\
        SWD & \underline{38.05} & \underline{44.96} & \underline{40.47} & \underline{39.93} & \underline{31.57} & \underline{39.71} \\
        \rowcolor{gray!8}\textbf{PVD (Ours)} & \textbf{47.34} & \textbf{45.00} & \textbf{42.61} & \textbf{42.04} & \textbf{33.54} & \textbf{43.14} \\
        \midrule
        \rowcolor{gray!15} \textbf{Qwen-Image} & 48.69 & 51.34 & 49.92 & 44.31 & 46.35 & 48.78 \\
        ArcFlow & 0.38 & 2.26 & 3.39 & 30.17 & 1.14 & 4.88 \\
        Pi-Flow & 5.23 & 24.18 & 22.29 & 33.91 & 10.31 & 18.08 \\
        TwinFlow & \underline{48.19} & \underline{47.43} & \underline{47.54} & \textbf{44.12} & \underline{38.85} & \underline{46.16} \\
        \rowcolor{gray!8}\textbf{PVD (Ours)} & \textbf{48.87} & \textbf{49.23} & \textbf{48.06} & \underline{43.64} & \textbf{39.28} & \textbf{46.85} \\
        \bottomrule
    \end{tabular}
    }
\end{table*}

\clearpage
\section{Detailed Cost Measurements}
\label{app:hardware}
{\cref{tab:hardware_cost} reports the detailed parameter counts, cumulative computation, latency, throughput, and peak VRAM. The metrics are measured at $1024\times1024$ image resolution on 1 PPU using BF16. 
We report backbone-only peak VRAM, excluding text encoders and the VAE. Active parameters refer to the number of parameters of the active model during generation.
Across all three backbones, PVD keeps cumulative computation close to one full-backbone evaluation while reducing active parameters by 49.10--50.89\% and peak VRAM by 45.76--48.36\% relative to the corresponding teachers.}

\begin{table}[H]
    \centering
    \caption{\textbf{Detailed hardware-cost comparison.} The final two columns report changes relative to the corresponding teacher; $\downarrow$ and $\uparrow$ denote reductions and increases, respectively.}
    \label{tab:hardware_cost}
    \label{tab:qwen_hardware_cost}
    \small
    \renewcommand{\arraystretch}{1.02}
    \resizebox{0.99\textwidth}{!}{%
    \begin{tabular}{@{}lcccccccc@{}}
        \toprule
        \multirow{2}{*}{\textbf{Model}}
        & \multicolumn{2}{c}{\textbf{Compute}}
        & \multicolumn{2}{c}{\textbf{Model runtime}}
        & \multicolumn{2}{c}{\textbf{Model footprint}}
        & \multicolumn{2}{c}{\textbf{Change vs. teacher}} \\
        \cmidrule(lr){2-3}\cmidrule(lr){4-5}\cmidrule(lr){6-7}\cmidrule(lr){8-9}
        & \makecell{\textbf{$N_{\mathrm{flops}}$}}
        & \textbf{TFLOPs}
        & \makecell{\textbf{Latency}\\\textbf{(s)}}
        & \makecell{\textbf{Throughput}\\\textbf{(images / s)}}
        & \makecell{\textbf{Active / Total}\\\textbf{Params. (B)}}
        & \makecell{\textbf{Peak VRAM}\\\textbf{(GiB)}}
        & \makecell{\textbf{Active}\\\textbf{Params.}}
        & \makecell{\textbf{Peak}\\\textbf{VRAM}} \\
        \midrule
        \rowcolor{gray!15}\textbf{SD3.5 Medium} & 100.00 & 701.93 & 15.49 & 0.07 & 2.24 / 2.24 & 4.99 & Ref. & Ref. \\
        LADD & 1.00 & 7.02 & 0.16 & 6.69 & 2.27 / 2.27 & 4.86 & $\uparrow$1.34\% & $\downarrow$2.61\% \\
        DMD2 & 1.00 & 7.02 & 0.15 & 6.89 & 2.24 / 2.24 & 4.81 & 0.00\% & $\downarrow$3.61\% \\
        TwinFlow & 1.00 & 7.02 & 0.16 & 6.67 & 2.25 / 2.25 & 4.83 & $\uparrow$0.45\% & $\downarrow$3.21\% \\
        SenseFlow & 1.00 & 7.02 & 0.16 & 6.67 & 2.24 / 2.24 & 4.81 & 0.00\% & $\downarrow$3.61\% \\
        SWD & 1.11 & 7.82 & 0.23 & 4.72 & 2.32 / 2.32 & 5.28 & $\uparrow$3.57\% & $\uparrow$5.81\% \\
        \rowcolor{gray!8}\textbf{PVD (Ours)} & \textbf{0.99} & \textbf{6.93} & \textbf{0.16} & \textbf{6.81} & \textbf{1.10 / 2.22} & \textbf{2.67} & \textbf{$\downarrow$50.89\%} & \textbf{$\downarrow$46.49\%} \\
        \midrule
        \rowcolor{gray!15}\textbf{FLUX.1 dev} & 50.00 & 2976.18 & 38.09 & 0.03 & 11.90 / 11.90 & 22.64 & Ref. & Ref. \\
        Hyper-FLUX & 1.00 & 59.52 & 0.76 & 1.35 & 12.25 / 12.25 & 23.29 & $\uparrow$2.94\% & $\uparrow$2.87\% \\
        FLUX-Turbo-$\alpha$ & 1.00 & 59.52 & 0.76 & 1.35 & 12.25 / 12.25 & 23.29 & $\uparrow$2.94\% & $\uparrow$2.87\% \\
        TDD & 0.94 & 56.21 & 0.76 & 1.35 & 12.06 / 12.06 & 22.91 & $\uparrow$1.34\% & $\uparrow$1.19\% \\
        Pi-Flow & 1.07 & 63.95 & 0.84 & 1.24 & 12.53 / 12.53 & 23.87 & $\uparrow$5.29\% & $\uparrow$5.43\% \\
        ArcFlow & 1.07 & 63.97 & 0.84 & 1.24 & 12.53 / 12.53 & 23.89 & $\uparrow$5.29\% & $\uparrow$5.52\% \\
        SenseFlow & 1.00 & 59.52 & 0.76 & 1.35 & 11.89 / 11.89 & 22.62 & $\downarrow$0.08\% & $\downarrow$0.09\% \\
        SWD & 1.05 & 62.37 & 0.98 & 1.07 & 12.29 / 12.29 & 24.45 & $\uparrow$3.28\% & $\uparrow$7.99\% \\
        \rowcolor{gray!8}\textbf{PVD (Ours)} & \textbf{1.01} & \textbf{60.27} & \textbf{0.86} & \textbf{1.22} & \textbf{5.98 / 6.14} & \textbf{12.28} & \textbf{$\downarrow$49.75\%} & \textbf{$\downarrow$45.76\%} \\
        \midrule
        \rowcolor{gray!15}\textbf{Qwen-Image} & 100.00 & 5585.14 & 71.77 & 0.01 & 20.43 / 20.43 & 38.42 & Ref. & Ref. \\
        ArcFlow & 1.07 & 59.77 & 0.78 & 1.36 & 21.37 / 21.37 & 40.36 & $\uparrow$4.60\% & $\uparrow$5.05\% \\
        Pi-Flow & 1.07 & 59.75 & 0.78 & 1.36 & 21.37 / 21.37 & 40.36 & $\uparrow$4.60\% & $\uparrow$5.05\% \\
        TwinFlow & 1.00 & 55.85 & 1.30 & 0.82 & 20.44 / 20.44 & 38.48 & $\uparrow$0.05\% & $\uparrow$0.16\% \\
        \rowcolor{gray!8}\textbf{PVD (Ours)} & \textbf{1.02} & \textbf{56.81} & \textbf{0.81} & \textbf{1.35} & \textbf{10.40 / 10.55} & \textbf{19.84} & \textbf{$\downarrow$49.10\%} & \textbf{$\downarrow$48.36\%} \\
        \bottomrule
    \end{tabular}%
    }
\end{table}

\clearpage
\begingroup
\raggedbottom
\section{More T2I Visualizations}
\label{app:t2i_visuals}

\noindent\Cref{fig:t2i_appendix} provides additional qualitative comparisons across SD3.5-Medium, FLUX.1-dev, and Qwen-Image, complementing the results in \cref{fig:t2i_compare} of the main paper.
\Cref{fig:t2i_appendix_further_train} further illustrates the improvements in lighting, composition, and fine-grained details achieved by continued training on Unsplash data.
Intermediate and final outputs of PVD distilled from Qwen-Image are shown in \cref{fig:qwen_intermediate_states}.

\par
\begin{figure}[H]
    \centering
    \includegraphics[width=1\columnwidth]{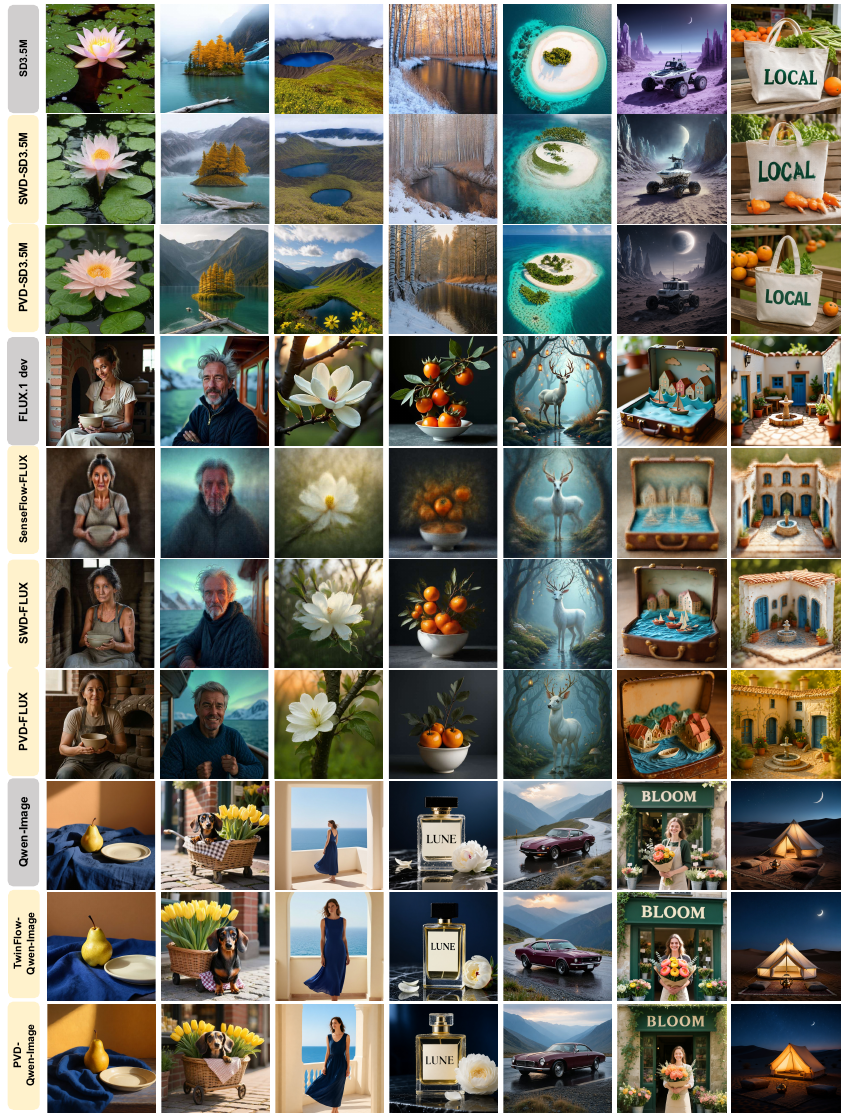}
    \caption{\textbf{Qualitative comparisons on SD3.5-Medium, FLUX.1-dev, and Qwen-Image.} Rows are grouped by backbone and labeled with the corresponding teacher or distillation method.}
    \label{fig:t2i_appendix}
\end{figure}

\par
\begin{figure}[H]
    \centering
    \includegraphics[width=0.98\columnwidth]{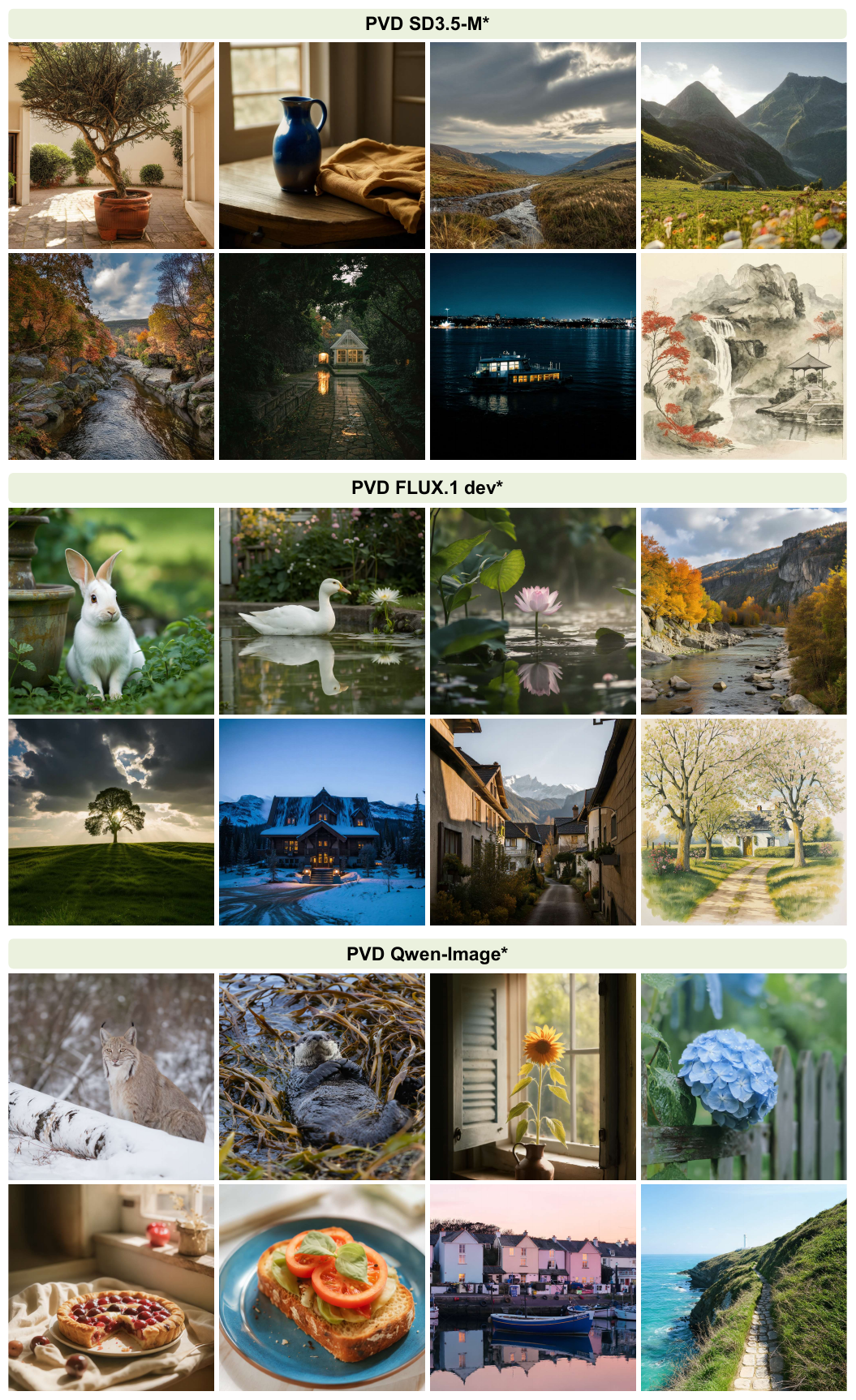}
    \caption{{\textbf{Additional visualizations after further training on Unsplash data.} The models further trained on Unsplash images are marked with $^{*}$.}}
\label{fig:t2i_appendix_further_train}
\end{figure}

\par
\endgroup

\clearpage
\begingroup
\paragraph{Intermediate States of PVD Distilled from Qwen-Image.}
\Cref{fig:qwen_intermediate_states} shows the intermediate-to-final pairs generated by PVD distilled from Qwen-Image. The first phase establishes a coarse structural scaffold: the overall composition, approximate subject silhouette, semantic layout, and color of regions are already discernible despite substantial residual noise.
At this coarse stage, the emerging content is consistent with the recovery of low-frequency scene structure, providing a global arrangement for subsequent refinement.
The second phase builds on this intermediate state, making fine-scale and high-frequency content more distinct.

\begin{figure}[H]
\centering
\captionsetup{font=small,labelfont={bf}, skip=6pt}
\footnotesize
\setlength{\tabcolsep}{3pt}
\renewcommand{\arraystretch}{1.05}
\begin{tabular}{@{}cccc@{}}
\multicolumn{2}{c}{\textbf{(a) Flying bird}} & \multicolumn{2}{c}{\textbf{(b) Origami crane}} \\
\textbf{Phase 1: intermediate} & \textbf{Phase 2: final} & \textbf{Phase 1: intermediate} & \textbf{Phase 2: final} \\[2pt]
\includegraphics[width=0.237\linewidth]{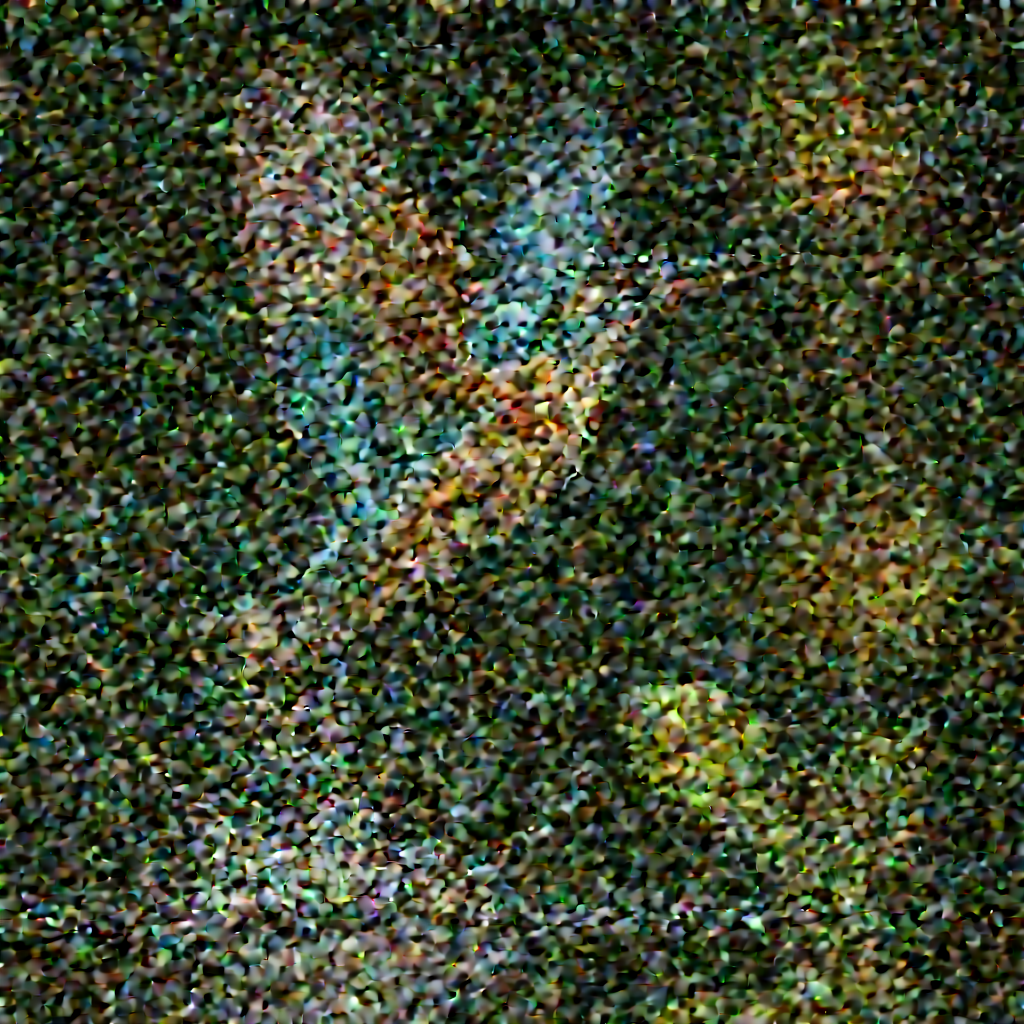} &
\includegraphics[width=0.237\linewidth]{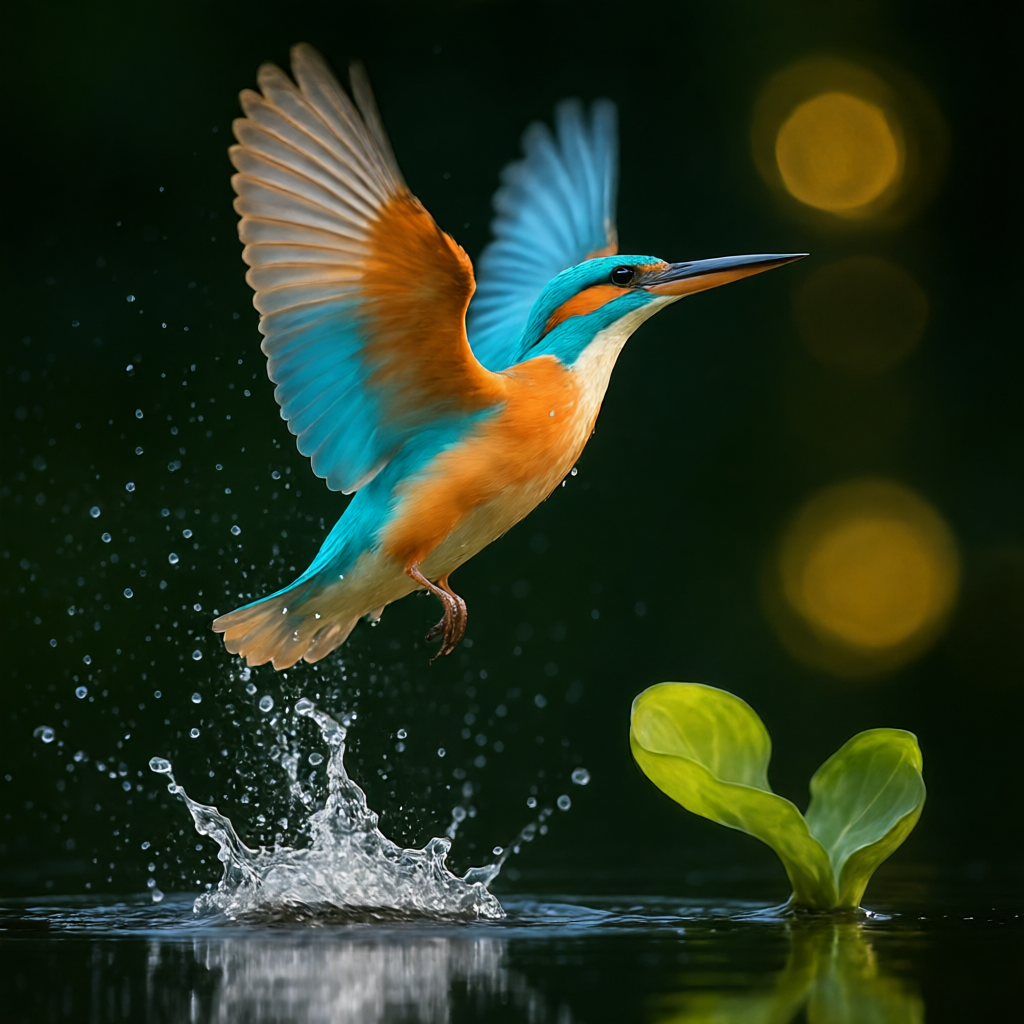} &
\includegraphics[width=0.237\linewidth]{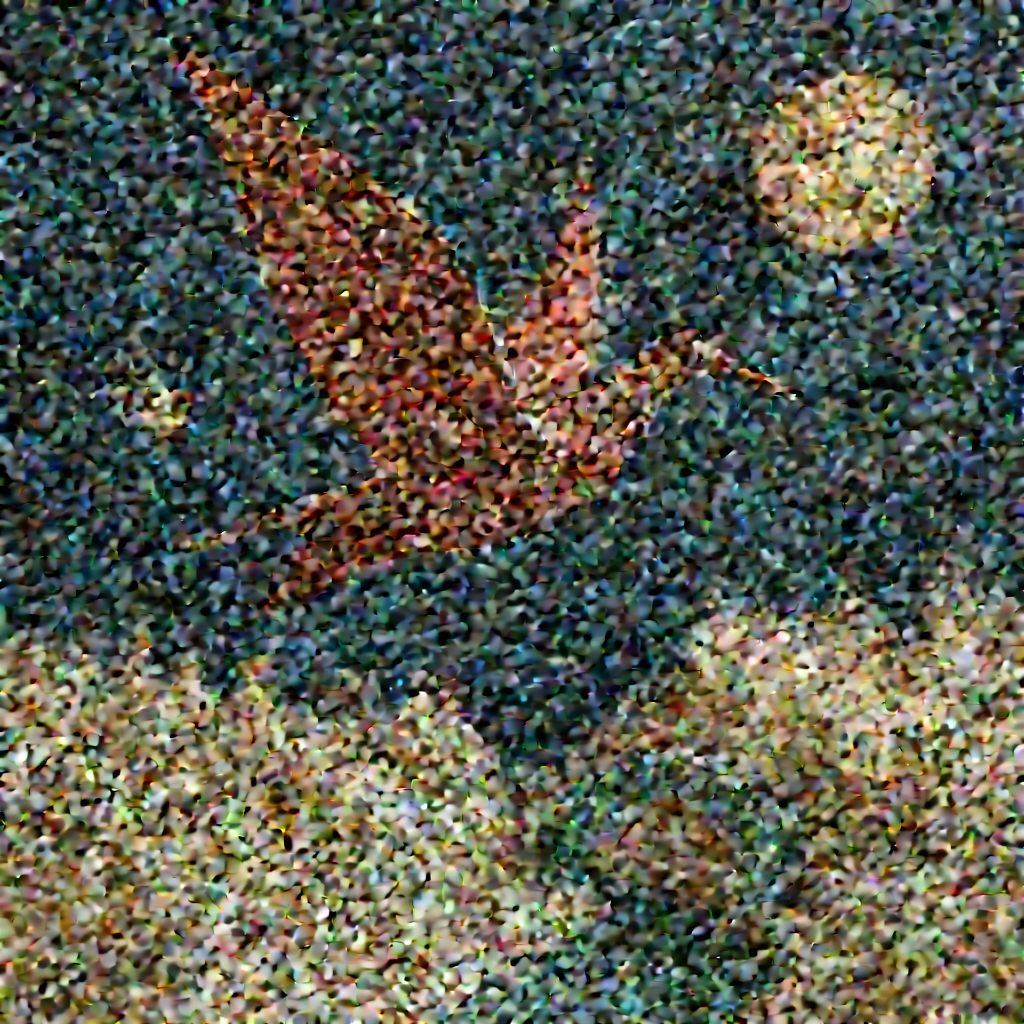} &
\includegraphics[width=0.237\linewidth]{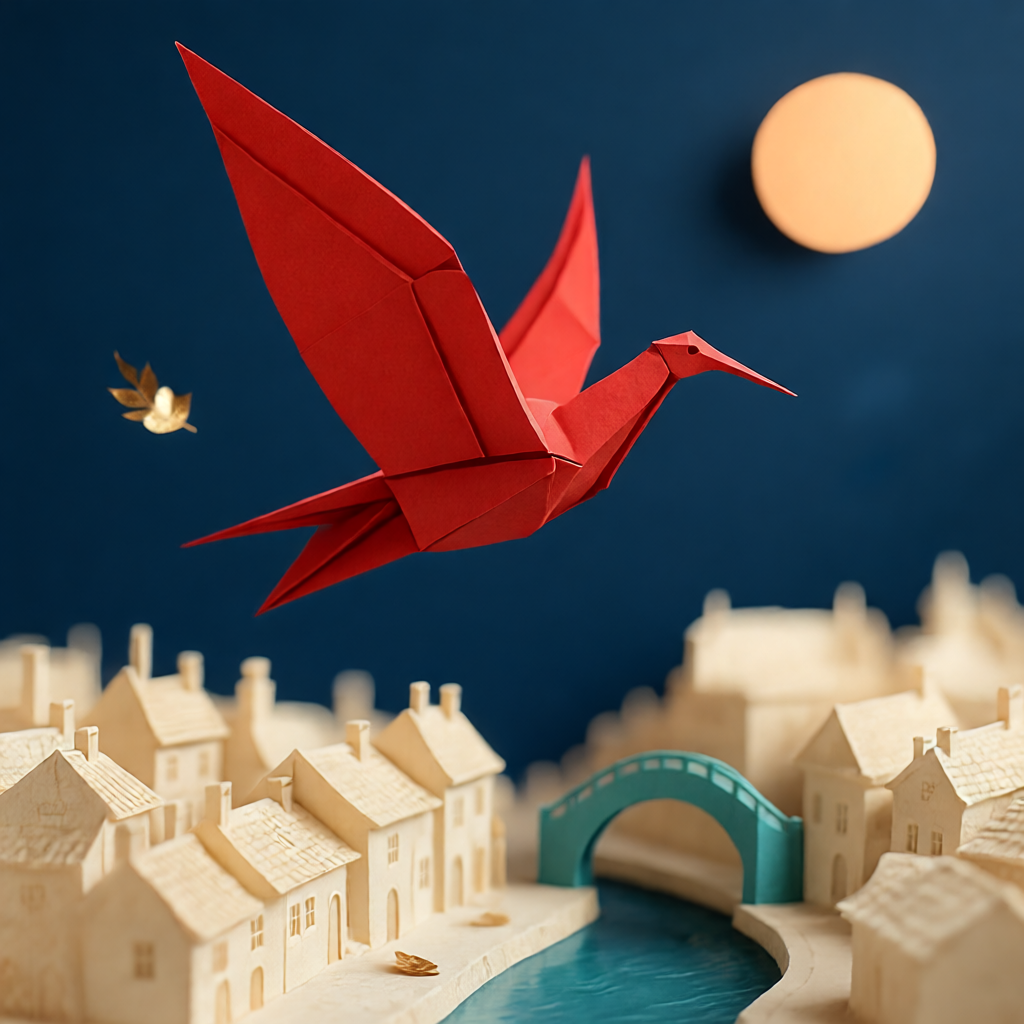} \\[6pt]
\multicolumn{2}{c}{\textbf{(c) Citrus soda can}} & \multicolumn{2}{c}{\textbf{(d) Astronaut portrait}} \\
\textbf{Phase 1: intermediate} & \textbf{Phase 2: final} & \textbf{Phase 1: intermediate} & \textbf{Phase 2: final} \\[2pt]
\includegraphics[width=0.237\linewidth]{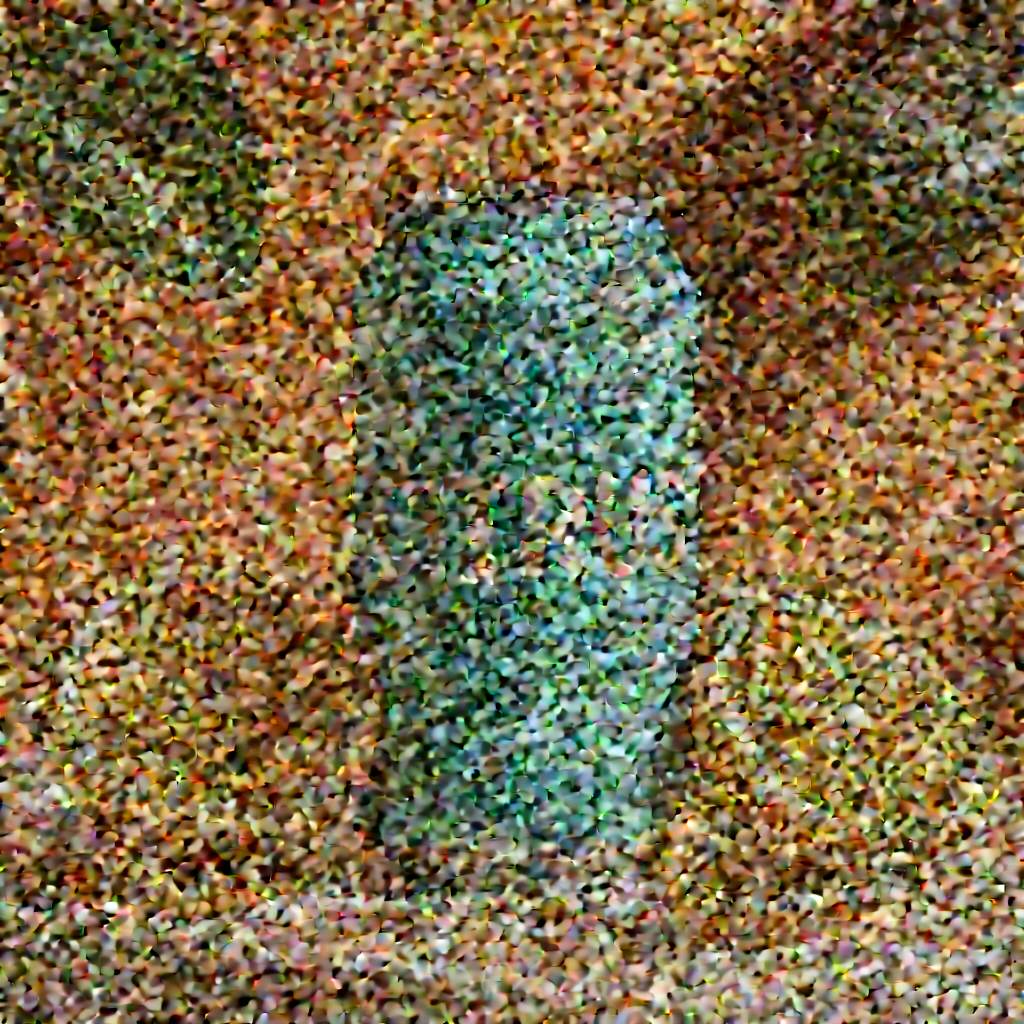} &
\includegraphics[width=0.237\linewidth]{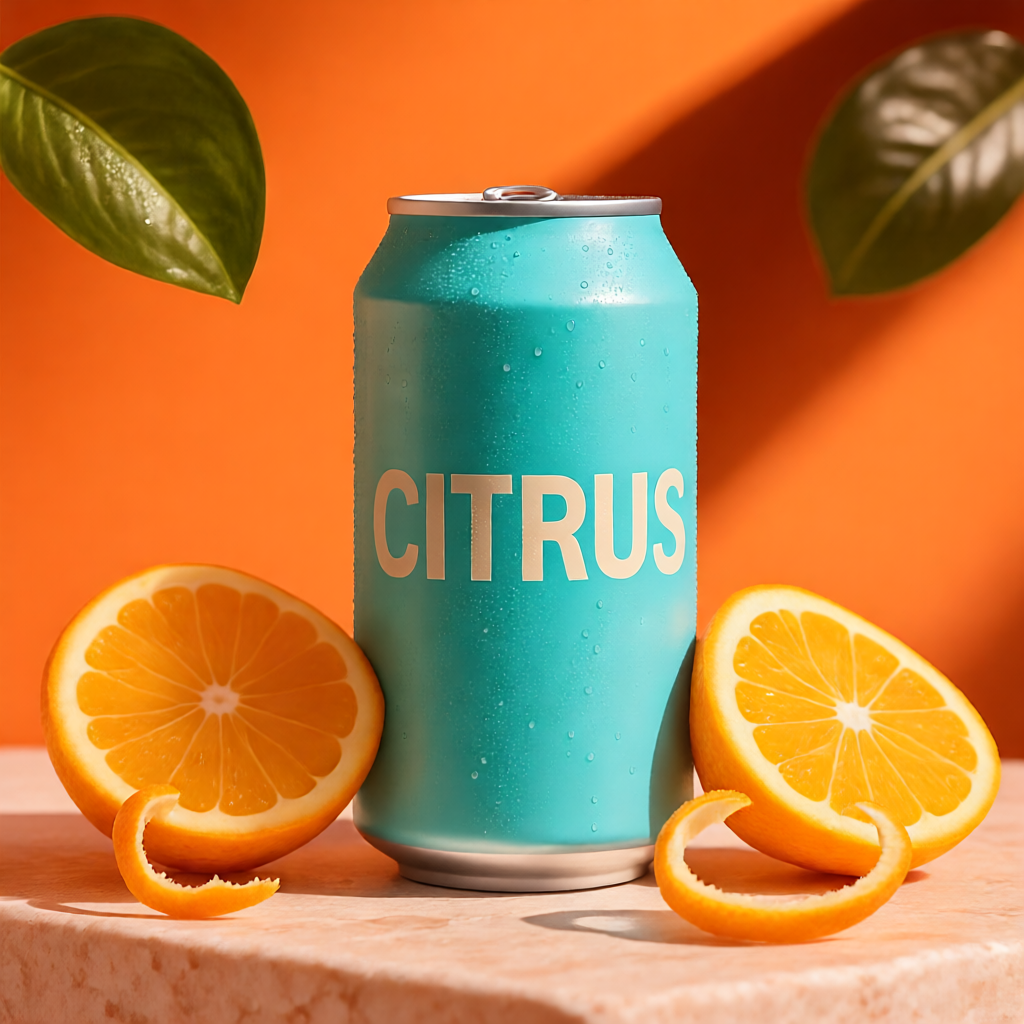} &
\includegraphics[width=0.237\linewidth]{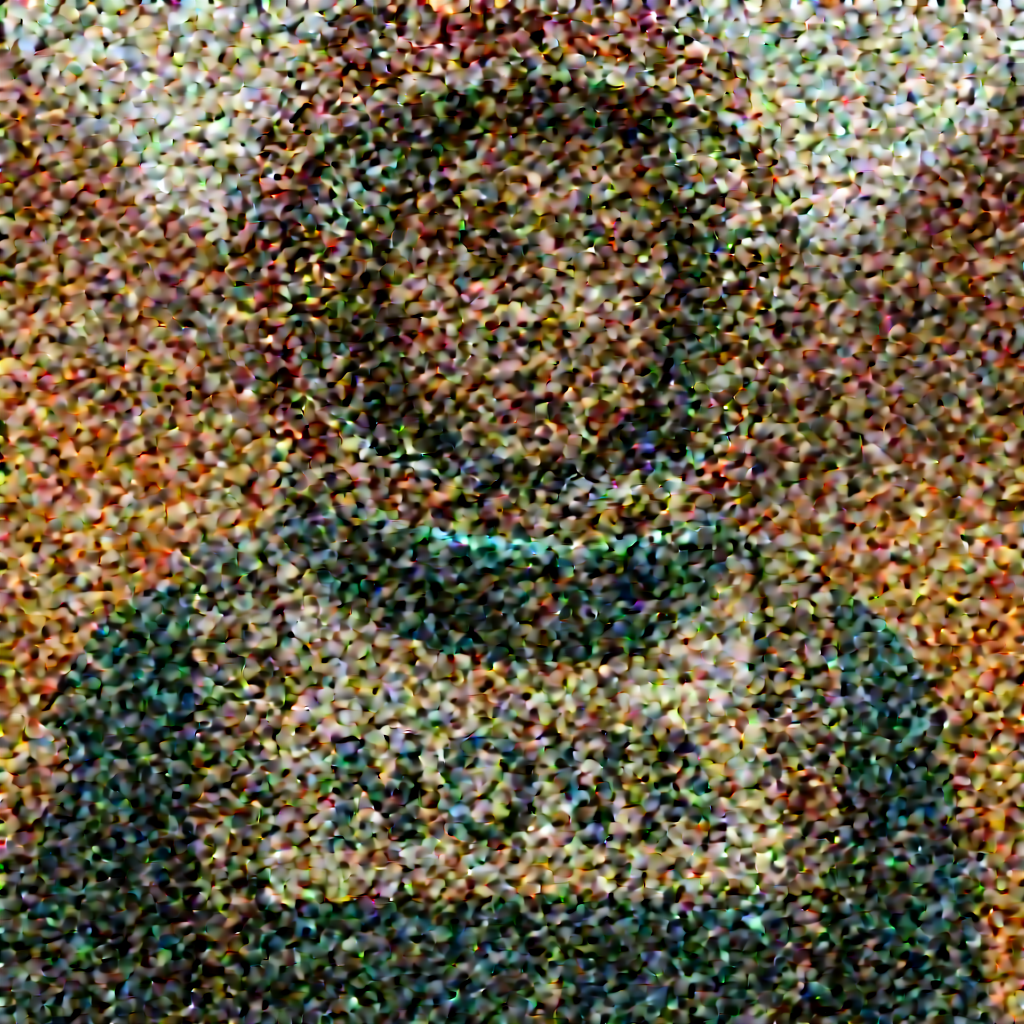} &
\includegraphics[width=0.237\linewidth]{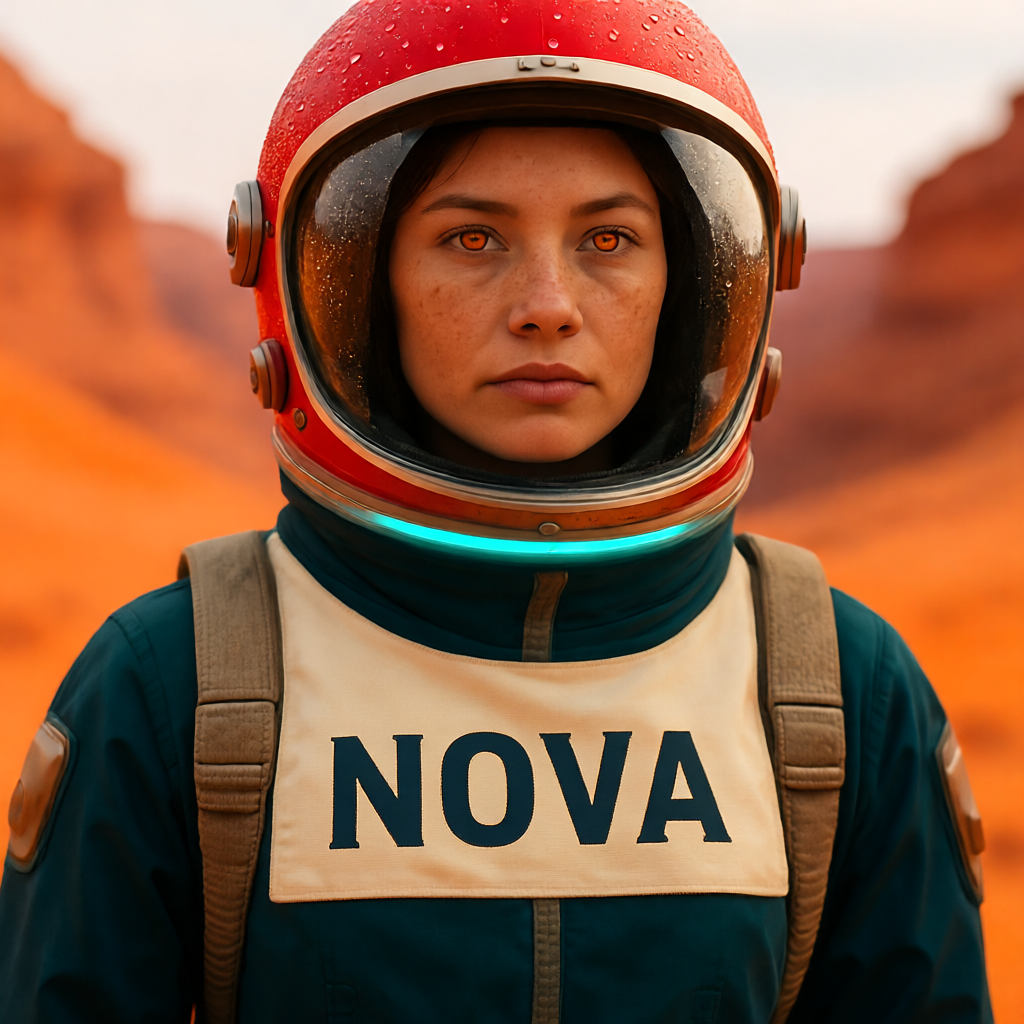} \\[6pt]
\multicolumn{2}{c}{\textbf{(e) Horse at the shore}} & \multicolumn{2}{c}{\textbf{(f) Toy holding a sign}} \\
\textbf{Phase 1: intermediate} & \textbf{Phase 2: final} & \textbf{Phase 1: intermediate} & \textbf{Phase 2: final} \\[2pt]
\includegraphics[width=0.237\linewidth]{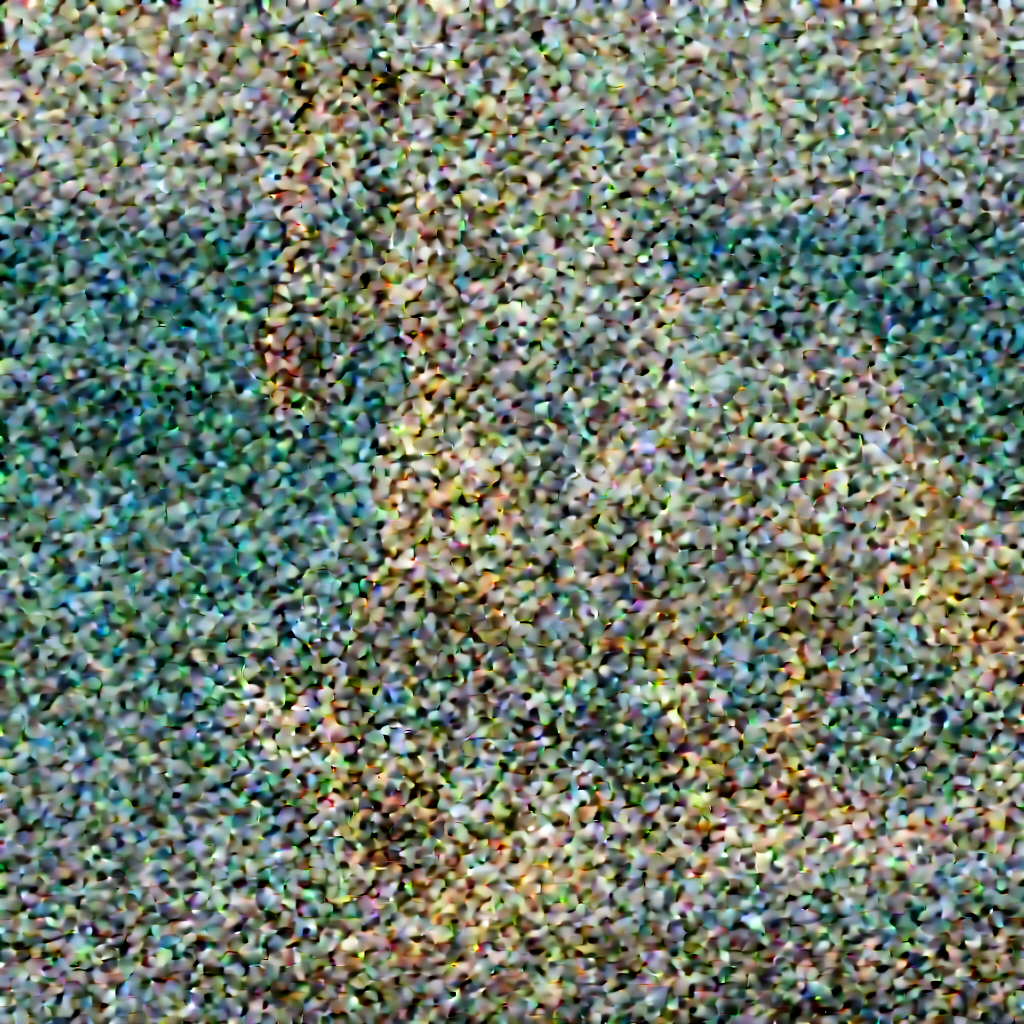} &
\includegraphics[width=0.237\linewidth]{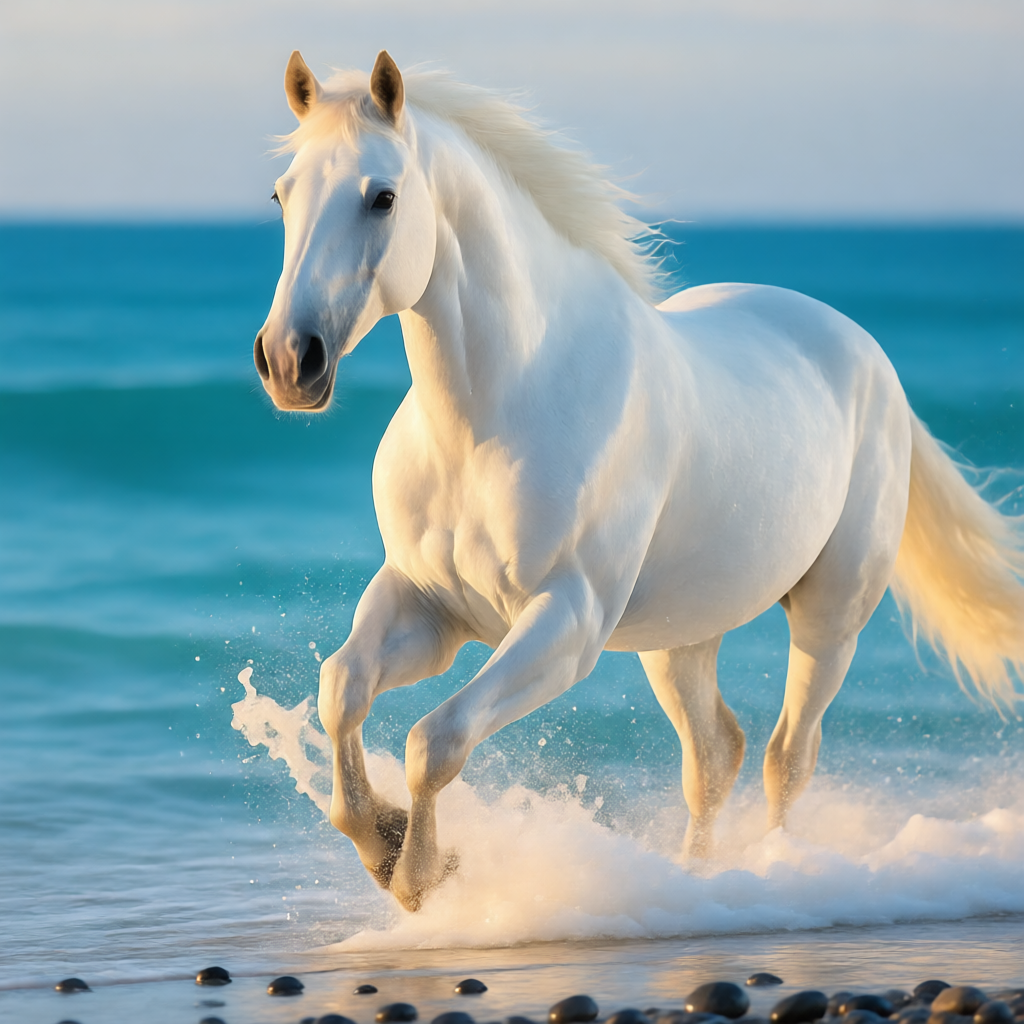} &
\includegraphics[width=0.237\linewidth]{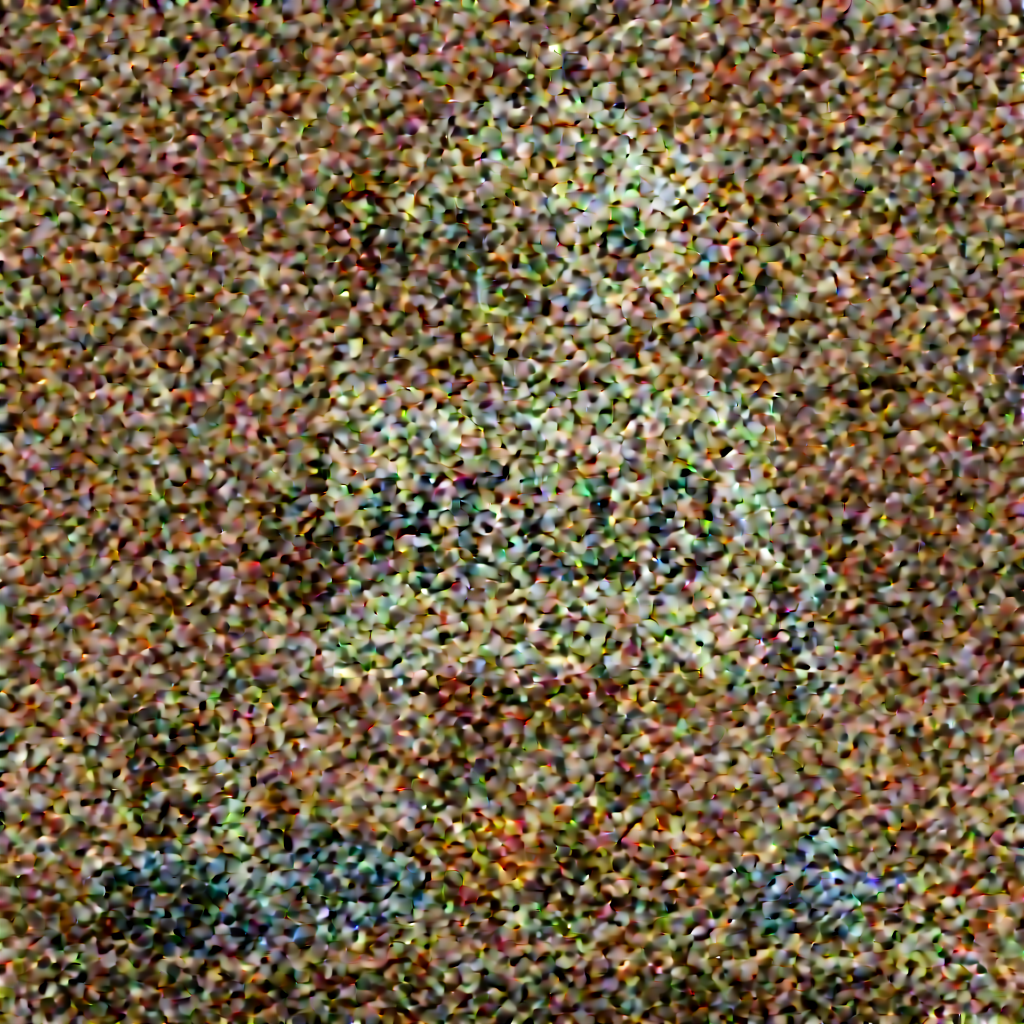} &
\includegraphics[width=0.237\linewidth]{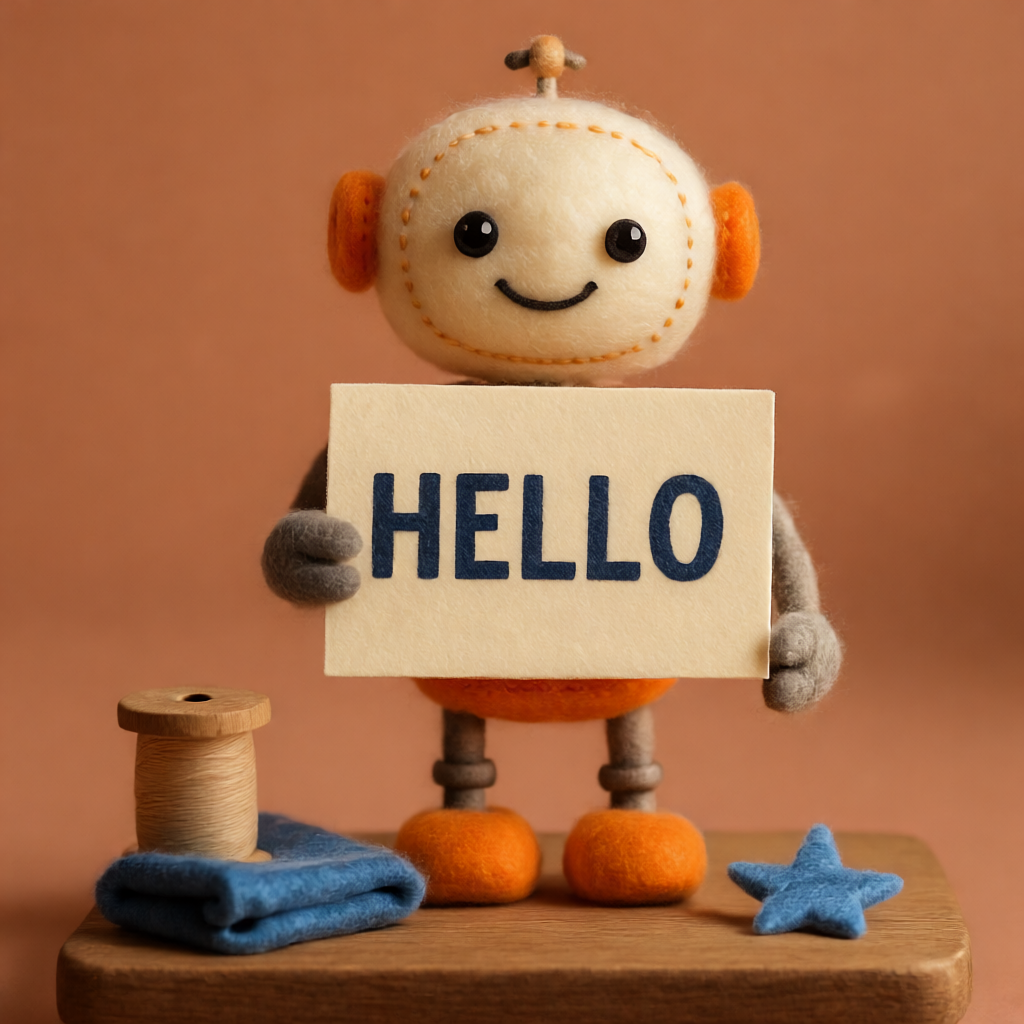} \\
\end{tabular}
\caption{\textbf{Coarse-to-fine progression of PVD distilled from Qwen-Image.} Within each pair, the left image visualizes the state after the first phase and the right image shows the output after the second phase. The first phase reveals coarse composition, approximate subject silhouettes, and color of regions despite residual noise. The second phase refines this structure with sharper edges, surface textures, and finer letter strokes.}
\label{fig:qwen_intermediate_states}
\end{figure}
\par
\endgroup
\clearpage

\section{Ablation Studies on T2I Task} \label{app:exp_results}

\begin{table}[H]
    \centering
    \caption{{\textbf{Reported quality--compute comparison on Qwen-Image.} GenEval results for MeanFlow are sourced from \cite{cheng2025twinflow}.}}
    \label{tab:qwen_meanflow_compute}
    \small
    \setlength{\tabcolsep}{4.5pt}
    \renewcommand{\arraystretch}{1.08}
    \resizebox{1\linewidth}{!}{
    \begin{tabular}{l c c c c c c}
        \toprule
        \textbf{Method}
        & \textbf{$N_{\mathrm{flops}}$}
        & \makecell{\textbf{Active}\\\textbf{Params. (B)}}
        & \makecell{\textbf{Total}\\\textbf{Params. (B)}}
        & \textbf{TFLOPs}
        & \makecell{\textbf{Peak VRAM}\\\textbf{(GiB)}}
        & \textbf{GenEval $\uparrow$} \\
        \midrule
        \rowcolor{gray!15} Qwen-Image & 100.00 & 20.43 & 20.43 & 5585.14 & 38.42 & 0.87 \\
        \midrule
        MeanFlow & 4.00 & 20.44 & 20.44 & 223.40 & 38.48 & 0.44 \\
        MeanFlow  & 8.00 & 20.44 & 20.44 & 446.80 & 38.48 & 0.49 \\
        \rowcolor{gray!8}\textbf{PVD (Ours)} & \textbf{1.02} & \textbf{10.40} & \textbf{10.55} & \textbf{56.81} & \textbf{19.84} & \textbf{0.88} \\
        \bottomrule
    \end{tabular}
    }
\end{table}

\begin{table}[H]
    \centering
    \caption{\textbf{Component and compute-control ablations on SD3.5-Medium.}}
    \label{tab:ablation}
    \resizebox{\linewidth}{!}{
    \begin{tabular}{l cccccc}
    \toprule
    \textbf{Configuration}
    & \textbf{$N_{\mathrm{flops}}$}
    & \makecell{\textbf{Active}\\\textbf{Params. (B)}}
    & \makecell{\textbf{Total}\\\textbf{Params. (B)}}
    & \textbf{TFLOPs}
    & \makecell{\textbf{Peak VRAM}\\\textbf{(GiB)}}
    & \textbf{GenEval $\uparrow$} \\
    \midrule
    w/o discriminator & 0.99 & 1.10 & 2.22 & 6.93 & 2.67 & 0.5599 \\
    \midrule
    {Full-size student} & 1.00 & 2.24 & 2.24 & 7.02 & 4.99 & 0.4142 \\
    {Full-size student} & 2.00 & 2.24 & 2.24 & 14.04 & 4.99 & 0.5695 \\
    \midrule
    \rowcolor{gray!10} \textbf{PVD} & \textbf{0.99} & \textbf{1.10} & \textbf{2.22} & \textbf{6.93} & \textbf{2.67} & \textbf{0.6948} \\
    \bottomrule
    \end{tabular}
    }
\end{table}

\begin{table}[H]
    \centering
    \caption{{\textbf{Expert specialization and execution-order ablations on SD3.5-Medium.} Rows specify the expert used in the first and second trajectory phases, respectively.}}
    \label{tab:expert_order}
    \setlength{\tabcolsep}{10pt}
    \begin{tabular}{lcc}
    \toprule
    {\textbf{Phase 1 $\rightarrow$ Phase 2}} & {\textbf{GenEval $\uparrow$}} & \textbf{ImageReward $\uparrow$} \\
    \midrule
    \rowcolor{gray!10}{$E_{\mathrm{early}}\!\rightarrow E_{\mathrm{late}}$}  & {\textbf{0.6948}} & {\textbf{0.4947}} \\
    {$E_{\mathrm{early}}\!\rightarrow E_{\mathrm{early}}$} & {0.6754} & {0.4598} \\
    {$E_{\mathrm{late}}\!\rightarrow E_{\mathrm{late}}$}   & {0.1820} & {-1.7190} \\
    {$E_{\mathrm{late}}\!\rightarrow E_{\mathrm{early}}$}  & {0.2346} & {-1.5811} \\
    \bottomrule
    \end{tabular}
\end{table}

\paragraph{Shared MeanFlow vs. Phase-specific Experts.}
{MeanFlow learns an interval-conditioned average velocity using one network shared across all sampled interval pairs. Under one-step generation, this network must model the complete noise-to-data transport in a single query. Pu et al.~\cite{pu2025few} show that this setting is substantially harder for T2I than for C2I: their MeanFlow adaptation collapses to pure noise after 1 step, retains visible noise and artifacts after 2 steps, and produces high-fidelity images after 4 steps. Increasing the number of steps can therefore improve fidelity, but it repeatedly invokes the same global capacity. PVD instead retains MeanFlow distillation while factorizing the transport into early and late phases handled by compact phase-specific experts.}

\noindent{\cref{tab:qwen_meanflow_compute} shows that MeanFlow improves the GenEval score from 0.44 to 0.49 when the sampling budget increases from 4 to 8 steps, while $N_{\mathrm{flops}}$ doubles from 4.00 to 8.00. PVD obtains a GenEval score of 0.88 with
$N_{\mathrm{flops}}=1.02$.
The Qwen-Image teacher obtains 0.87 at substantially higher inference cost. 

PVD uses the compute more effectively than the shared-model alternative: its compact phase-specific experts achieve substantially higher generation quality than repeatedly querying a shared MeanFlow backbone, while requiring only about one full-backbone evaluation in cumulative computation.}

\paragraph{Phase Decomposition and Adversarial Refinement.}
{We then separate the effect of the compact phase-wise cascade from additional computation and adversarial supervision. \cref{tab:ablation} compares default PVD with (i) the same cascade trained without the discriminator and (ii) a monolithic full-size student trained over the entire time interval, rather than over separate phase intervals, under matched and increased cumulative-compute budgets. For the full-size control, all other training settings are kept identical to those of the phase experts; the only differences are that the student retains the full backbone capacity and is optimized over $[0,1]$ instead of a phase-specific sub-interval.}

\noindent{We see that the default PVD reaches 0.6948 GenEval score at 6.93 TFLOPs and 2.67 GiB peak VRAM. At nearly matched cumulative computation, the full-size student obtains 0.4142 GenEval at 7.02 TFLOPs and 4.99 GiB. Increasing its budget raises GenEval to 0.5695, but also doubles $N_{\mathrm{flops}}$ to 2.00 and increases the cost to 14.04 TFLOPs. Removing the discriminator leaves inference cost unchanged but reduces GenEval from 0.6948 to 0.5599.}

\paragraph{Expert Specialization and Execution Order.}
{Finally, in Tab. \ref{tab:expert_order} we test whether the two compact experts learn interchangeable mappings or phase-specific operators. We reuse one expert for both phases or reverse the designated early-to-late execution order.}

\noindent{The designated $E_{\mathrm{early}}\!\rightarrow E_{\mathrm{late}}$ composition achieves the highest GenEval score (0.6948) and ImageReward (0.4947). Reusing the early expert in both phases lowers the GenEval score to 0.6754 and ImageReward to 0.4598, while the variants that begin with the late expert lead to significantly worse GenEval and ImageReward values.

It can be concluded that our PVD design is more than a two-call compact cascade: its advantage comes from assigning dedicated experts to complementary trajectory phases in a natural early-to-late order. Reusing one expert or reversing this order consistently weakens generation quality.}

\clearpage
\begingroup
\setlength{\intextsep}{6pt plus 1pt minus 1pt}
\section{Failure Cases}
\label{app:failure_cases}

\begin{figure}[H]
\centering
\captionsetup{font=footnotesize,labelfont={bf},skip=4pt}
\footnotesize
\setlength{\tabcolsep}{5pt}
\begin{tabular}{@{}cc@{}}
\textbf{Qwen-Image} & \textbf{PVD} \\
\multicolumn{2}{c}{\textbf{Request: 35 macarons, 5 rows $\times$ 7 columns.}} \\[2pt]
\includegraphics[width=0.32\linewidth]{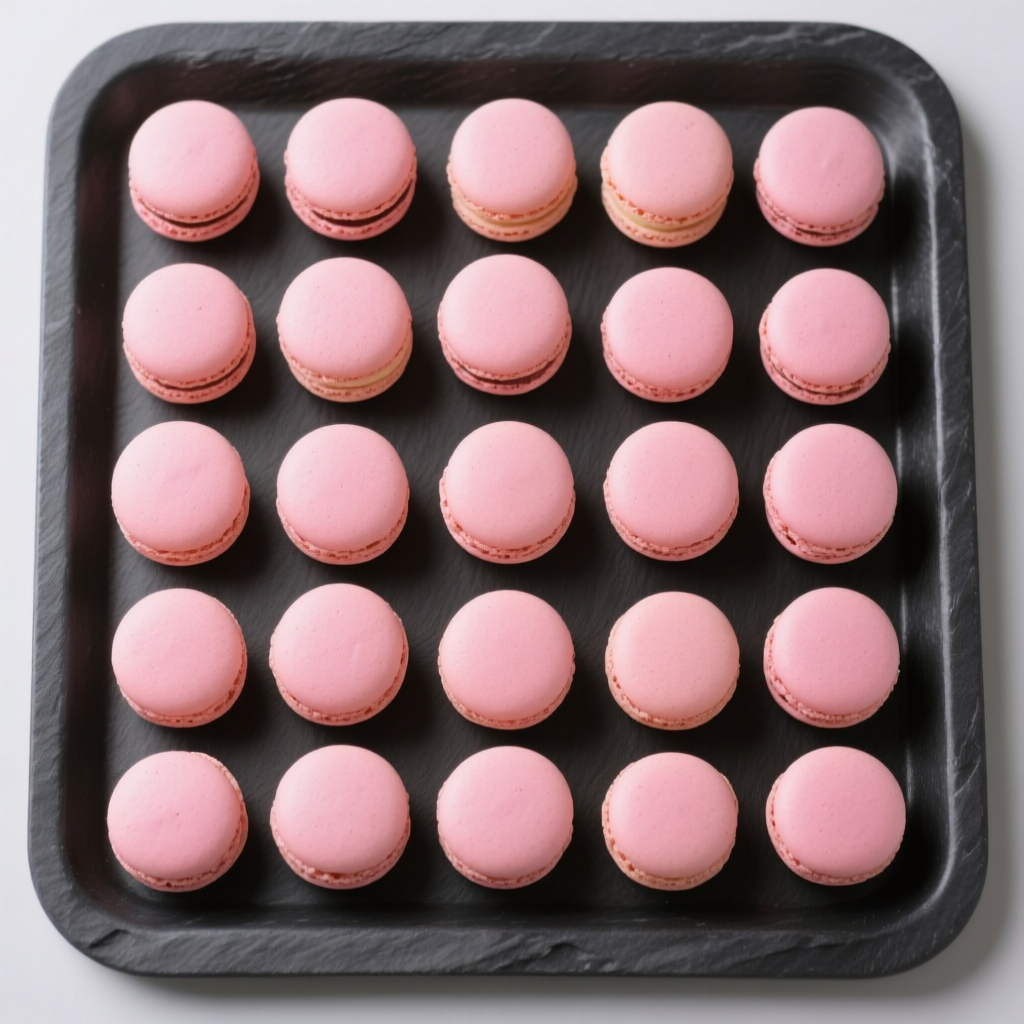} &
\includegraphics[width=0.32\linewidth]{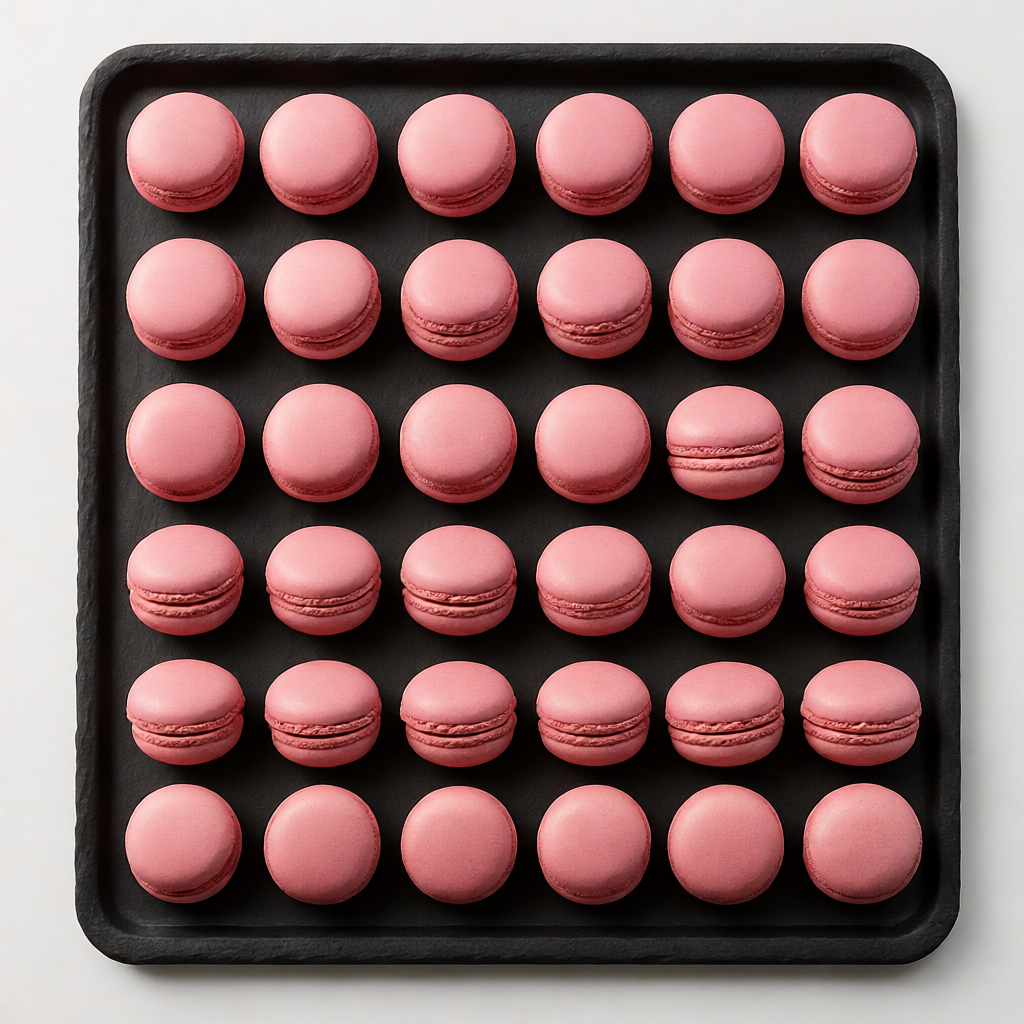} \\
$5\times5=25$ macarons & $6\times6=36$ macarons \\
\end{tabular}
\caption{Prompt: \textit{A top-down studio photograph of exactly thirty-five small pink macarons arranged in a precise rectangular grid of five rows and seven columns on a dark slate tray. Every row contains exactly seven macarons and every column exactly five. All macarons are identical, separate, fully visible, and not touching. The entire tray is inside the frame.} Both methods violate the requested count and grid dimensions despite producing recognizable, regularly arranged objects.}
\label{fig:failure_cases}
\end{figure}

\begin{figure}[H]
\centering
\captionsetup{font=footnotesize,labelfont={bf}, skip=4pt}
\footnotesize
\setlength{\tabcolsep}{5pt}
\begin{tabular}{@{}cc@{}}
\textbf{Qwen-Image} & \textbf{PVD} \\
\multicolumn{2}{c}{\textbf{Request: pencil through the handle hole, outside the cup.}} \\[2pt]
\includegraphics[width=0.32\linewidth]{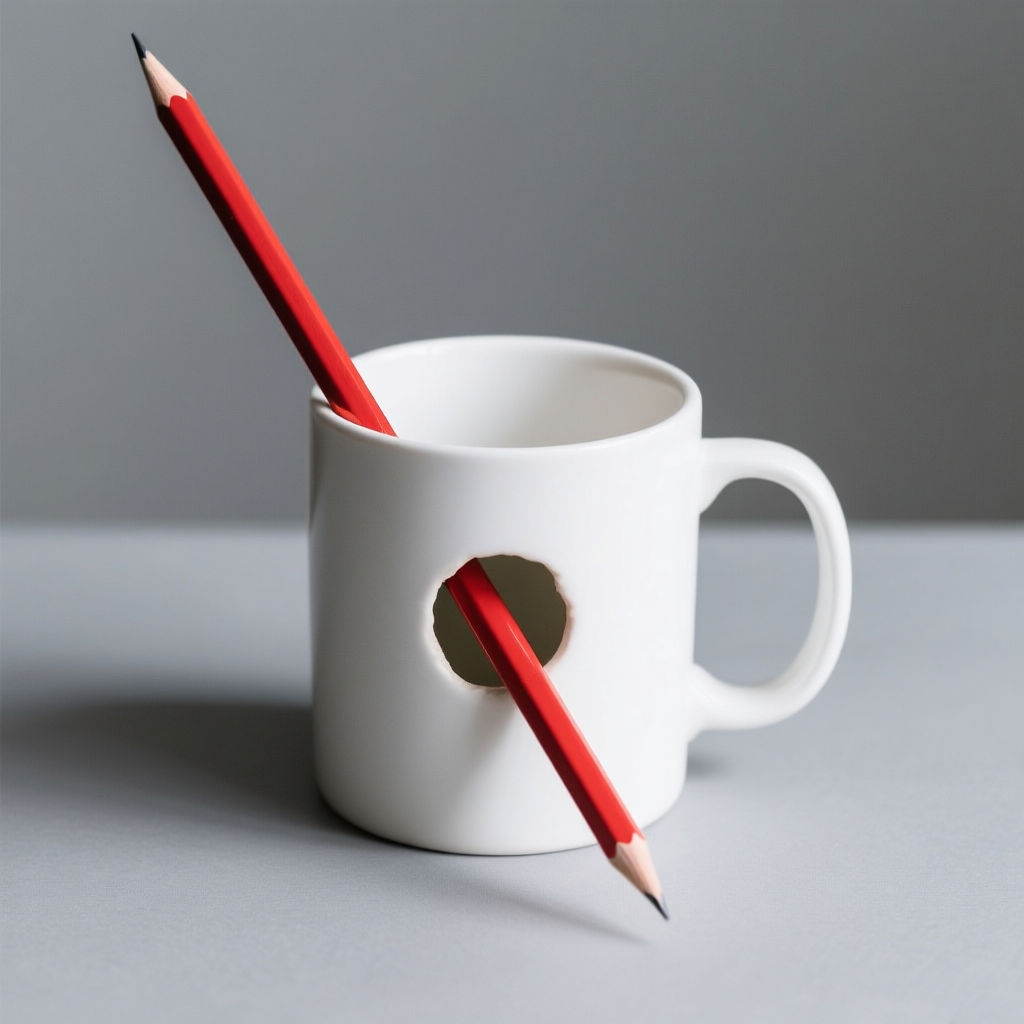} &
\includegraphics[width=0.32\linewidth]{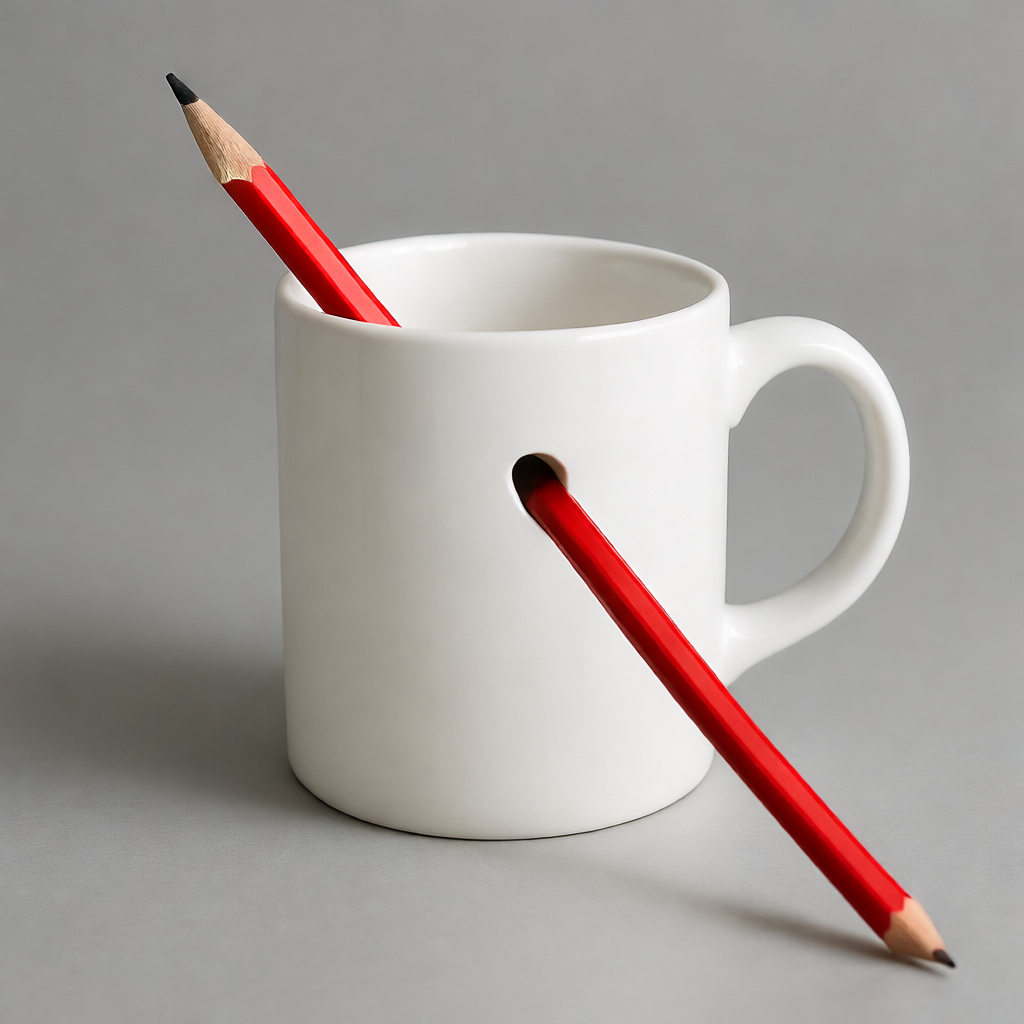} \\
\multicolumn{2}{c}{Both: pencil enters the cup and passes through the mug wall.} \\
\end{tabular}
\caption{Prompt: \textit{A clear studio photograph of one upright white ceramic mug with its handle on the right. A long red pencil passes diagonally through the empty hole of the mug handle, from upper left to lower right. The pencil is outside the cup, never entering the opening at the top. Both ends of the pencil extend visibly beyond the handle. Three-quarter front view showing the handle hole and the empty cup opening, on a plain gray tabletop.} Both methods place the pencil through the cup opening and mug wall instead of the existing handle hole.}
\label{fig:failure_spatial}
\end{figure}

In this section, we present failure cases of PVD regarding T2I generation. The failures highlight the limitations of current generative models in adhering to precise instructions, even when the generated content appears coherent and realistic.

In \cref{fig:failure_cases}, the prompt specifies exactly 35 macarons in five rows and seven columns. Qwen-Image produces a $5\times5$ grid (25), whereas PVD produces a $6\times6$ grid (36). Both preserve the appearance of an orderly tray but violate the requested count and arrangement.

In \Cref{fig:failure_spatial}, the prompt requests a pencil passing through the mug handle while remaining outside the cup. Both the teacher and student models instead place it inside the cup opening and through an additional hole in the mug wall, leaving the handle hole empty. The objects remain recognizable, but their spatial relation contradicts the prompt.

Targeted post-training may improve adherence to these constraints. Preference-based fine-tuning or reinforcement learning with relative rewards could provide complementary supervision, helping to align the generated images with the intended specifications.

\par
\endgroup

\clearpage
\section{Broader Impacts} \label{app:broader_impacts}
\textbf{Positive Impacts.} By cutting inference cost and halving VRAM consumption, PVD lowers the barrier to accessing state-of-the-art generative AI. Such exceptional efficiency enables deployment on consumer-grade edge devices, substantially reducing energy consumption and empowering a variety of real-time interactive applications.

\textbf{Negative Impacts and Mitigations.} Conversely, the unprecedented speed and low cost of synthesizing photorealistic images may increase the risks of generating deepfakes and misinformation. To mitigate these risks, we advocate for responsible deployment, including the mandatory integration of watermarking for traceability and safety content filters into the generation pipeline.

\end{document}